\documentclass{article}

\usepackage{multicol}
\usepackage{hyperref}

\usepackage{microtype}
\usepackage{graphicx}
\usepackage{amsfonts}
\usepackage{amsthm}
\usepackage{amsmath}
\usepackage{amssymb}
\usepackage{subfigure}
\usepackage{booktabs} 
\newtheorem{theorem}{Theorem}
\newtheorem{conclusion}{Conclusion}

\newtheorem{lemma}{Lemma}

\newtheorem{experiment}{Experiment}
\newcommand{\iu}{\mathrm{i}\mkern1mu}

\usepackage[preprint, nonatbib]{neurips_2022}

\usepackage[utf8]{inputenc} 
\usepackage[T1]{fontenc}    
\usepackage{hyperref}       
\usepackage{url}            
\usepackage{booktabs}       
\usepackage{amsfonts}       
\usepackage{nicefrac}       
\usepackage{microtype}      
\usepackage{xcolor}         

\title{On the Activation Function Dependence of the Spectral Bias of Neural Networks}

\author{%
  Qingguo Hong\\
  Department of Mathematics, Pennsylvania State University\\
  University Park, PA 16801 \\
  \texttt{huq11@psu.edu} \\
   \And
 Jonathan W. Siegel\\
  Department of Mathematics, Pennsylvania State University\\
  University Park, PA 16801 \\
  \texttt{jus1949@psu.edu} \\
   \And
  Qinyang Tan \\
  Department of Mathematics, University of Southern California\\
  Los Angeles, CA 90089 \\
  \texttt{qinyangt@usc.edu} \\
  \And
  Jinchao Xu\\
  Department of Mathematics, Pennsylvania State University\\
  University Park, PA 16801 \\
  \texttt{xu@math.psu.edu} \\
}

\begin{document}

\maketitle

\begin{abstract}
 Neural networks are universal function approximators which are known to generalize well despite being dramatically overparameterized. We study this phenomenon from the point of view of the spectral bias of neural networks. Our contributions are two-fold. First, we provide a theoretical explanation for the spectral bias of ReLU neural networks by leveraging connections with the theory of finite element methods. Second, based upon this theory we predict that switching the activation function to a piecewise linear B-spline, namely the Hat function, will remove this spectral bias, which we verify empirically in a variety of settings. Our empirical studies also show that neural networks with the Hat activation function are trained significantly faster using stochastic gradient descent and ADAM. Combined with previous work showing that the Hat activation function also improves generalization accuracy on image classification tasks, this indicates that using the Hat activation provides significant advantages over the ReLU on certain problems.
\end{abstract}

\section{Introduction}
\label{intro}
Despite being heavily overparameterized, deep neural networks have been shown to be remarkably good at generalizing to natural data. This has raised the important problem of understanding the implicit regularization effect of deep neural networks \cite{poggio2018theory,rahaman2019spectral,soudry2018implicit,xu2018understanding}. Recently, it has been proposed that the training of deep neural networks exhibits a spectral bias \cite{rahaman2019spectral,xu2018understanding}, in other words that low frequencies are learned faster via training by stochastic gradient descent. This is proposed as a mechanism which biases networks toward low-complexity solutions \cite{rahaman2019spectral}, and also helps explain the recent success in using neural networks to solve PDEs \cite{han2018solving,raissi2019physics}.

In this work, we study the dependence of the spectral bias phenomenon on the activation function of the neural network. We begin by giving a \textit{theoretical derivation} of the spectral bias phenomenon for neural networks with ReLU activation function. This analysis is based on the connections between ReLU networks and finite element methods \cite{he2020relu}. Based upon this theory, we predict that a simple modification of the ReLU activation function, called the Hat activation function, will remove the spectral bias of neural networks. We empirically verify these predictions on a variety of problems, including both shallow and deep networks and real and simulated data. Moreover, we demonstrate that removing the spectral bias has the effect that neural networks with the Hat activation function are trained much faster than ReLU neural networks. Combined with recent results showing that the Hat activation can improve performance on image classification tasks \cite{wang2022cnns}, this demonstrates a significant advantage to using the Hat activation function over the ReLU. 

\textbf{Contributions}
\begin{enumerate}
 \item We provide a theoretical explanation for the spectral bias of ReLU neural networks by drawing a comparison with the analysis of finite element methods. Our analysis applies to networks of arbitrary width, i.e. not only in the infinite or large width limit, and also shows how the spectral bias increases with increasing width. 
 \item Based upon this theory, we predict that neural networks with a simple modification of the ReLU, called the Hat activation function, will not experience the spectral bias observed in ReLU networks.
 \item We validate this prediction empirically. Specifically, we show that the spectral bias disappears when the ReLU activation function is replaced by the Hat activation function or scaled Hat activation function. This shows that the spectral bias of neural networks depends critically upon the activation function used. Our empirical studies also show that the networks with Hat activation function are much easier to train by gradient descent type training algorithm. 
\end{enumerate}

\section{Related Work}
Deep neural networks have been shown to be able to fit random data with perfect accuracy \cite{zhang2021understanding,arpit2017closer}. This has motivated the study of the implicit regularization which prevents overfitting for such networks \cite{poggio2018theory,rahaman2019spectral,cao2019towards}. The spectral bias as a mechanism for regularization has been proposed and studied in \cite{rahaman2019spectral,xu2018understanding,cai2019multi,li2020gradient,xu2022overview,fridovich2021spectral,zhang2020rethink,yang2019fine}. The convergence of neural networks on pure frequencies, which supports the spectral bias phenomenon, has been studied in \cite{basri2020frequency}. For applications of neural networks to differential equations, it has been observed that high-frequencies are often missing or appear late in training, and potential solutions, including novel activation functions, have been proposed \cite{sitzmann2020implicit,biland2019frequency,cai2020phase}. The eigenstructure of the Fisher information matrix of neural network models has been studied in \cite{karakida2019pathological}, where it has been shown that only a few eigenvalues are large for overparameterized neural networks. The spectral bias has also been related to kernel methods and the neural tangent kernel \cite{jacot2018neural,canatar2021spectral,hu2020surprising,kopitkov2020neural,luo2020exact}, and has been observed for GANs \cite{khayatkhoei2020spatial}. In addition, the spectral bias, or F-principle, has been observed to be robust to the specific optimization methods used \cite{ma2021frequency}.

Our work is primarily concerned with the theoretical underpinning of the spectral bias. This has previously been studied in \cite{xu2018understanding} for the sigmoidal activation function. We provide a theory explaining the spectral bias for the ReLU activation function. However, our theory and empirical work also show that the spectral bias is not \textit{universal}. In particular, it depends significantly upon the activation function used. A similar conclusion was reached by analyzing the spectrum of the neural tangent kernel in \cite{yang2019fine}, where it was shown that a sigmoidal network does not exhibit a spectral bias with appropriate initialization. 

In contrast, in our work we show that neural networks with a Hat activation function do not exhibit a spectral bias. The Hat activation function has been empirically analyzed for computer vision problems in \cite{wang2022cnns}, where is it shown that the can improve the performance of the neural network. The closely related Mexican Hat function has also been proposed as an activation function in \cite{nie2014multistability}. Instead of considering the decay of the eigenvalues of the neural tangent kernel, which is adapted to the network architecture, we consider Fourier modes and eigenfunctions of kernels which only depend upon the data, which relates the spectral bias to more natural notions of frequency. The advantage of our approach is that our analysis holds for networks of arbitrary width and shows how the spectral bias increases with increasing width. Further, our empirical work demonstrates that neural networks with the Hat activation function train significantly more quickly than ReLU networks, which motivates further study into the use of the Hat activation.

\section{Spectral Bias of Neural Networks}

The spectral bias of neural networks has been well-established in prior works. This concept of the spectral bias depends upon an appropriate notion of frequency combined with the observation that certain frequencies converge faster during training \cite{rahaman2019spectral,cao2019towards}.

\subsection{Notion of Frequency}

By frequency we mean the eigenfunctions of a self-adjoint compact operator $T$ on $L^2(d\mu)$ for an appropriate measure $d\mu$. This notion has been well-studied in the machine learning literature in the context of kernels \cite{fasshauer2011positive,hofmann2008kernel}. We begin with two important examples, the first corresponding to a continuous measure $d\mu$ and the second to a discrete measure $d\mu$ on the training data.


\textbf{Continuous measure}: Consider the torus $\mathbb{R}^d/\mathbb{Z}^d$ with the Lebesgue measure $d\mu$. Let $T$ be the convolution operator $T_s$ given by
\begin{equation}
 T_s(f)(x) = \left(\frac{s}{2\pi}\right)^{d/2}\int_{\mathbb{R}^d} e^{-\frac{s}{2} |x-y|^2}f(y)dy.
\end{equation}
Note that in the integral above $f(y)$ is viewed as a periodic function on the whole space. It is well-known that this convolution operator gives the solution to the heat equation with periodic boundary conditions and initial condition $f$. Specifically, if $u$ denotes the solution to
\begin{equation}
 u_t + \frac{1}{2}\Delta u = 0,~u(0,x) = f(x),
\end{equation}
then we have $u(s,x) = T_{s^{-1}}(f)(x)$. Thus, the operator $T_s:L^2(\mathbb{R}^d/\mathbb{Z}^d)\rightarrow L^2(\mathbb{R}^d/\mathbb{Z}^d)$ can be thought of as a smoothing or averaging operator.
The eigenfunctions of this operator are given by the Fourier modes $e^{2\pi \iu \omega\cdot x}$ with $\omega\in \mathbb{Z}^d$ and the eigenvalues in this case are given by the Fourier coefficients of the associated kernel $k_s:(\mathbb{R}^d/\mathbb{Z}^d)\times (\mathbb{R}^d/\mathbb{Z}^d)$ given by $$k(x,y) = \left(\frac{s}{4\pi}\right)^{d/2}\sum_{k\in \mathbb{Z}^d} e^{-(s/2)|x-y-k|^2},$$ which are $(4\pi s)^{-d/2}e^{-(2s)^{-1}|\omega|^2}$ \cite{berlinet2011reproducing,hofmann2008kernel}. High frequencies correspond to small eigenvalues. This is due to the intuitive notion that high frequencies decay rapidly under averaging.

\textbf{Discrete measure}: Consider a finite set of data $\mathcal{D} = \{d_1,...,d_N\}$ in $\mathbb{R}^d$. In high dimensions, the number of data points $N$ must scale exponentially with $d$ in order for traditional Fourier modes to be a useful notion of frequency. In particular, in $d$-dimensions, there are $M = O(2^d)$ vectors $\omega_i\in \mathbb{R}^d$ for $i=1,...,M$ which satisfy $|\omega_i - \omega_j| > \frac{1}{2}$ and $|\omega_i| \leq 1$. However, if we only have $N$ datapoints, at most $N$ of the functions $\cos(2\pi \omega_i\cdot x)$ can be linearly independent when evaluated at the data. Thus, in order to distinguish even low frequencies the number of datapoints must be exponential in the dimension.

Instead of considering Fourier modes in this situation, we follow the approach laid out in \cite{rahaman2019spectral} to define a useful notion of frequency on sparse high-dimensional data. Specifically, our notion of frequency is given by the eigenfunctions of the following kernel map. Let $X=L^2(\mathcal{D})$ be the space of functions defined on the dataset $\mathcal{D}$. Define the kernel map $K:X\rightarrow X$ by
\begin{equation}\label{kernel-definition}
 K(f)(y) = \sum_{z\in \mathcal{D}} k(y,z) f(z)
\end{equation}
for a suitable positive definite kernel $k:\mathcal{D}\times \mathcal{D} \rightarrow \mathbb{R}$. The Gaussian RBF kernel $k(y,z) = e^{-s|y-z|^2}$ is a natural choice. The eigenfunctions $\phi_i$ of the kernel map $K$ resemble sinusoids and can be thought of as a proxy for frequency when the data is sparse in high dimensions \cite{rahaman2019spectral,fasshauer2011positive}. This is due to the fact that eigenfunctions corresponding to small eigenvalues decay rapidly when averaging using the kernel, similar to the high frequency functions considering in the previous two examples. Thus this notion gives a reasonable generalization of the notion of frequency on an arbitrary set of datapoints, which is needed to formulate the spectral bias on high-dimensional real data.

\subsection{Spectral Bias}
The spectral bias of neural networks refers to the observation \cite{rahaman2019spectral} that during training a neural network fits the low frequency components of a target faster than the high-frequency components. Specifically, consider a basis of eigenfunctions $\phi_1,\phi_2,...$ corresponding to the eigenvalues $\mu_1 \geq \mu_2\geq \cdots$ of a suitable compact self-adjoint operator $T$ on $L^2(d\mu)$. Note that since the operator $T$ is self-adjoint the eigenfunctions $\phi_i$ can be taken to be orthonormal. If we let $f^\ell$ denote the neural network function at training iteration $\ell$ and $f^*$ the true empirical minimizer, then expanding the difference in terms of the eigenfunctions $\phi_i$, we get
\begin{equation}
 f^\ell - f^* = \sum_i \alpha_i^\ell\phi_i.
\end{equation}
The spectral bias is the statement that $\alpha_i^\ell$ decays more rapidly in $\ell$ for smaller frequencies $i$. The spectral bias phenomenon has been observed experimentally for deep ReLU networks when the $\phi_i$ are Fourier modes \cite{cao2019towards,rahaman2019spectral} and when the $\phi_i$ are the eigenfunctions of a Gaussian RBF kernel \cite{rahaman2019spectral}. This has been proposed as an explanation for the \textit{implicit regularization} of neural networks \cite{rahaman2019spectral}, since it biases neural networks toward smooth functions in a manner similar to using the $T^{-1}$-norm as a regularizer.

We remark that here the larger eigenvalues $\mu_i$ correspond to lower frequencies, which is due to our definition in terms of the compact operator $T$. In some cases, $T$ will be (formally) given by $T = e^{-S}$ for an operator $S$ with eigenfunctions $\phi_i$ and eigenvalues $\lambda_i = e^{-\mu_i}$. In this case, smaller eigenvalues $\lambda_i$ correspond to lower frequencies. A typical example of this is to take $T$ to be the heat kernel and $S$ to be the Laplacian. The advantage of using the operator $T$ instead of $S$ is that $S$ is usually an unbounded operator and thus requires additional technical machinery to interpret correctly.

\section{Spectral Analysis of Shallow Neural Networks}

\subsection{Preliminaries}
Throughout this section we consider shallow neural network functions $f:\mathbb{R}^d\rightarrow \mathbb{R}$ defined by a shallow neural network of the form
\begin{equation}
 f(x) = \sum_{i=1}^n a_i\sigma(\omega_i\cdot x + b_i),
\end{equation}
where $\omega_i\in \mathbb{R}^d$ and $b_i\in \mathbb{R}$. The activation functions we will consider are the ReLU activation function $\sigma(p) :=\rm ReLU(p)= \max(0,p)$ and the Hat activation defined by
\begin{equation}\label{hat-definition}
 \sigma_{H}(p) :={\rm Hat}(p)= \begin{cases}
   0 & \text{$p < 0$~or~$p \geq 2$}\\
   p & 0\leq p < 1\\
   2-p & 1\leq p < 2.
\end{cases}
\end{equation}
The Hat activation function is not scale invariant. As a result, we will also consider more generally a scaled Hat function $\sigma_H(\alpha\cdot)$ as an activation function. It is known that shallow neural networks with either the ReLU or Hat activation function can approximate arbitrary continuous functions \cite{hornik1989multilayer}. This universal approximation property partially explains the success of neural networks. We analyze the spectral bias of shallow neural networks by leveraging its connections with finite element methods.

\subsection{Spectral Bias in 1D for ReLU and Hat Networks}

We consider fitting a one dimensional function $u:[0,1]\rightarrow \mathbb{R}$ with a shallow ReLU neural network. For the theoretical analysis, we consider the simplified situation where the inner weights are fixed and the network is given by
\begin{equation}
 f_{NN}(x,\vec{a}) = \sum_{i=1}^n a_i\sigma\left(x-\frac{i}{n}\right).
\end{equation}
Here only the weights $a_i$ are learned and the $\omega_i = 1$ and $b_i = i/n$ are fixed. We consider minimizing the following loss function using gradient descent
\begin{equation}
 L(\vec{a}) = \frac{1}{2}\int_0^1 (u(x) -f_{NN}(x,\vec{a}))^2 dx.
\end{equation}
Lemma \ref{loss-function-lemma}, which is well-known in the theory of finite elements \cite{ciarlet2002finite}, describes the structure of the loss function $L$ (we refer to the appendix for proofs and technical details).
\begin{lemma}\label{loss-function-lemma}
 The loss function $L$ takes the form
 \begin{equation}
  L(\vec{a}) = \vec{a}^TM_{\sigma}\vec{a} - b_{u,\sigma}^T\vec{a},
 \end{equation}
 where the components of $b_{u,\sigma}$ are given by
 \begin{equation}
  (b_{u,\sigma})_i = \int_0^1 u(x)\sigma\left(x-\frac{i}{n}\right)dx,
 \end{equation}
 and the mass matrix $M_\sigma$ is given by
 \begin{equation}
  (M_\sigma)_{ij} = \int_0^1 \sigma\left(x-\frac{i}{n}\right)\sigma\left(x-\frac{j}{n}\right)dx.
 \end{equation}
 Further, the matrix $M$ is positive definite.

\end{lemma}
The eigenvalues of the mass matrix $M_\sigma$ play a key role in explaining the spectral bias. Theorems \ref{relu-matrix-theorem} and \ref{hat-matrix-theorem} describe the eigenstructure of $M_\sigma$ for the ReLU activation function and for the Hat activation function, respectively. The detailed proof of Theorem \ref{relu-matrix-theorem} can be found in Appendix \ref{Proofrelu-matrix-theorem} and \ref{spectral}. Theorem \ref{hat-matrix-theorem} can be obtained in a straightforward manner from Theorem \ref{MPhieigen}  in Appendix \ref{spectral}.

\begin{theorem}\label{relu-matrix-theorem}
 Let $\sigma$ be the ReLU activation function and let $\lambda_1\leq\cdots\leq\lambda_n$ denote the eigenvalues of $M_\sigma$. Then we have
 \begin{equation}
  \frac{\lambda_n}{\lambda_j} \sim \frac{n^4}{j^4}.
 \end{equation}
\end{theorem}

\begin{theorem}
\label{hat-matrix-theorem}
 Let $\sigma(x) = \sigma_H(nx)$ be a Hat activation function scaled to match a finite element discretization of $[0,1]$ and let $\lambda_1\leq\cdots\leq\lambda_n$ denote the eigenvalues of $M_\sigma$. Then we have
 \begin{equation}
  \frac{\lambda_n}{\lambda_j} = O(1).
 \end{equation}
\end{theorem}

Now consider training the weights $\vec{a}$ using gradient descent for the objective $L$. To guarantee convergence of gradient descent the learning rate $\eta$ will be taken as $\eta \leq \lambda_n^{-1}$. Theorem \ref{gradient-descent-proposition}, which can be obtained by writing out the details of the gradient descent method, describes the evolution of spectrum of the error.
\begin{theorem}\label{gradient-descent-proposition}
 Consider training the parameters $\vec{a}$ using gradient descent on the loss function $L$ with step-size $s = \lambda_n^{-1}$. Here $\lambda_1\leq \cdots \leq \lambda_n$ are the eigenvalues of the matrix $M_\sigma$. Let $\vec{\psi}^1,...,\vec{\psi}^n$ denote the corresponding eigenvectors and define functions $\phi_j$ by
 \begin{equation}
  \phi_j(x) = \sum_{i=1}^n \psi^j_i\sigma\left(x - \frac{i}{n}\right).
 \end{equation}
 Let $\vec{a}_1,...,\vec{a}_k$ denote the iterates generated by gradient descent. Consider the expansion
 \begin{equation}
  u(x) -f_{NN}(x,\vec{a}_\ell) = \sum_{i=1}^n \alpha_{i}^\ell\phi_i(x).
 \end{equation}
 Then the coefficients $\alpha_{i}^\ell$ satisfy
 \begin{equation}
  \alpha_{i}^\ell = \alpha_{i}^\ell\left(1 - \frac{\lambda_i}{\lambda_n}\right)^\ell.
 \end{equation}
\end{theorem}
We remark that the functions $\phi_j$ appearing in Theorem \ref{gradient-descent-proposition} are orthogonal. This follows since from the definition of $M_\sigma$, we see that
\begin{equation}
\langle \phi_j, \phi_k\rangle_{L^2([0,1])} = (\psi^j)^TM_\sigma (\psi^k) = \lambda_k (\psi^j)^T(\psi^k) = 0,
\end{equation}
where the final equality follows since $\psi^j$ and $\psi^k$ are eigenfunctions of the symmetric matrix $M_\sigma$.

\subsection{Eigenfunction Analysis for ReLU Networks}
Theorem \ref{gradient-descent-proposition} shows that the components of the error corresponding to the large eigenfunction of $M_\sigma$ decay fastest. The following Theorem \ref{relu-eigenfunction-thm}, which can be obtained from Theorem \ref{xi1xn} in Appendix \ref{eigenMPsi}, 
gives the structure of these eigenfunction when $\sigma$ is the ReLU activation function. 
\begin{theorem}\label{relu-eigenfunction-thm}
 Consider the matrix $M_\sigma$ for $\sigma$ to ReLU activation function and let $\phi_j(x)$ be as in Theorem \ref{gradient-descent-proposition}. Then we have
 \begin{equation}
  \frac{\int_0^1 \phi'_j(x)^2 dx}{\int_0^1 \phi_j(x)^2 dx} \sim \frac{n^4}{j^4}.
 \end{equation}
\end{theorem} 
This theorem shows that for the ReLU activation function the eigenfunctions of $M_\sigma$ with large eigenvalue consists of smooth functions, while the eigenfunctions corresponding to small eigenvalues consist of highly oscillatory, high-frequency functions. This can be seen clearly in Figure \ref{relu-eigenfunction-figure}, where we plot both the large and small eigenfunctions $\phi$ corresponding to the ReLU activation function.

\begin{figure*}[!htbp]
	\centering
	\includegraphics[scale=0.2]{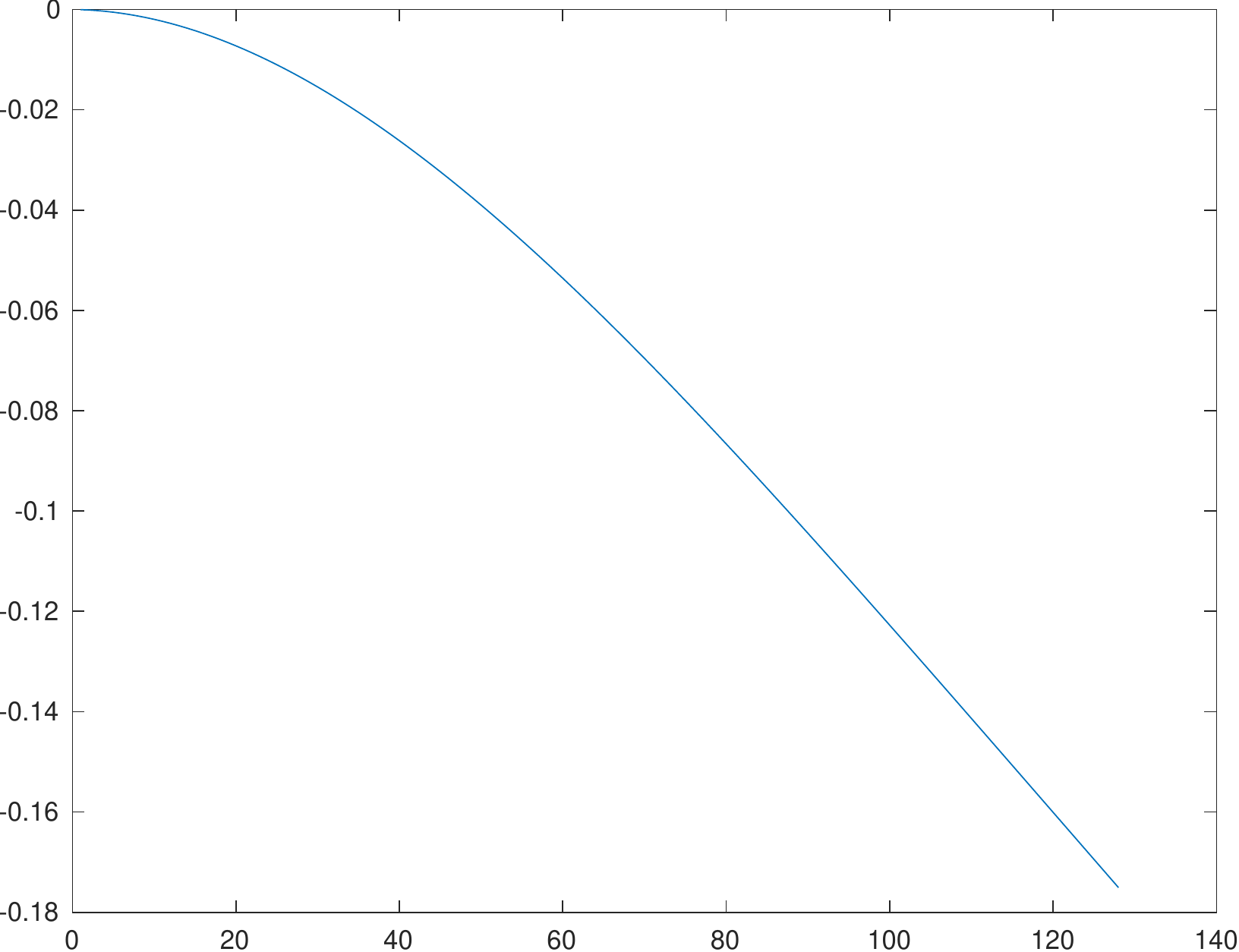}
	\includegraphics[scale=0.2]{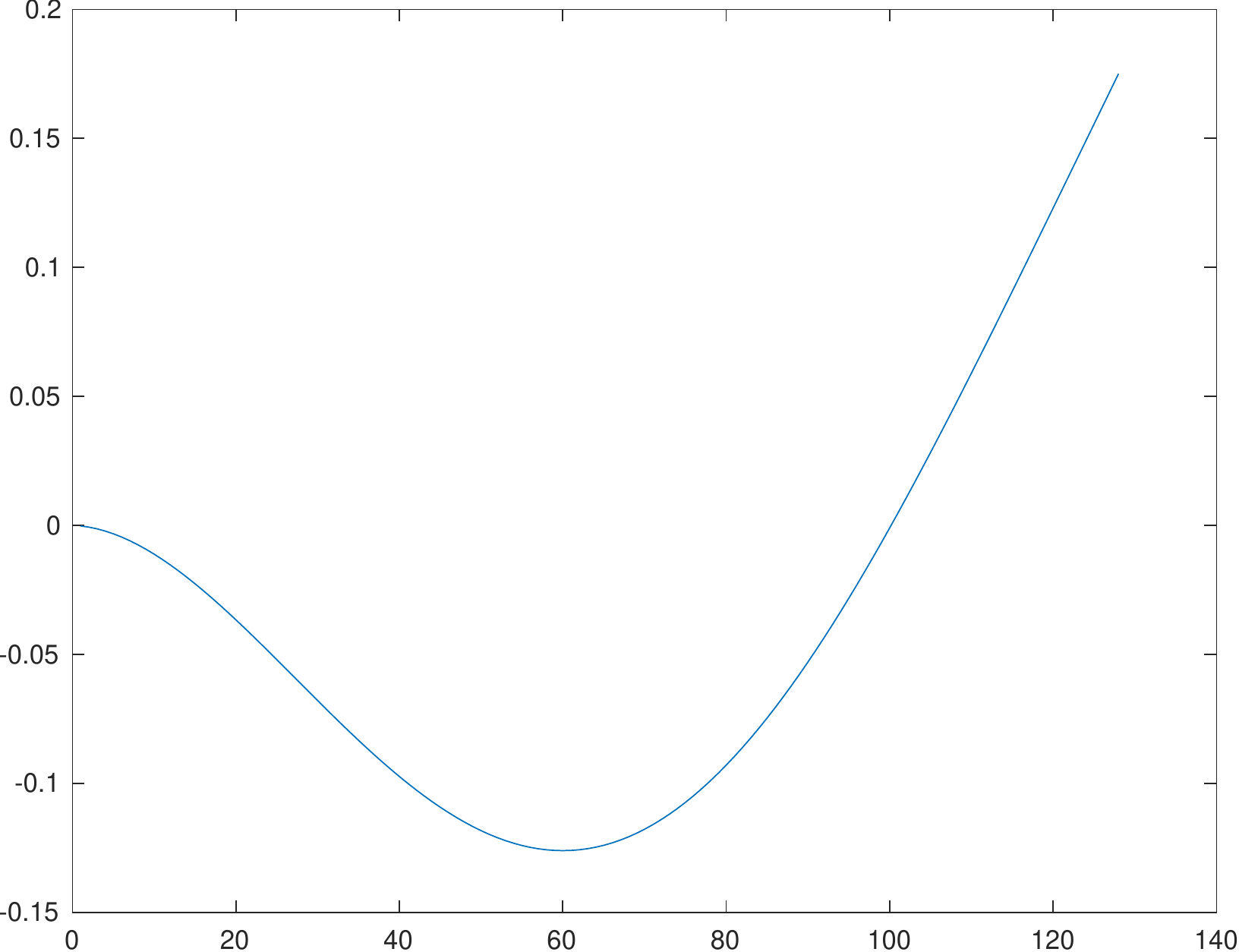}
	\includegraphics[scale=0.2]{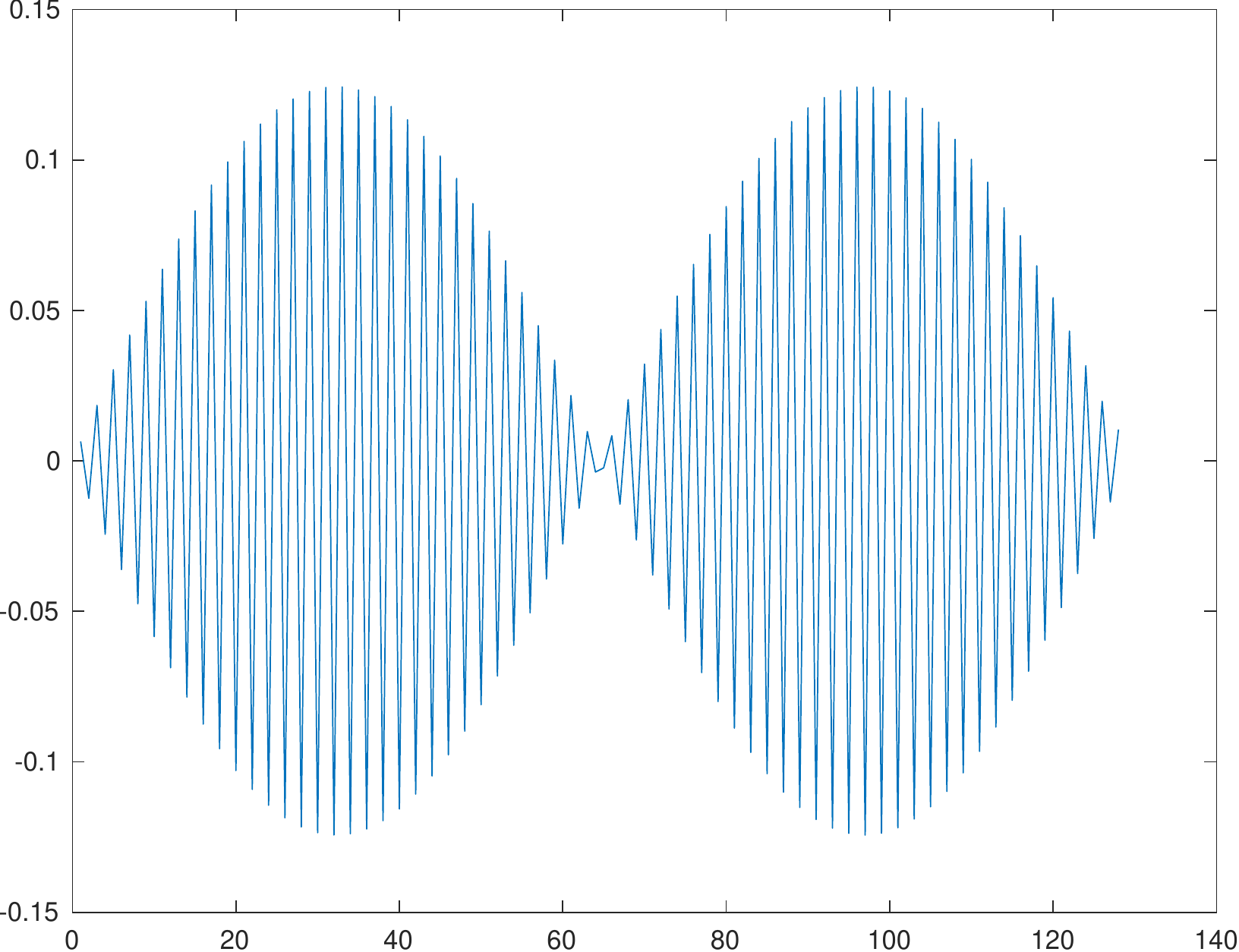}
	\includegraphics[scale=0.2]{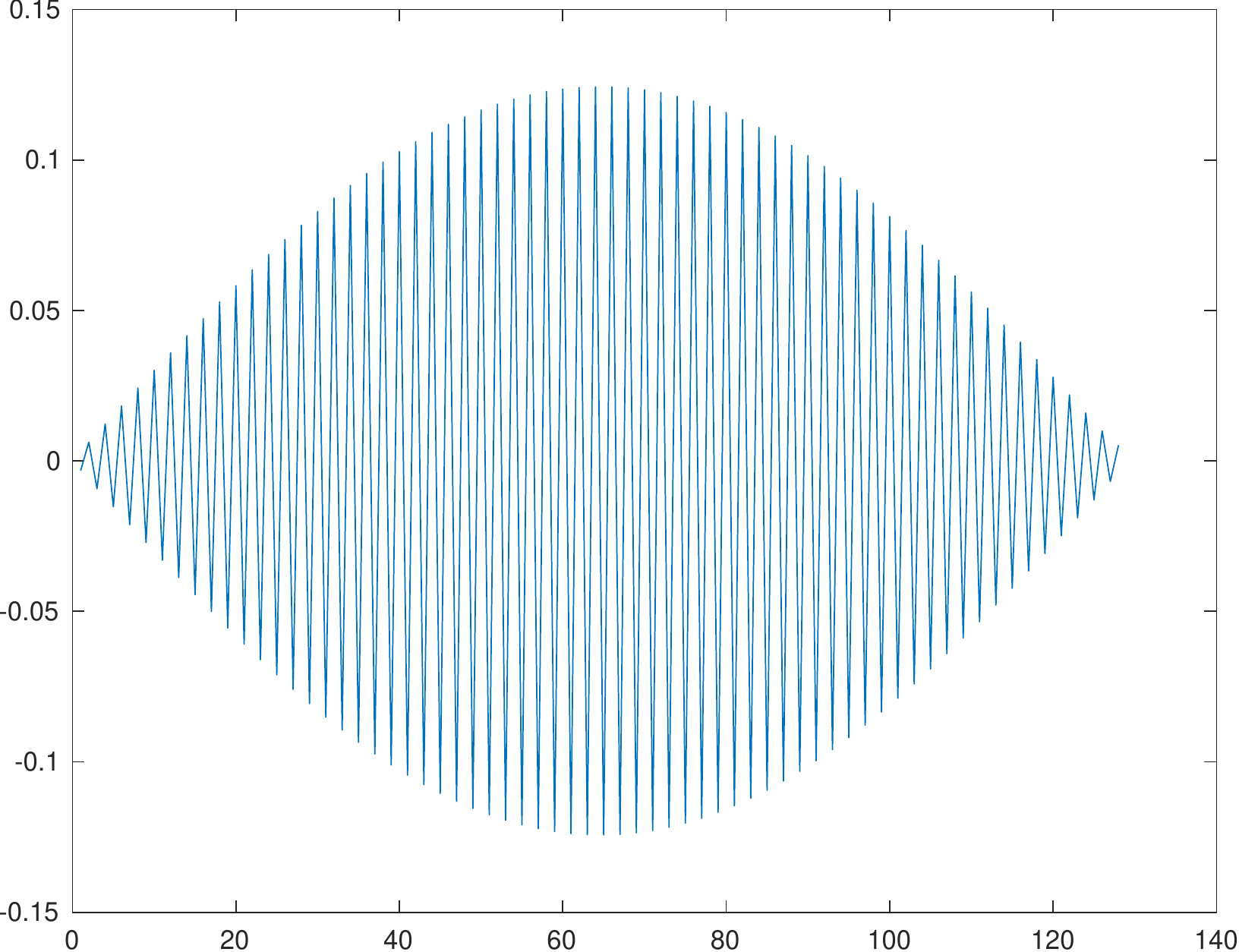}
	\caption{Eigenfunction of $M_\sigma$ when $\sigma$ is a ReLU. Eigenfunctions with the two largest eigenvalues on the left. Eigenfunction with the two smallest eigenvalues on the right.}
	\label{relu-eigenfunction-figure}
\end{figure*}

\subsection{Discussion}
Theorem \ref{gradient-descent-proposition} shows that the eigencomponents of the error decay faster for $\phi_i$ corresponding to large $i$. Moreover, the difference in the convergence rate of different eigencomponents depends upon the ratio $\frac{\lambda_n}{\lambda_1}$. Theorem \ref{relu-matrix-theorem} shows that if the activation function is taken to be the ReLU, then this ratio grows rapidly with the width and the components of the error corresponding to large eigenfunctions decay much faster than the components for small eigenfunction. Further, Theorem \ref{relu-eigenfunction-thm} shows that the large eigenfunction are smooth functions, while the small eigenfunctions are highly oscillatory functions. This means that the high frequency components of the solution are learned much more slowly than the low frequency components. We argue that this is basis behind the spectral bias observed in \cite{rahaman2019spectral}.

On the other hand, Theorem \ref{hat-matrix-theorem} shows that when $\sigma$ is a scaled Hat activation function, then the ratio $\frac{\lambda_n}{\lambda_1}$ remains bounded as the width $n$ increases. In this case the different eigencomponents of the error decay at roughly the same rate. For this reason, we expect that neural networks with Hat activation function will not exhibit the spectral bias observed in \cite{rahaman2019spectral}. 

Finally, we remark that other activation functions can be handled in a similar manner by analyzing the mass matrix $M_\sigma$. For instance, a sinusoidal activation function results in a singular matrix (since the set $\{\sin(x+t),~t\in \mathbb{R}\}$ lies in a two-dimensional linear space) and we have also experimentally observed that such networks also exhibit a spectral bias (see the appendix). Thus, new ideas are required to explain the recent success of sinusoidal representation networks \cite{sitzmann2020implicit}.

\section{Experiments for Shallow Neural Networks}\label{ExperimentsShallow}


In this section, we report the result of two numerical experiments in which we use shallow neural networks to fit target functions in $\mathbb R^d (d=1,2)$. In all experiments in this section, we minimize the mean square error (MSE) using the Adam optimizer \cite{kingma2014adam}. We consider three different activation functions
\begin{equation}\label{activation-functions-experiments}
\sigma(p)=\left\{
\begin{aligned}
		&{\rm tanh}(p) ;\\
		 &{\rm ReLU}(p) ;\\
		&{\rm Hat}(p).\\
	\end{aligned}\right.
\end{equation}
Based upon the results of the following two experiments, we observe that: 
\begin{conclusion}\label{conclusionA}
The spectral bias holds for both tanh and ${\rm ReLU}$ shallow neural networks, but it does not hold for Hat shallow neural networks.  
\end{conclusion}

\begin{experiment}
\normalfont
We use shallow neural networks with each of the activation functions in \eqref{activation-functions-experiments} to fit the following target function $u$ on $[-\pi, \pi]$:
\begin{equation}
u(x) = \sin(x) + \sin(3x)+\sin(5x).
\end{equation}
For each activation function our network has one hidden layer with size 8000.
The mean square error (MSE) loss is given by
\begin{equation}\label{MSEloss}
L(f,u) = \frac1N\sum_{i=1}^{N} (f(x_i)-u(x_i))^2,
\end{equation}
where $x_i$ is a uniform grid of size $N=201$ on $[-\pi,\pi]$. The three networks
are trained using ADAM with a learning rate of 0.0002. When using a tanh or ReLU activation function, all parameters are 
initialized following a Gaussian distribution with mean 0 and 
standard deviation 0.1, while the network with Hat activation function is 
initialized following a Gaussian distribution with mean 0 and 
standard deviation 0.8. 

Denote 
\begin{equation}\label{delta-definition-experiments}
\Delta_{f,u} (k) = \left| \hat{f}_k - \hat{u}_k 
\right|/\left|\hat{u}_k \right|,
\end{equation}
where $k$ represents the frequency, $\left| \cdot \right|$ represents the 
norm of a complex number and $\hat{\cdot}$ represents discrete Fourier 
transform. We plot $\Delta_{f,u} (1)$, $\Delta_{f,u} (3)$, 
and $\Delta_{f,u} (5)$ in Figure \ref{tanhvshat} for each of the three networks.
\begin{figure}[!htbp]
 \centering
 \includegraphics[scale=0.28]{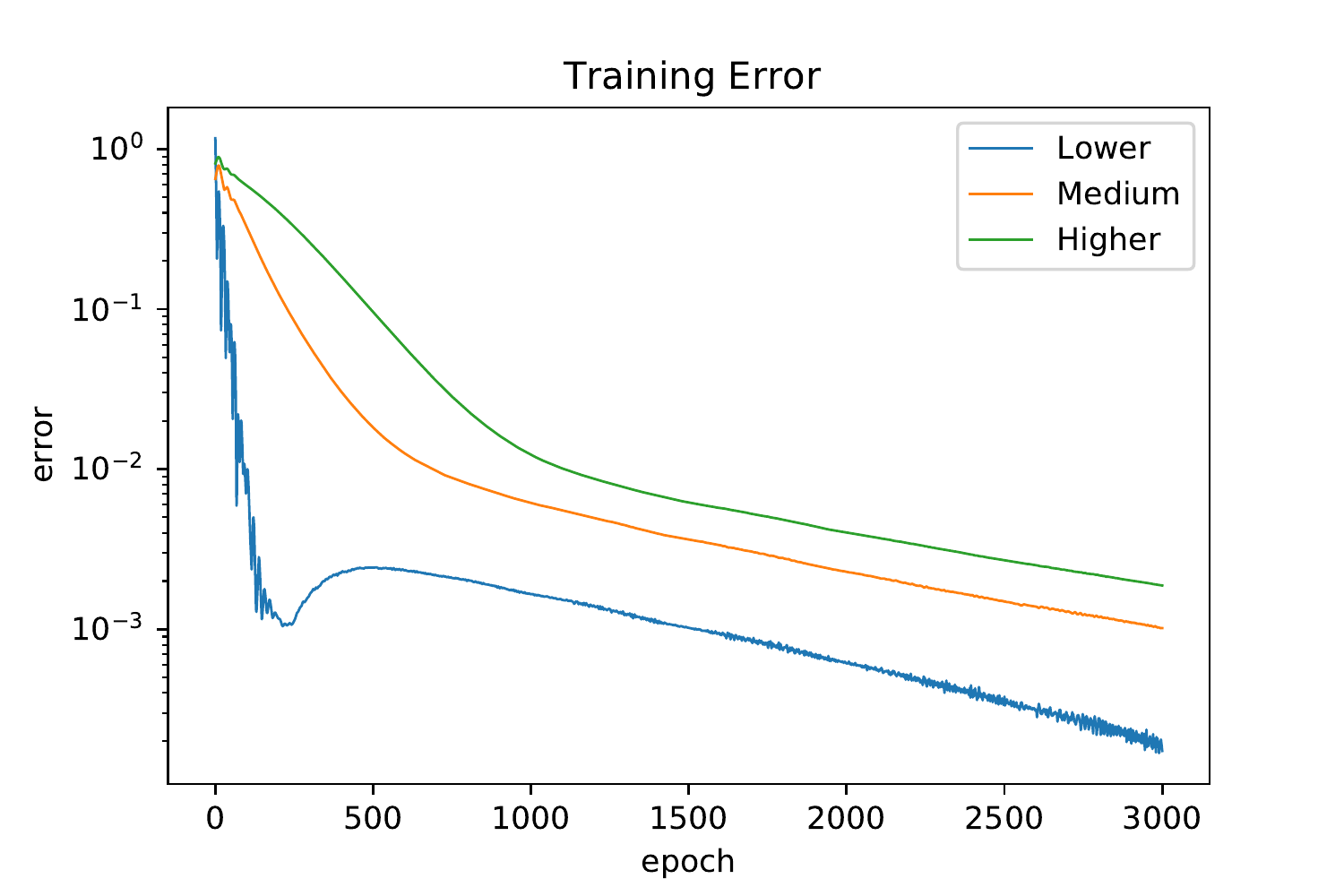}\hspace{-8pt}
  \includegraphics[scale=0.28]{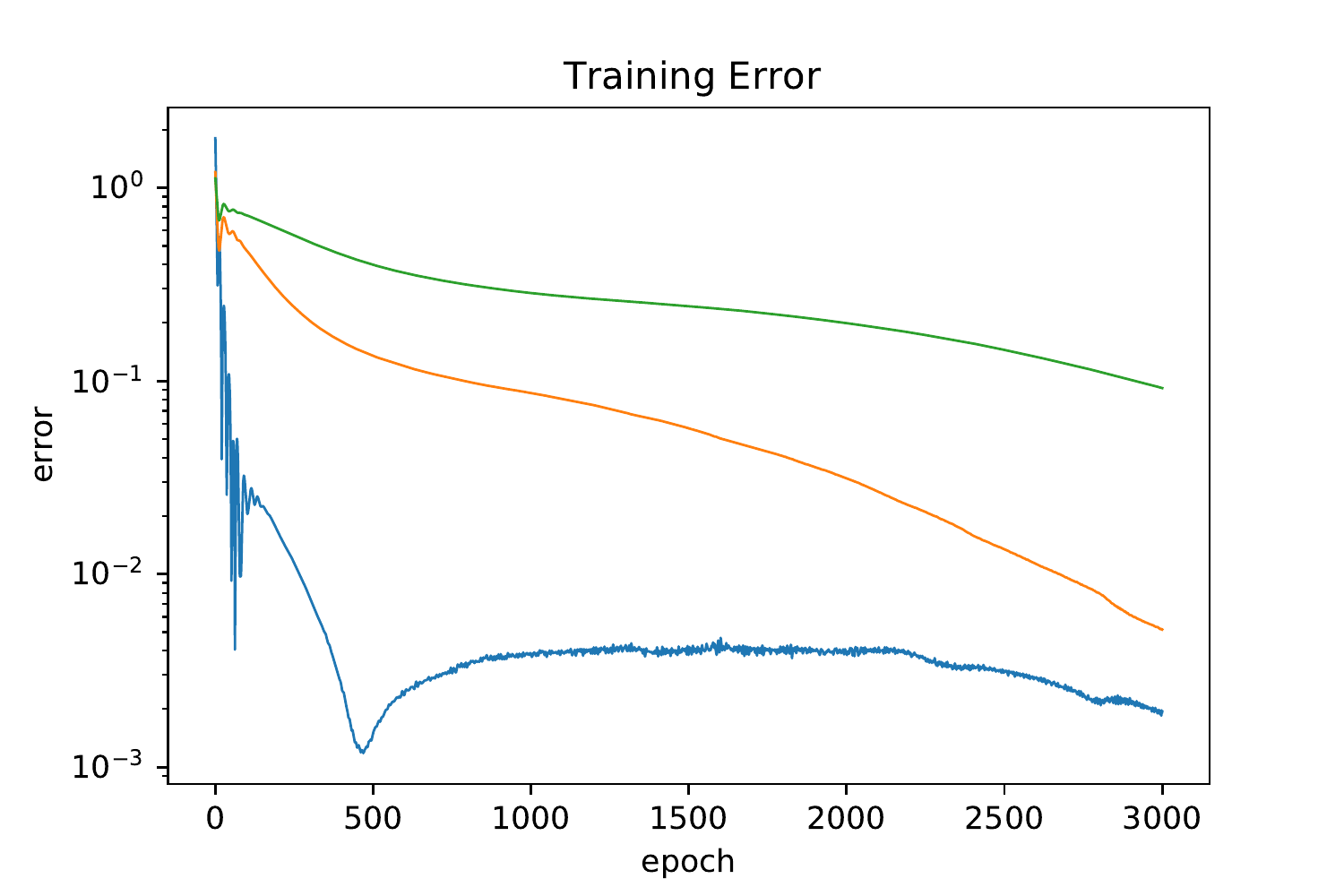}\hspace{-8pt}
 \includegraphics[scale=0.28]{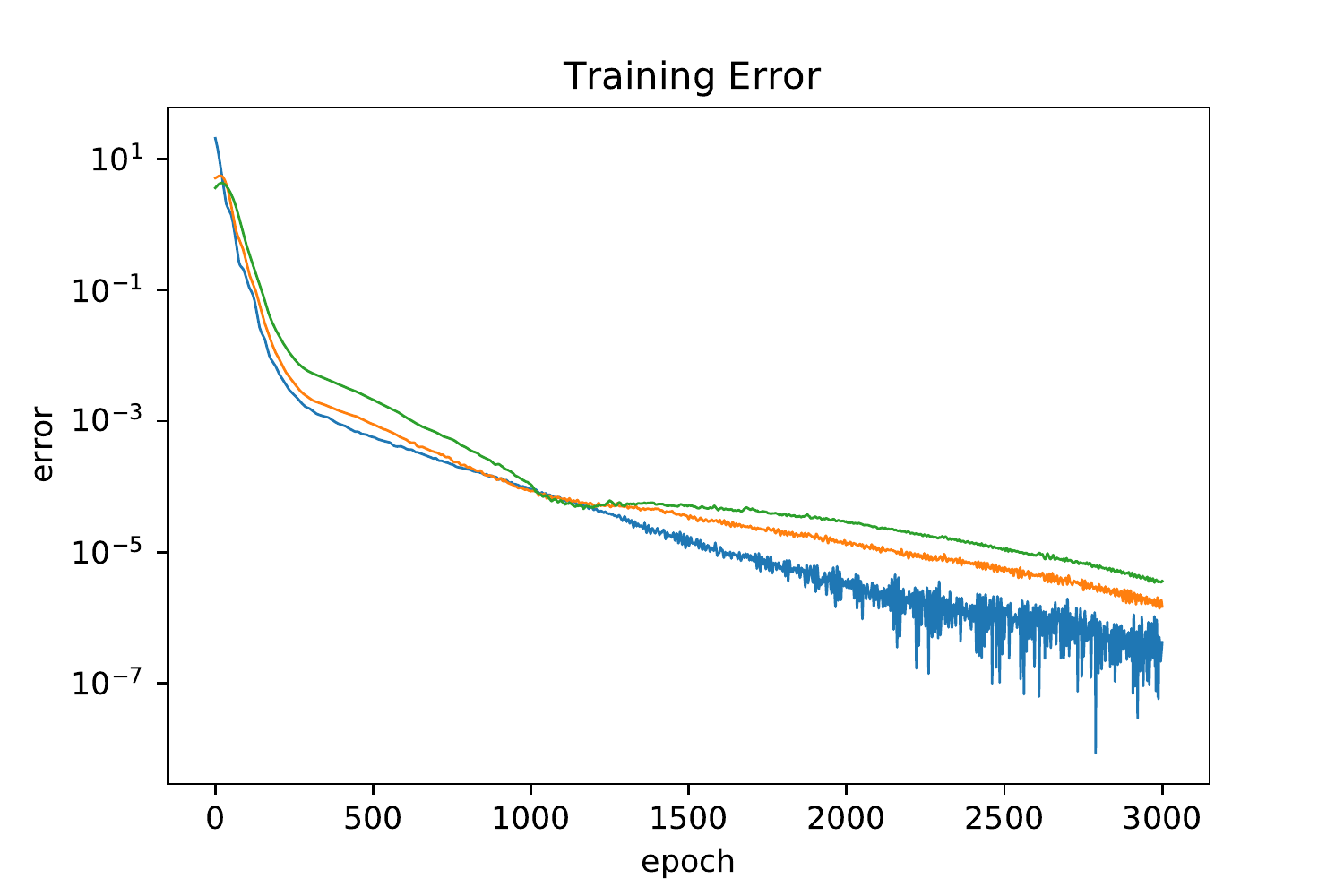}
 	\vspace{-3pt}
	\caption{$\sigma(p)=\text{tanh}(p)$ (left), $\sigma(p)=\text{ReLU}(p)$ (middle), $\sigma(p)=\text{ Hat}(p)$ (right)}
	\label{tanhvshat}
\end{figure}
From these results, we 
observe the spectral bias for tanh  \cite{xu2019frequency} and ReLU neural networks (see left and middle of Figure \ref{tanhvshat}), 
while there is no spectral bias for the Hat neural network as shown in right of Figure \ref{tanhvshat}. 
\end{experiment}

\begin{experiment}\label{2Ddirect-ex}
\normalfont
Next, we fit the target function $u$ on $[0,1]^2$ given by:
\begin{equation*}
u(\boldsymbol x) = \sin(2 \pi x_1)\sin(2 \pi x_2) + \sin(10 \pi x_1)\sin(10 \pi x_2).
\end{equation*} 
Here $\boldsymbol x=(x_1,x_2)\in [0,1]^2$.  For this experiment, we use shallow neural networks with the ReLU $\sigma(p)=\text{ReLU}(p)$ and with a scaled Hat function $\sigma(p)={\rm  Hat}(100p)$ as activation function.
Both models have width 10000 and all parameters are initialized using the default initialization in Pytorch. We use the ADAM optimizer for both experiments. The learning rate for the ReLU neural network is initialized to 0.00005 and and the learning rate for the Hat neural network is initialized to 0.00075. Both learning rates are decayed by half every 1000 epochs. 

\vspace{-8pt}
\begin{figure}[!htbp]
		\centering
\includegraphics[scale=0.30]{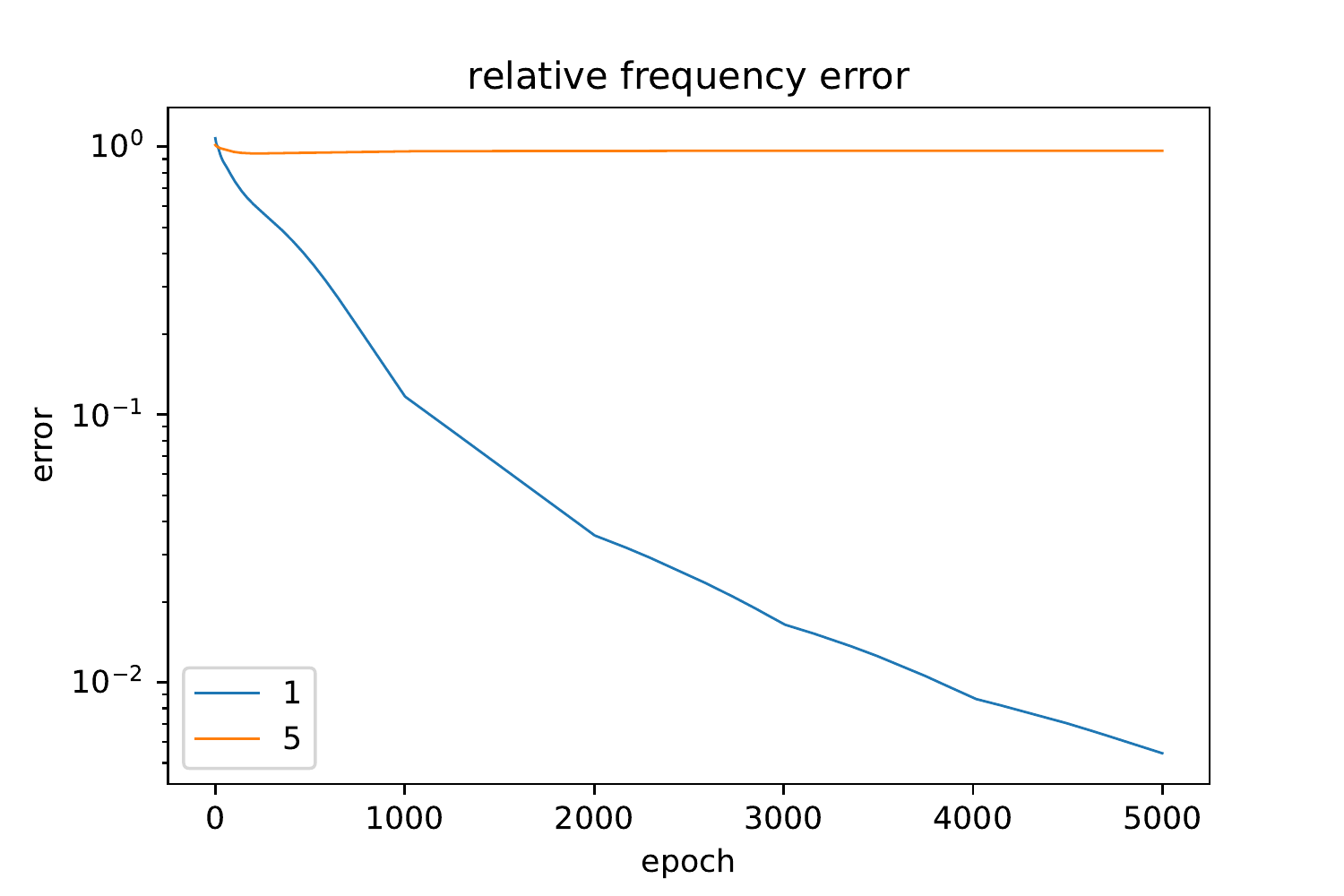}\hspace{-10pt}
\includegraphics[scale=0.30]{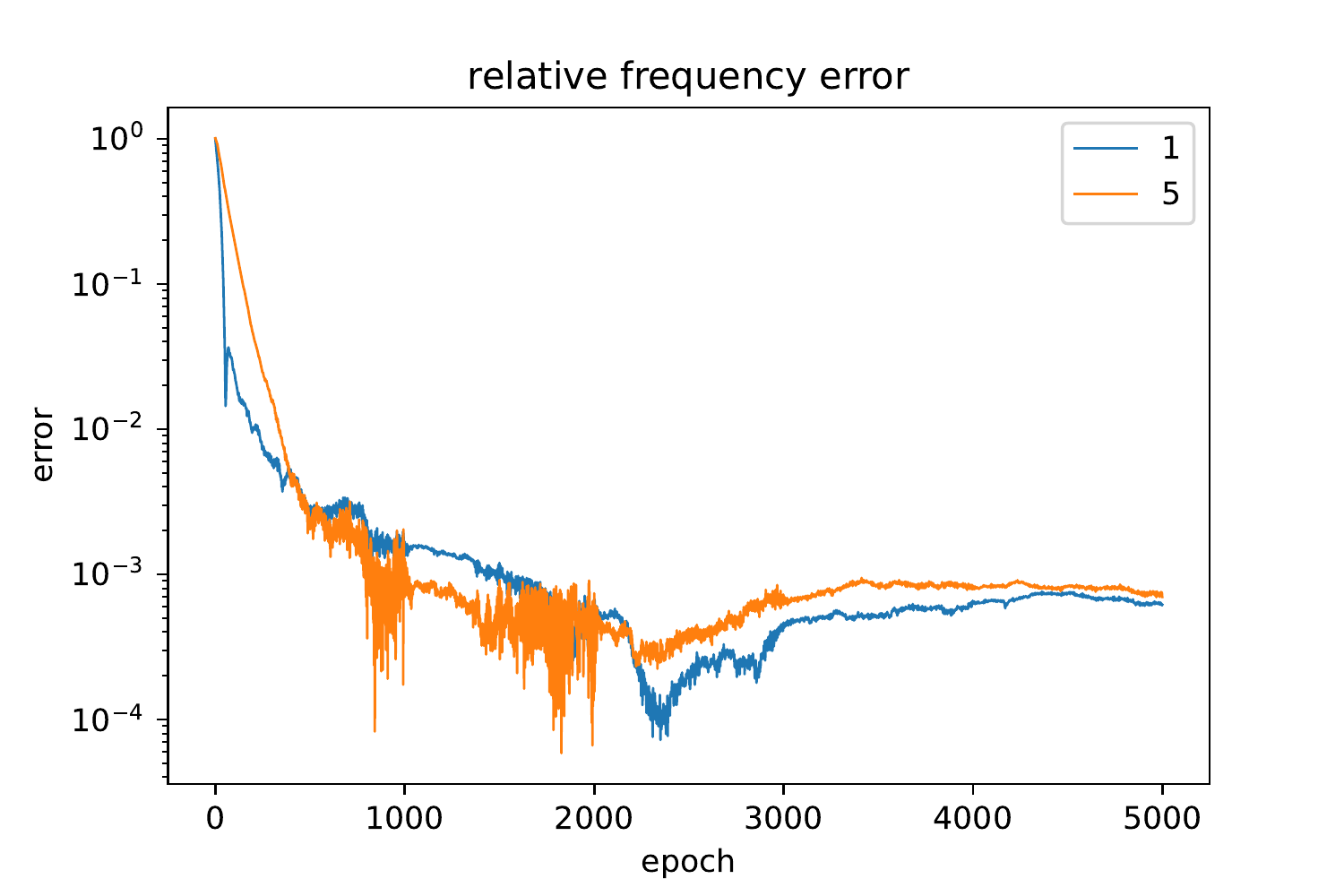}
\vspace{-3pt}
		\caption{$\sigma(p)=\text{ReLU}(p)$ (left), $\sigma(p)=\text{ Hat}(100 p)$ (right)}
		\label{2Ddirect}
\end{figure}
\vspace{-2pt}
To observe the spectral bias, we look at a slice of the target and network functions, defined by setting $x_2 = 31/128$. For this one dimensional function, we plot $\Delta_{f,g} (1)$ and $\Delta_{f,g} (5)$ for both neural network models with ReLU and Hat activation functions. Here $\Delta_{f,g}$ is defined in the same way as in equation \eqref{delta-definition-experiments}.
\end{experiment}

\section{Experiments for Deep Neural Networks}\label{ExperimentsDeep}

\subsection{Experiments with Synthetic Data}\label{ExperimentsDeep1}

\begin{experiment}\label{1D10fre}
\normalfont
The experimental setup here is an extension of the experiment presented in \cite{rahaman2019spectral}. The target function is:
\begin{align}
	u(x) = \sum_{k=1}^{10} \sin{(10\pi kx +  c_k)},
\end{align}
where $c_k, i=1,2,\cdots, 10$ are sampled from the uniform distribution $ U(0, 2\pi)$. We fit this target function using a squared error loss sampled at 200 equally spaced points in [0, 1] using a neural network with 6 layers. Three networks are trained, one with a ReLU activation function and two with the Hat activation function. We train the models with Adam using the following hyperparameters:
\begin{itemize}
\item ReLU activation function with 256 units per layer: all parameters are initialized from a Gaussian distribution $\mathcal{N} (0, 0.04)$. The learning rate is 0.001 and decreased by half every 10000 epochs. This corresponds to the experiment in \cite{rahaman2019spectral}.
\item Hat activation function with 256 units per layer: all parameters are initialized from the uniform distribution $\mathcal{U} (-1.0, 1.0)$. The learning rate is 0.00001 and is decreased by half every 250 epochs. 
\item Hat activation function with 128 units per layer:  all parameters are initialized from the uniform distribution $\mathcal{U} (-2.0, 2.0)$. The learning rate is 0.0001 and decreased by half every 200 epochs. 


\end{itemize}
\vspace{-10pt}
\begin{figure*}[!htbp]
	\centering
	\includegraphics[scale=0.28]{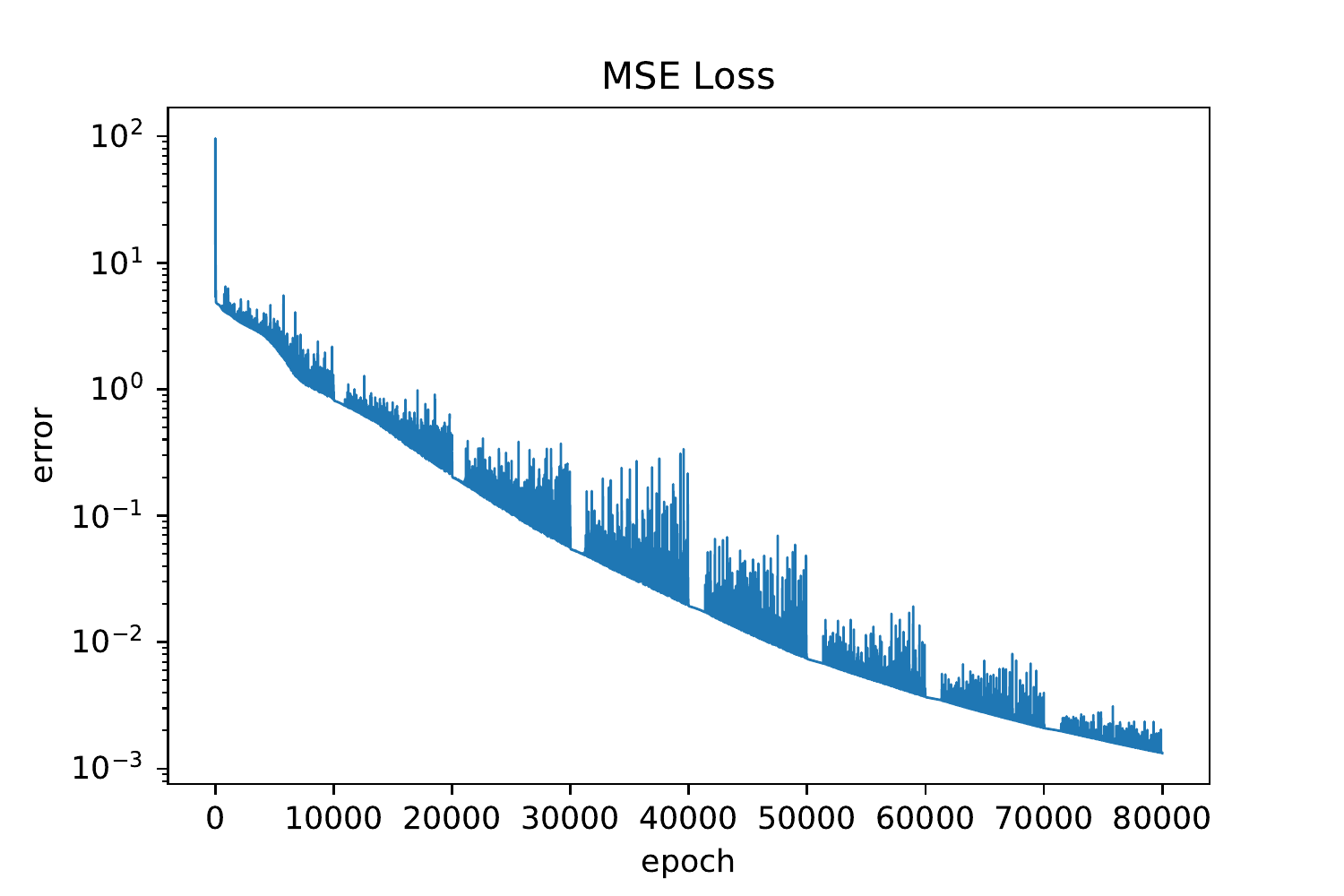}\hspace{-11pt}
	\includegraphics[scale=0.28]{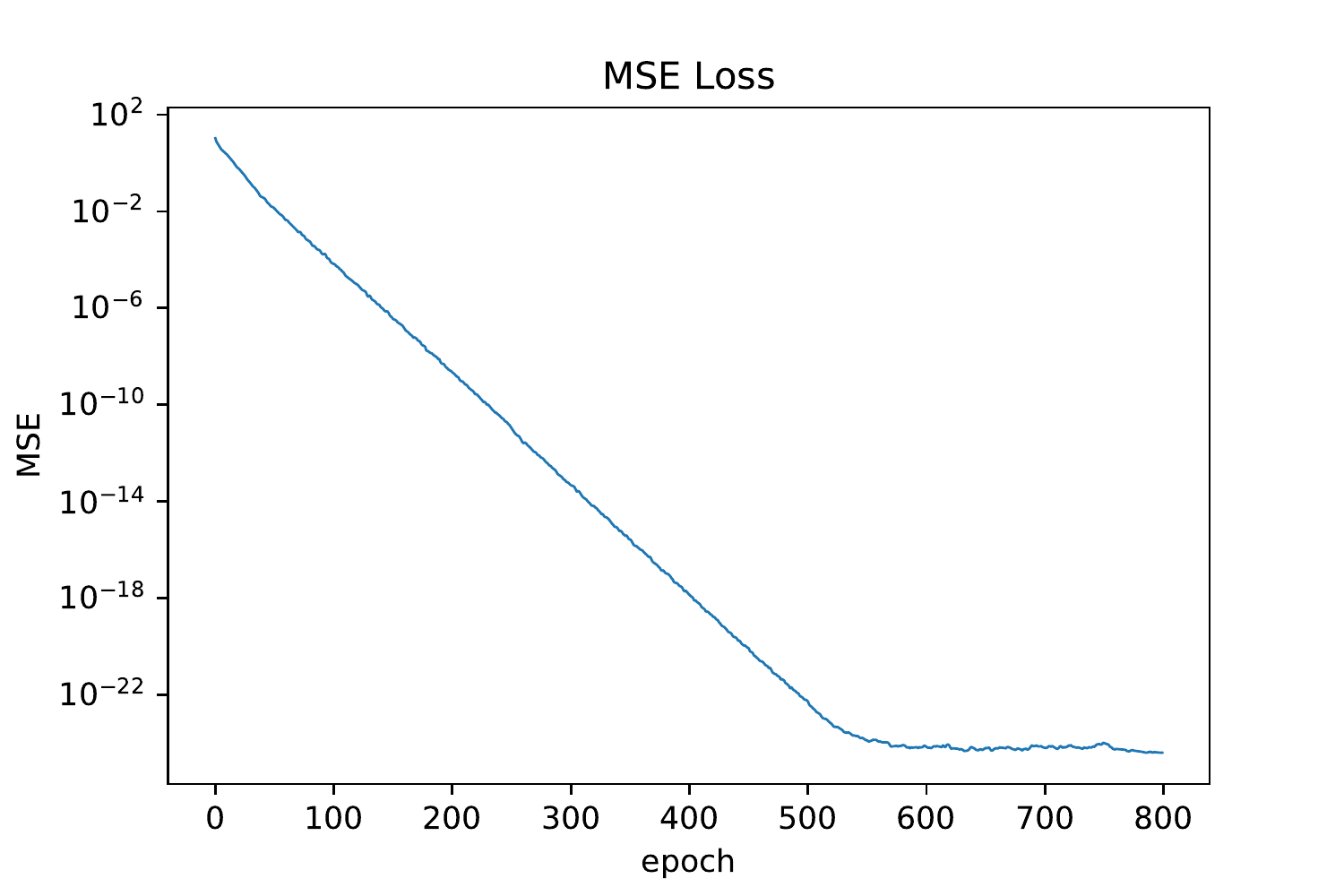}\hspace{-11pt}
	\includegraphics[scale=0.28]{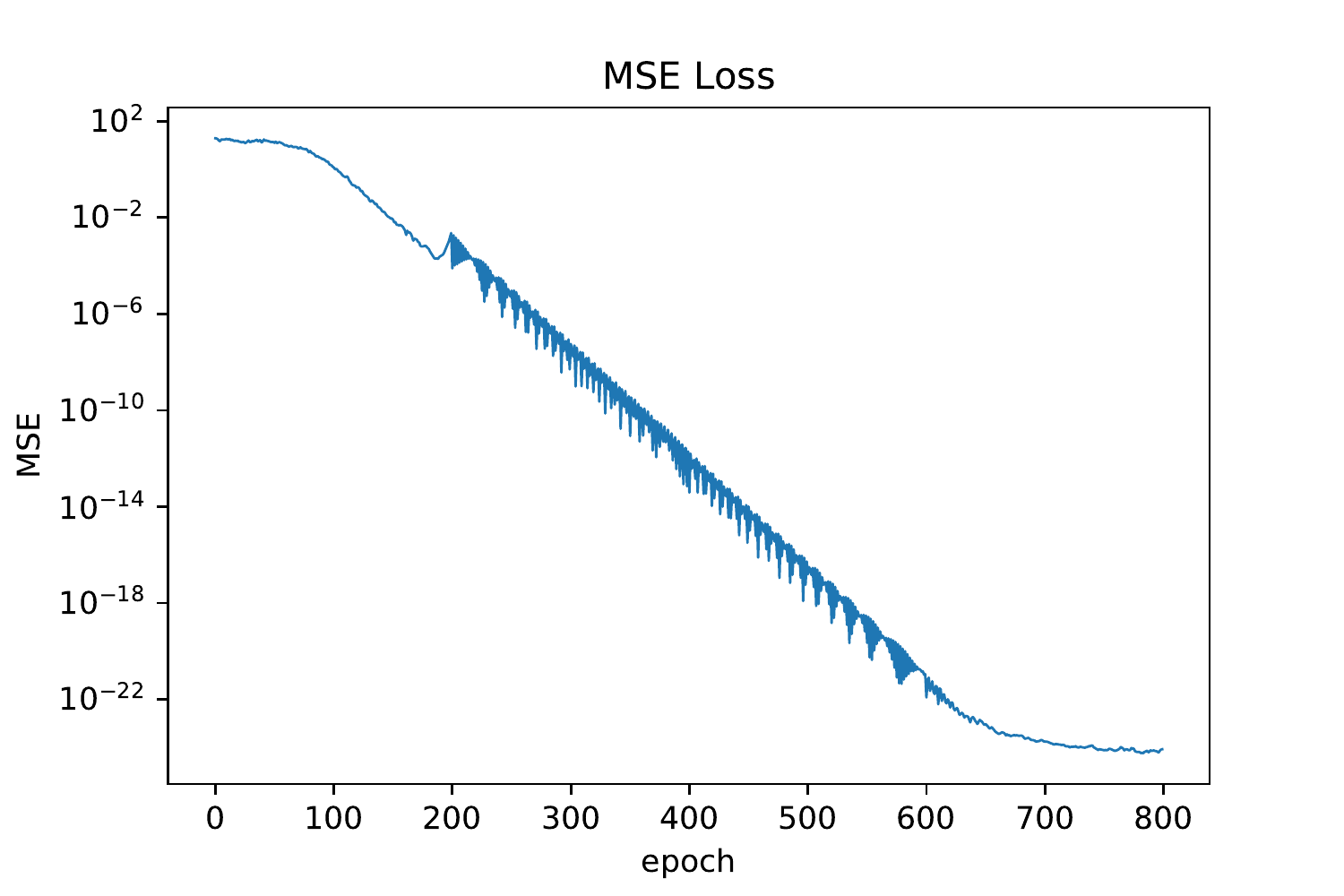}\\
	\includegraphics[scale=0.28]{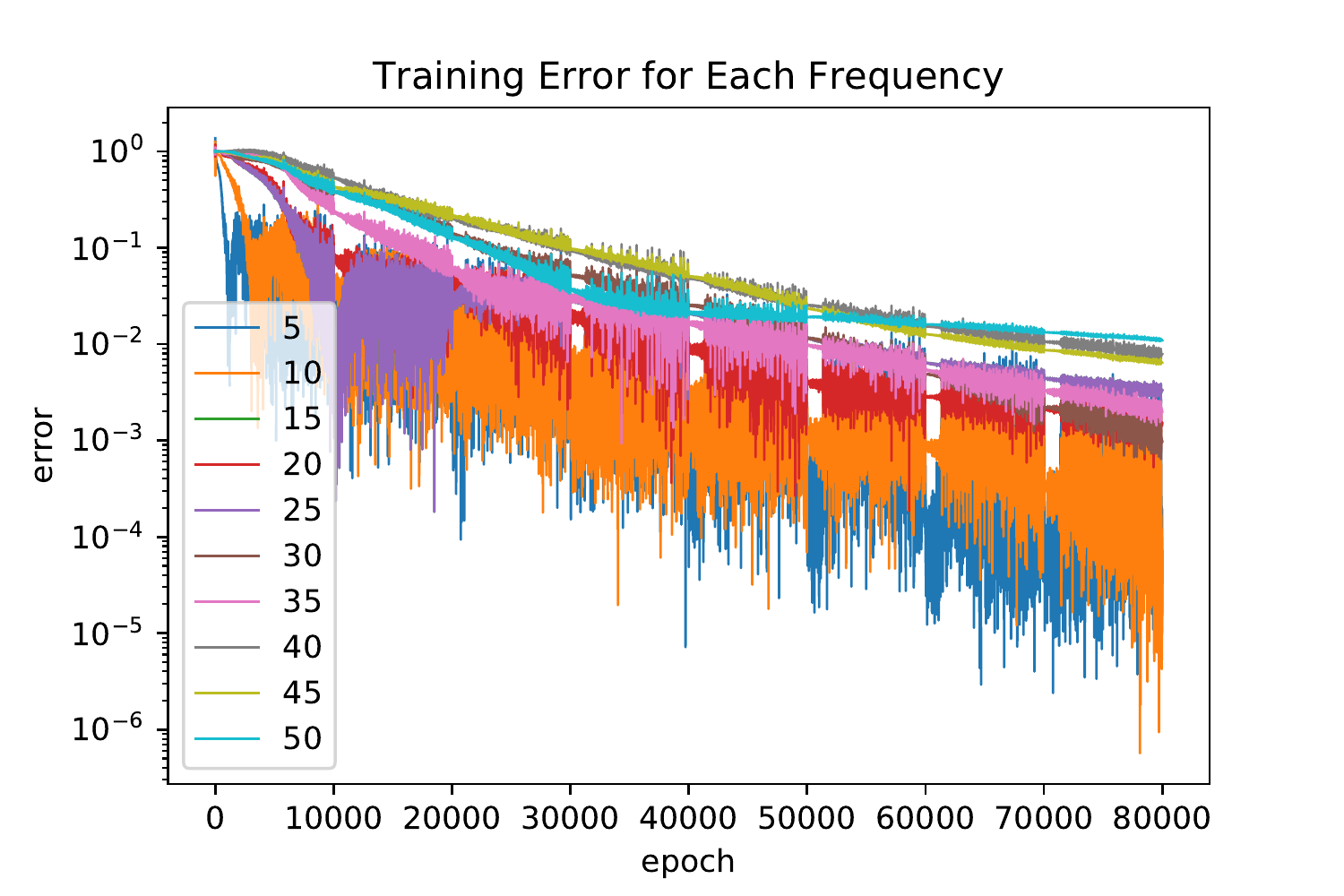}\hspace{-11pt}
	\includegraphics[scale=0.28]{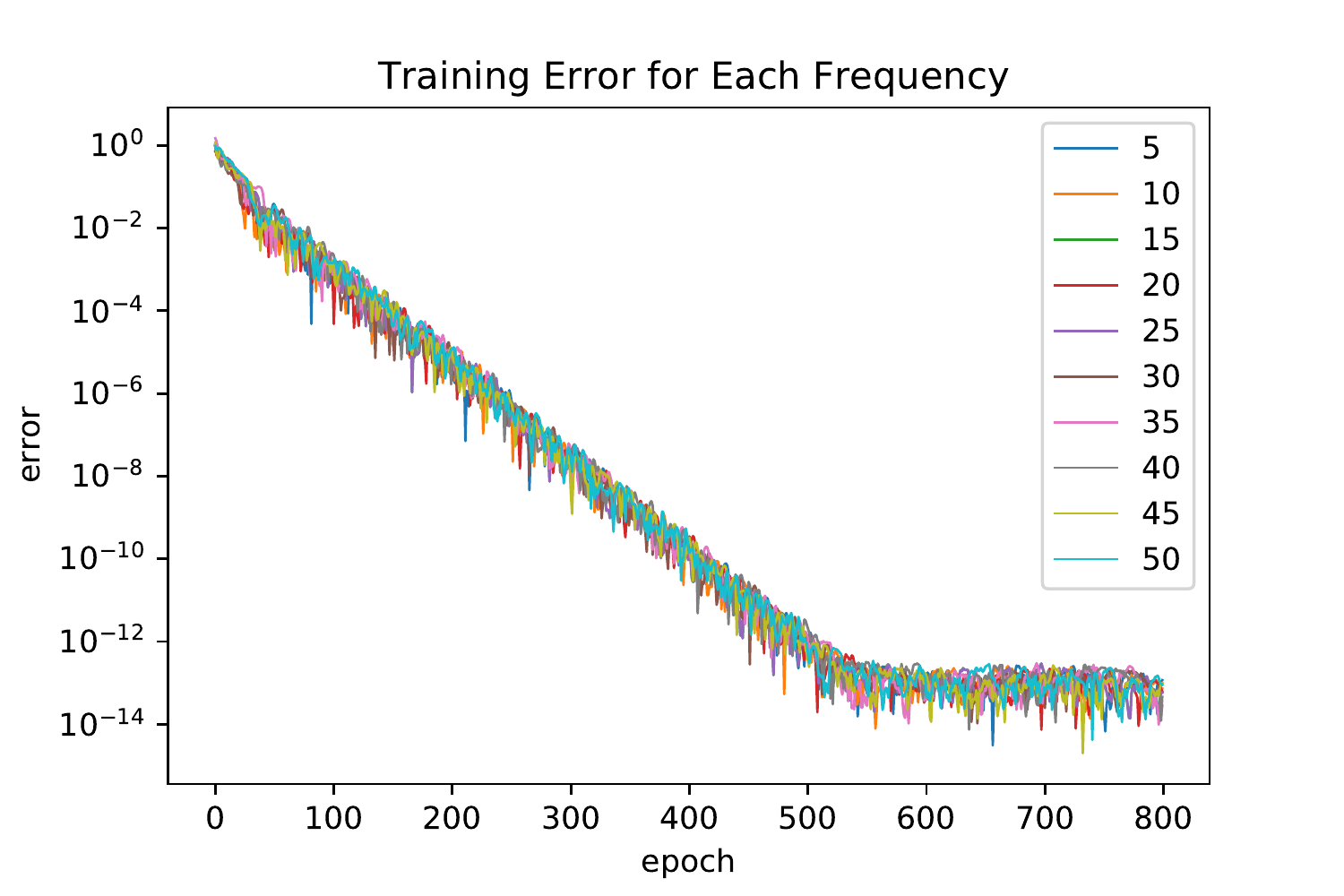}\hspace{-11pt}
	\includegraphics[scale=0.28]{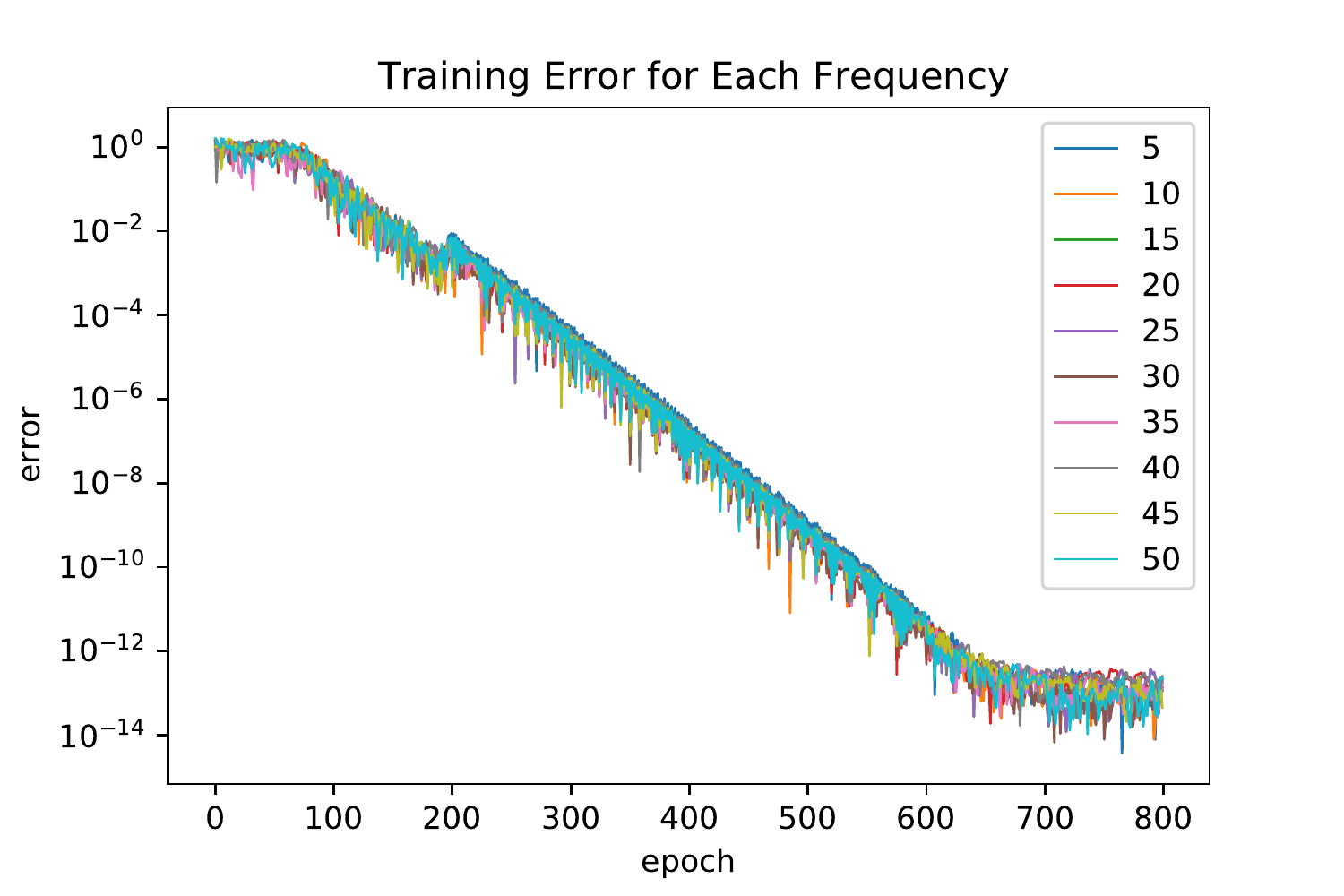} 		
	\caption{ReLU, width 256 \cite{rahaman2019spectral} (left), Hat activation, width 256 (middle), Hat activation, width  128 (right).}
	\label{1DDNN-2}
\end{figure*}

\begin{figure*}[!htbp]
	\centering
	\includegraphics[scale=0.28]{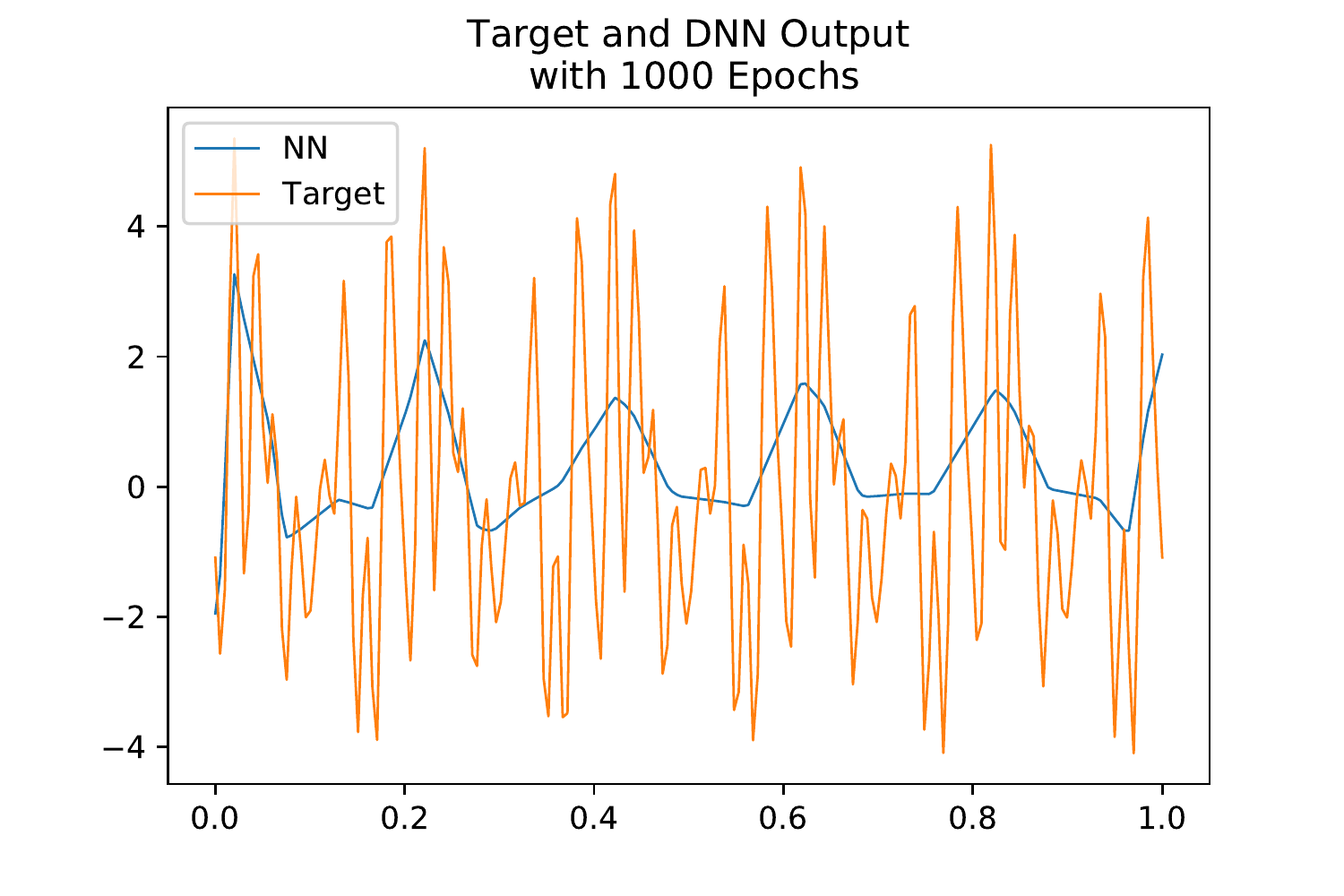}\hspace{-11pt}
	\includegraphics[scale=0.28]{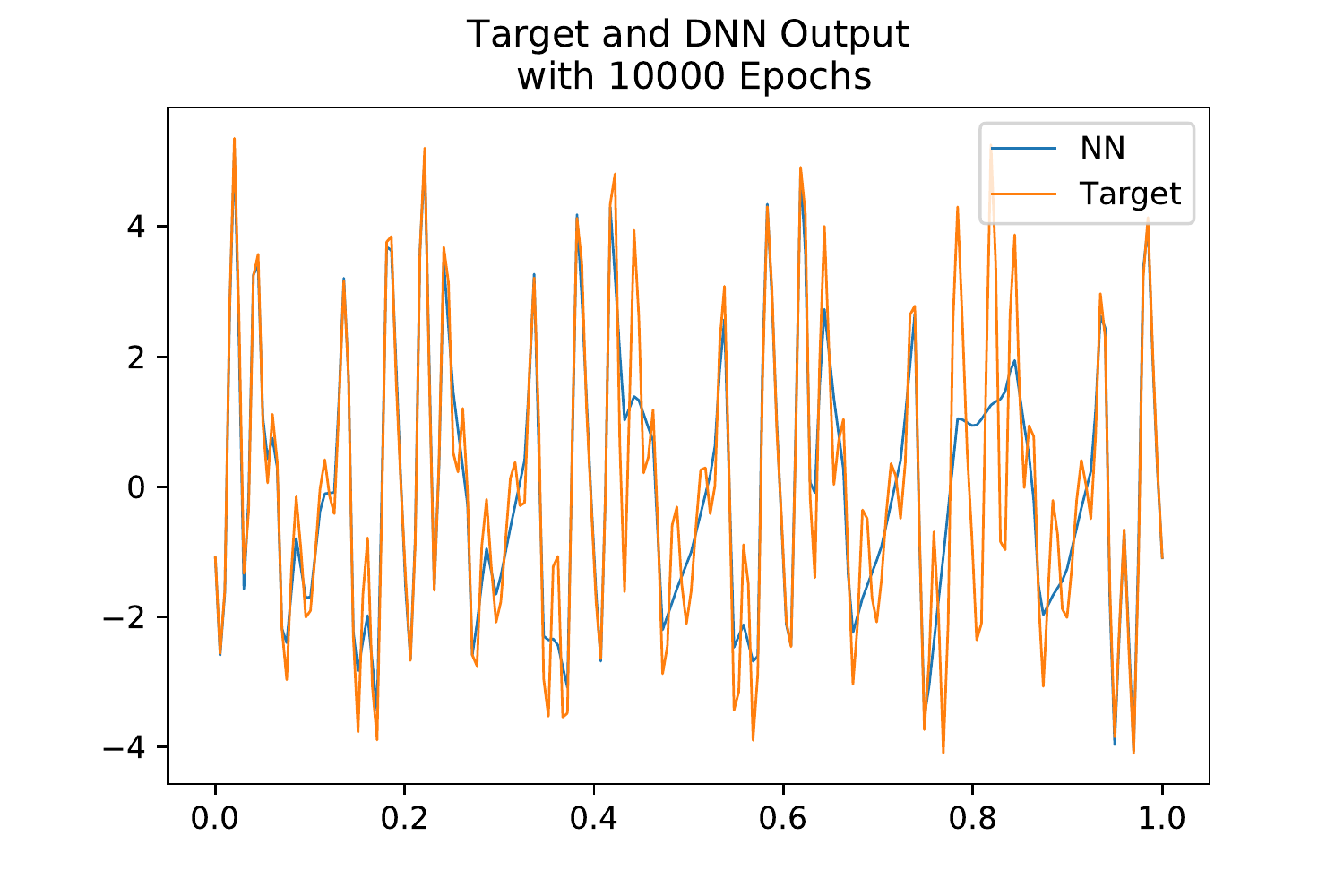}\hspace{-11pt}
	\includegraphics[scale=0.28]{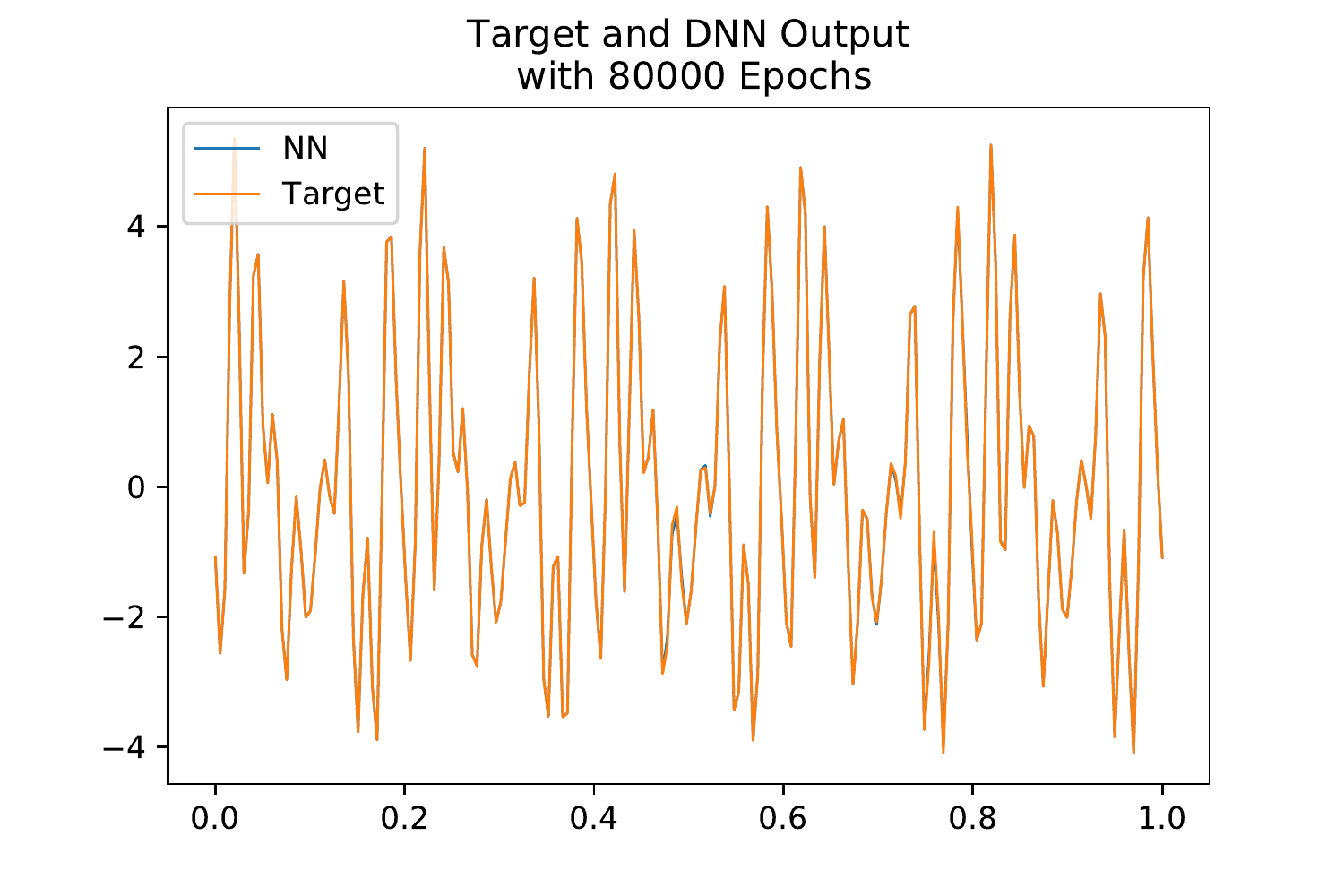}	
	\\
	\includegraphics[scale=0.28]{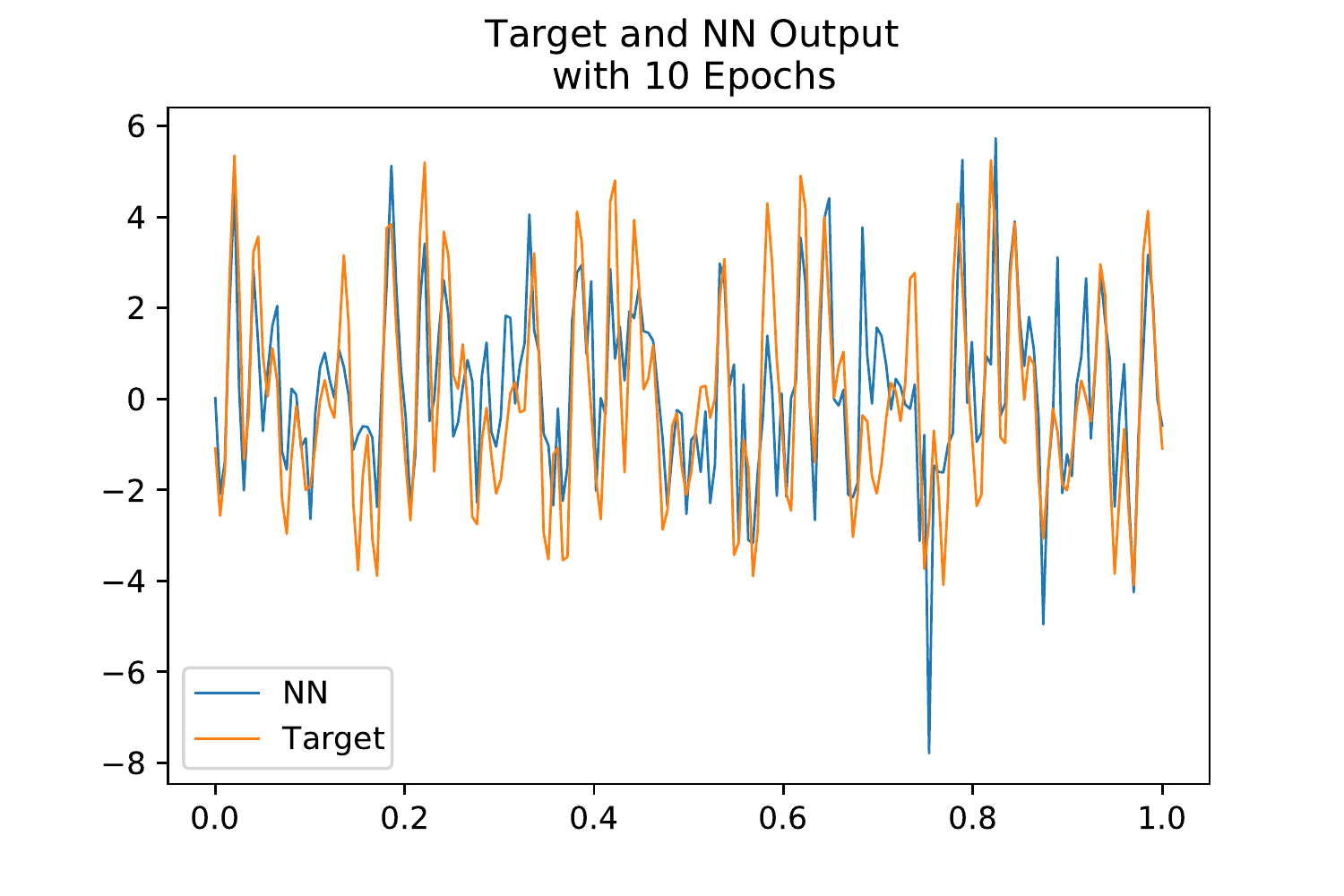}\hspace{-11pt}
	\includegraphics[scale=0.28]{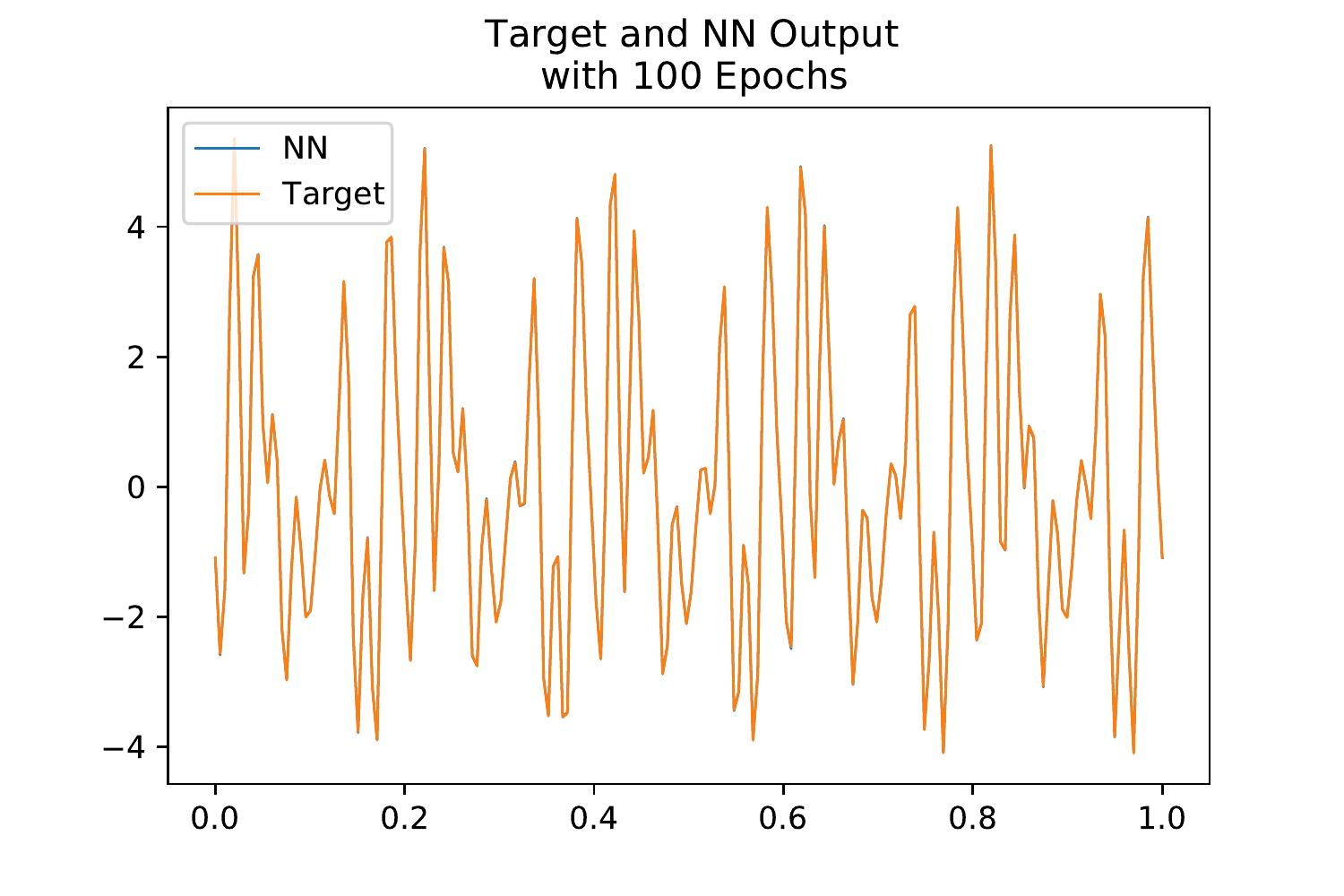}\hspace{-11pt}
	\includegraphics[scale=0.28]{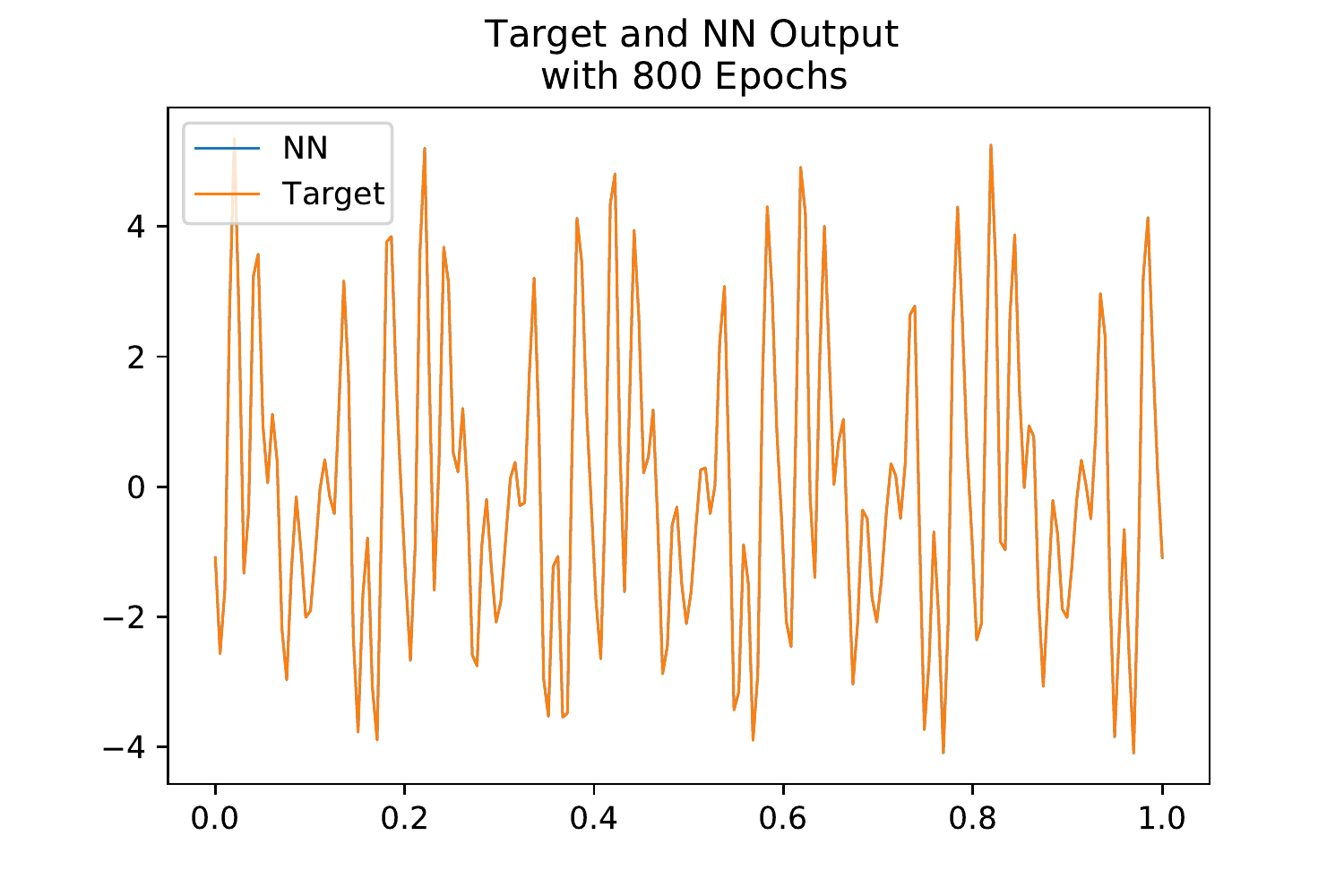} 	
	\caption{ReLU, width 256 \cite{rahaman2019spectral} (Top), Hat activation, width 256(Bottom).}
	\label{1DDNN-1}
\end{figure*}

From Figure \ref{1DDNN-2}, we see that the spectral bias is significant for ReLU deep neural networks (DNN). Some frequencies, for example 50, will not even converge below an error of $10^{-2}$. Changing the activation function to a linear Hat function removes the spectral bias and the loss decreases rapidly for all frequencies. In fact, the training loss for 
ReLU-DNN is only about $10^{-3}$ after 80000 epochs, while the training loss for Hat neural networks are are already $10^{-22}$ after only 800 epochs. As shown in Figure \ref{1DDNN-1}, Hat neural networks fit the target function much better and faster than ReLU networks. 
\end{experiment}

\begin{experiment}\label{imagefitting}
\normalfont
	In this experiment, we consider fitting a grayscale image using deep neural networks. We view the grayscale image on the left of Figure \ref{picture_fit} as two dimensional function and fit this function using a deep neural network with hidden layers of size 2-4000-500-400-1. Our loss function is the squared error loss. We consider using both the ReLU and Hat neural networks and compare their performance. For both networks, we initialize all parameters from a normal distribution with standard deviation 0.01. For the Hat network, the learning rate is 0.0005 and is reduced by  half every 1000 epochs, while the learning rate for ReLU model is 0.0005 and reduced by half every 4000 epochs. From Figure \ref{picture_fit}, we can see that the image is fit much better using the Hat network, which is able to capture the high frequencies in the image, while the ReLU network blurs the image due to its inability to capture high frequencies.
	
\begin{figure*}[!htbp]
	\centering
        \includegraphics[scale=0.38]{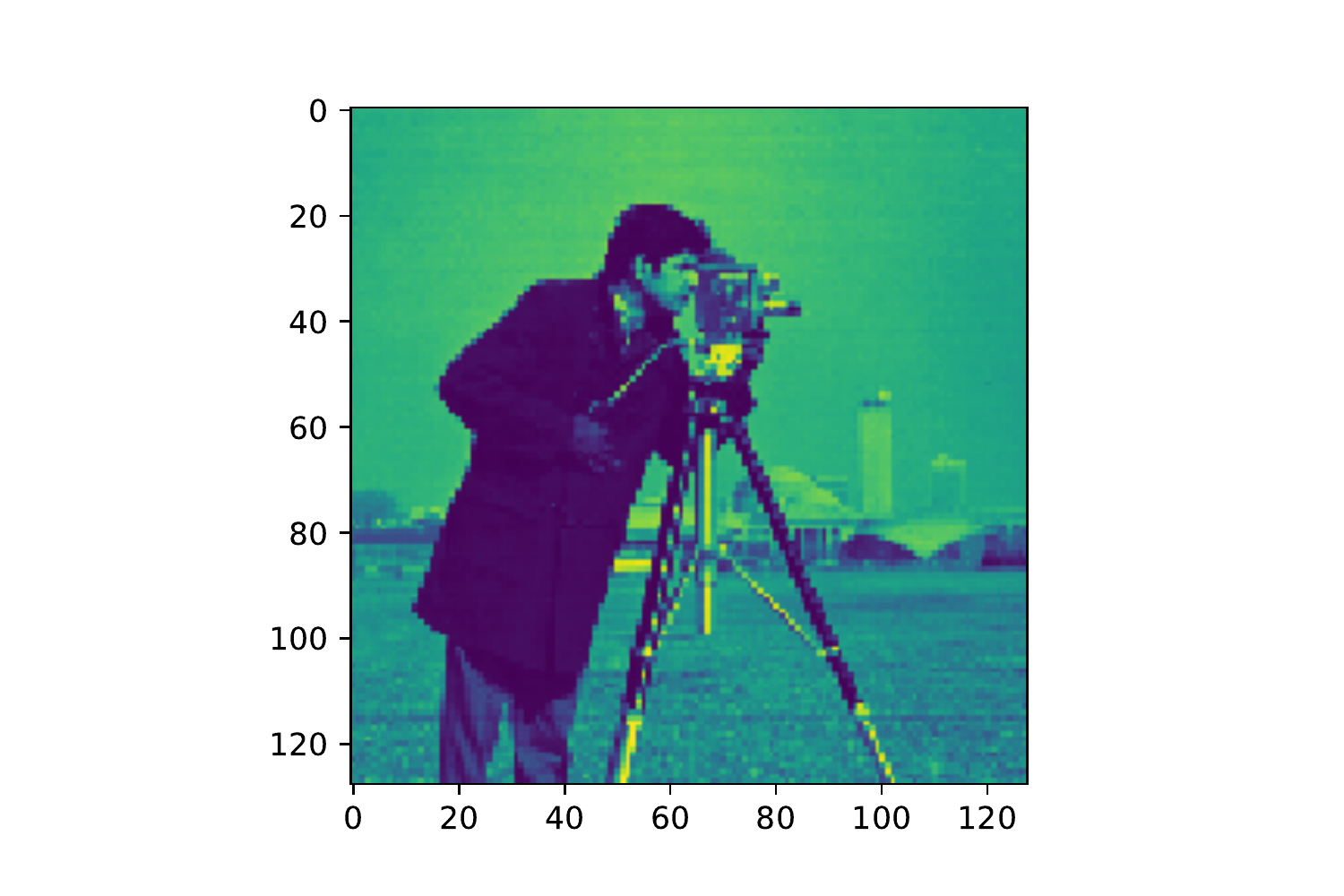}\hspace{-55pt}
        \includegraphics[scale=0.38]{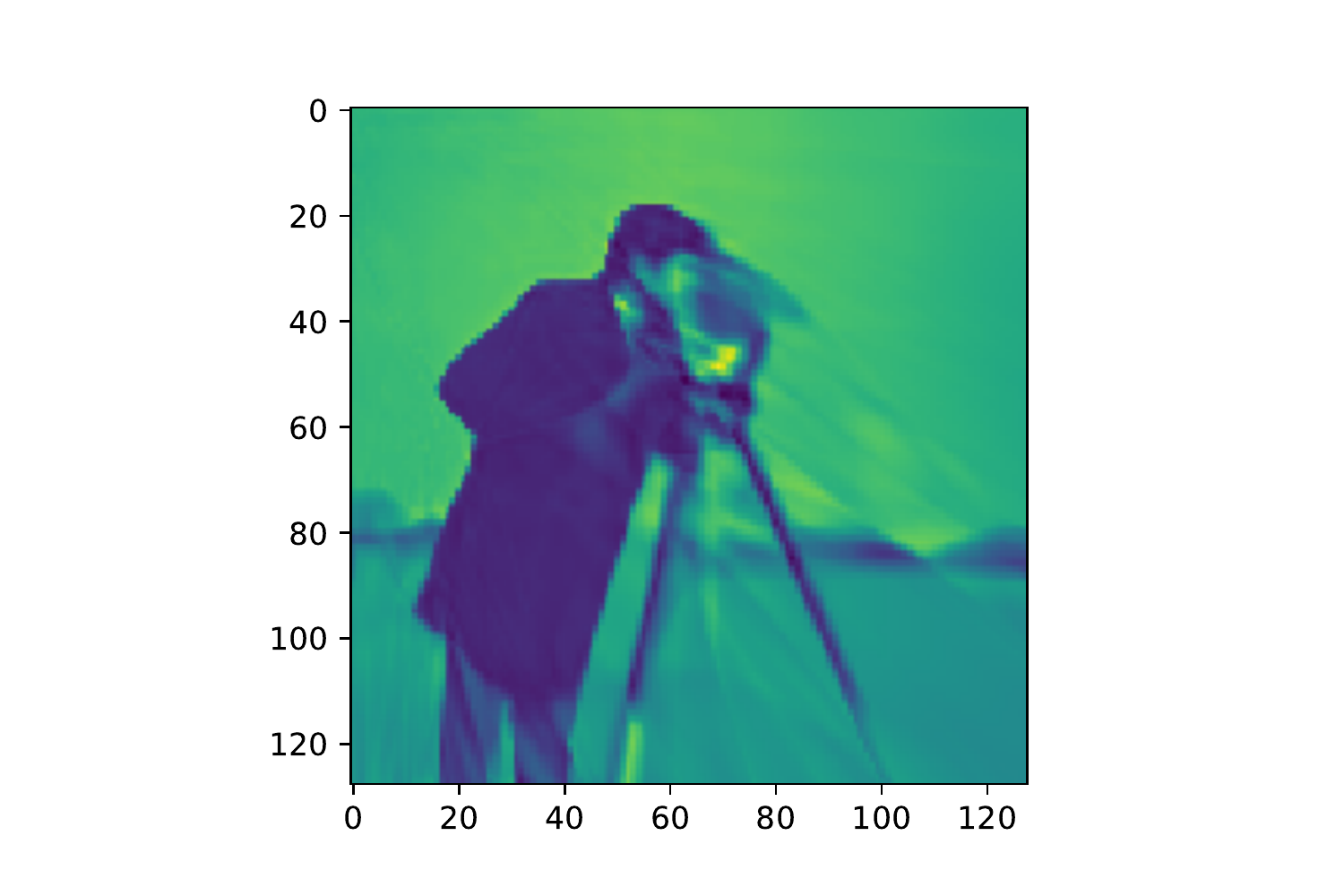}\hspace{-55pt}
        \includegraphics[scale=0.38]{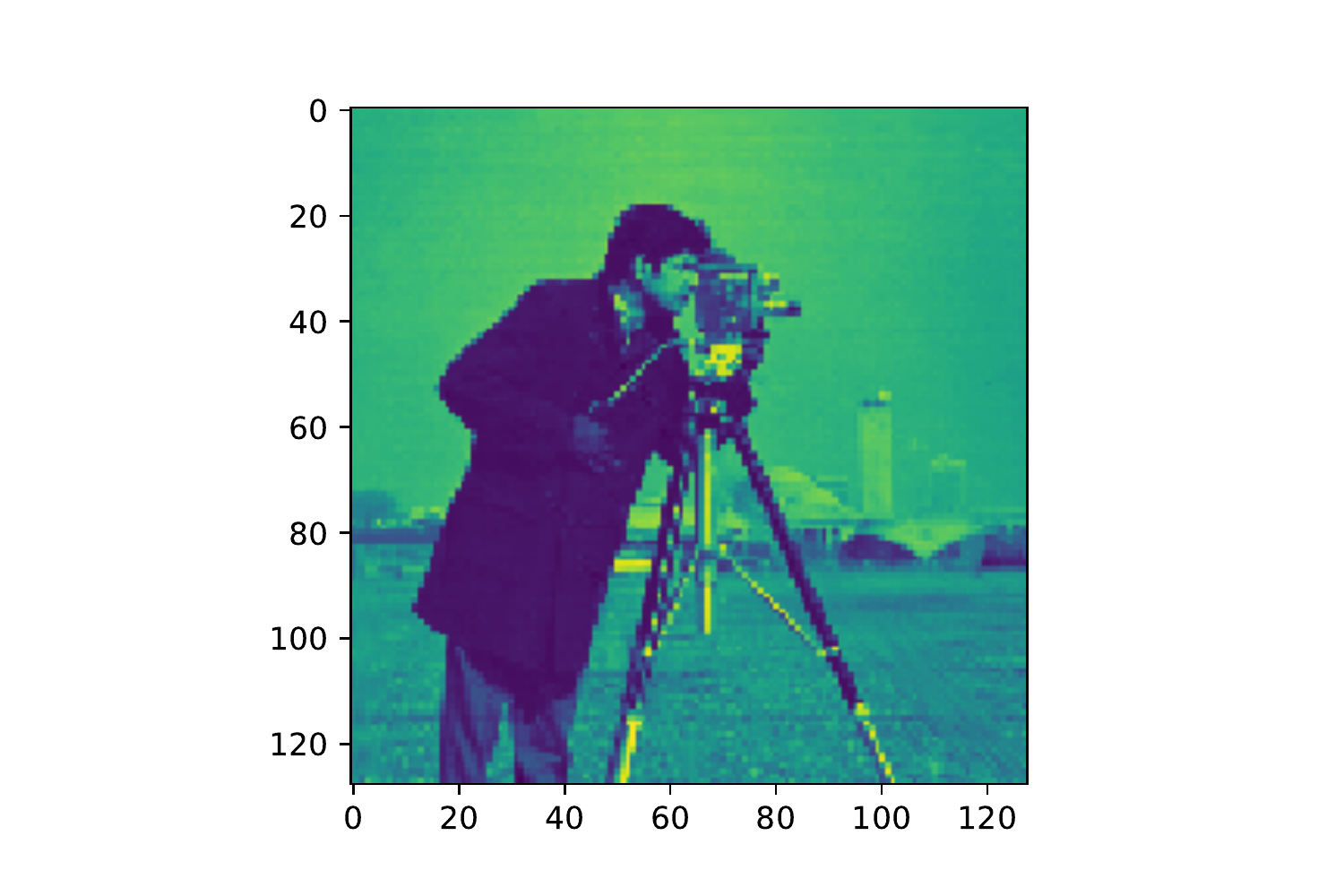}
        \caption{Left: Original Image, Middle: Image fit using deep ReLU network, Right: Image fit using deep Hat network.}
        \label{picture_fit}
	\end{figure*}
	
%
%
%
\end{experiment}

\subsection{Experiments with Real Data}
Next, we test the spectral bias of neural networks on real data, specifically on the MNIST dataset. Since the data lives in a very high dimension relative to the number of samples, we argue that Fourier modes are not a suitable notion of frequency. 
To get around this issue, we consider the eigenfunctions of a Gaussian RBF kernel \cite{braun2006model}. We largely follow the experimental setup presented in \cite{rahaman2019spectral}, with the notable difference that we compare neural networks with both a ReLU and Hat activation function \eqref{hat-definition}.

Specifically, we choose two digits $a$ and $b$ and consider fitting the classification function
\begin{equation}\label{target-function-mnist}
 u(x) = \begin{cases}
         0 & x~\text{is the digit $a$}\\
         1 & x~\text{is the digit $b$}
        \end{cases}
\end{equation}
via least squares on $2000$ training images from MNIST. Following \cite{rahaman2019spectral}, we add a moderate amount of high-frequency noise and plot the spectrum of the target function and the training iterates in Figure \ref{MNIST-experiment-figure}. We see that even with this generalized notion of frequency, the network with a Hat activation function is able to fit the higher frequencies much faster than a ReLU network.
\begin{figure}[!htbp]
		\centering
		\includegraphics[scale=0.33]{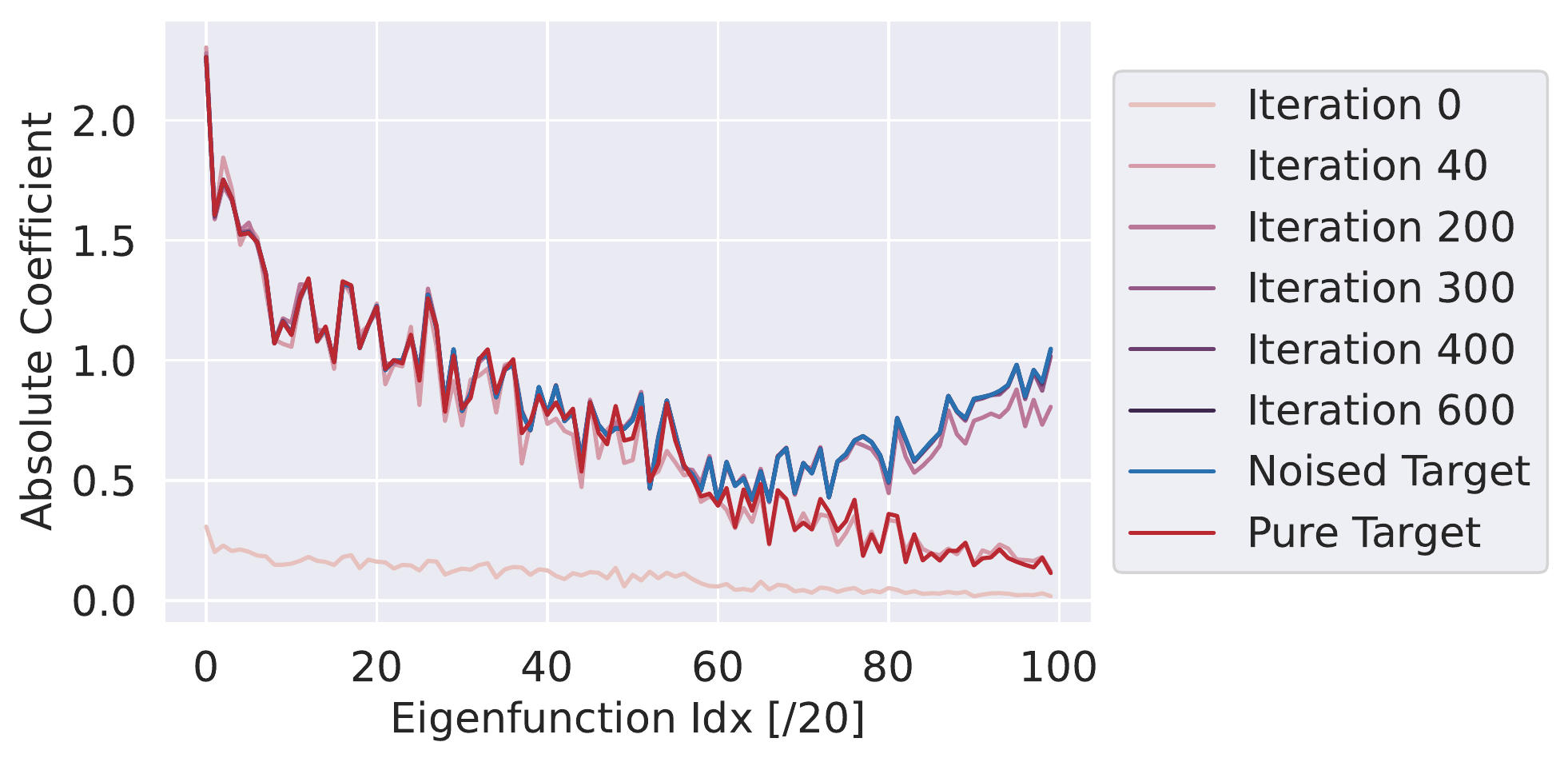}\hspace{10pt}
		\includegraphics[scale=0.33]{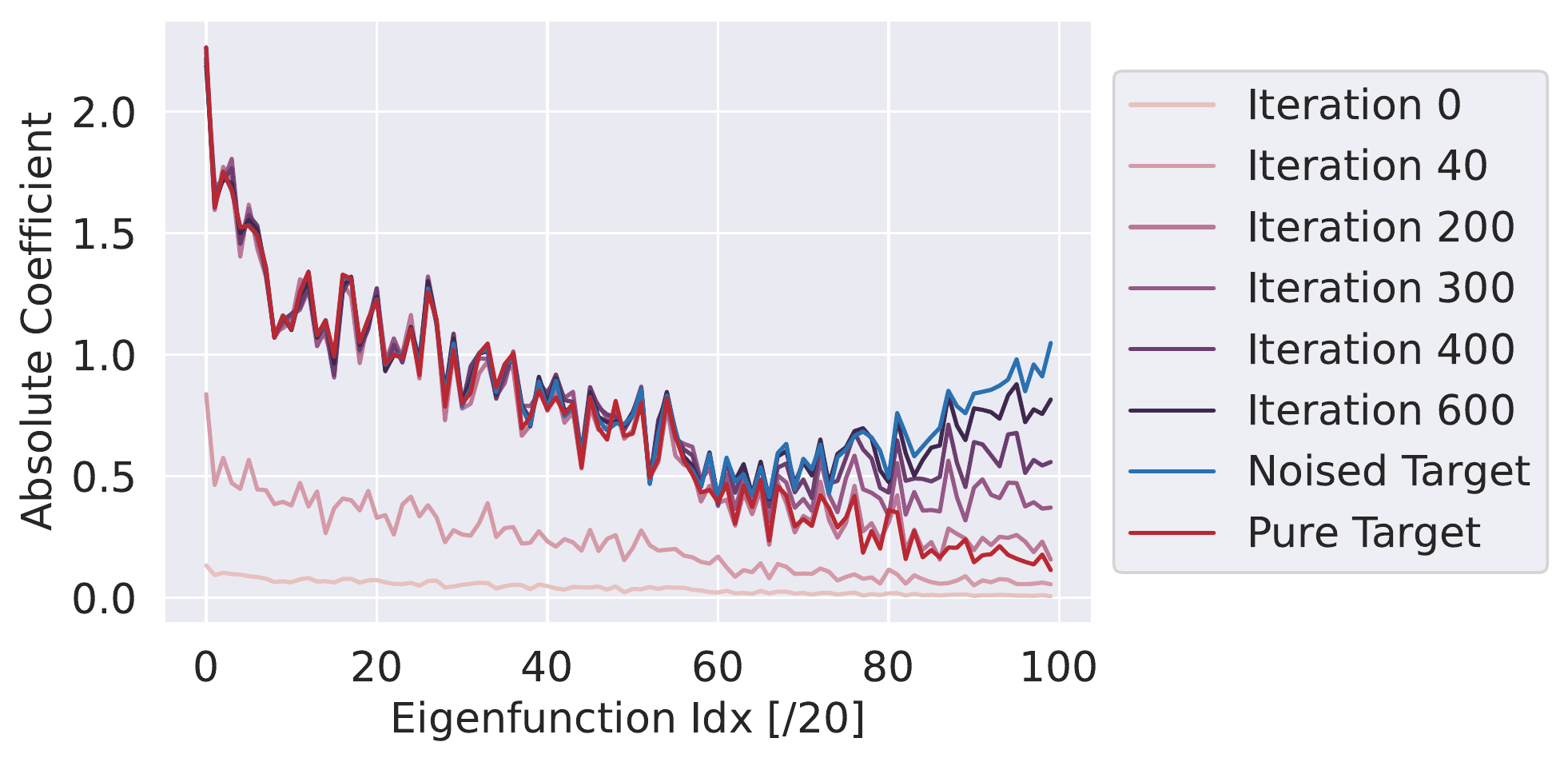}
		\caption{Real Data Experiments on MNIST using Gaussian RBF kernel eigenfunctions. We can clearly see that the error components in each eigenfunction have all converged by iteration 200 when using the Hat activation function (left), while the higher frequency components converge much more slowly when using the ReLU activation function (right).}
		\label{MNIST-experiment-figure}
	\end{figure}
	
\vspace{-10pt}

\section{Conclusion}\label{conclusion}
We have provided a theoretical explanation for the spectral bias of ReLU networks and shown that using the Hat activation function will remove this spectral bias, which we also confirmed empirically. Further research directions include studying the spectral bias for a wider variety of activation functions. In addition, we view deepening the connections between finite element methods and neural networks as a promising research focus.



\bibliographystyle{plain}

\appendix

\section{Appendix}

\subsection{Additional experiments}\label{Experimentsadditional}
In this subsection, we first report several additional numerical experiments using shallow neural networks to fit target functions in $\mathbb R^d (d=2,3)$ showing that the spectral bias holds for ReLU shallow neural networks, but it does not hold for Hat neural networks. In addition, we report a comparison numerical experiment between the hat neural network and
the neural network with $\sin$ activation function and a comparison numerical experiment between hat neural network and ReLU neural network with accurate approximation to the loss function.  Also we add an experiment with fewer neurons by rerunning the Experiments 1 and 2. Then we add an experiment using SGD method as optimizer.  Finally we add an  experiment showing that when the shallow ReLU network gets wider, the spectral bias gets stronger. 
\begin{experiment}\label{2Dinde-ex}
\normalfont

In this experiment, we investigate how the validation performance depends on the frequency of noise added to the training target in two dimension case. 
We consider the ground target function on $[0,1]^2$
\begin{equation}
u_0(\boldsymbol x)=\sin(2\pi x_1) \sin(2\pi x_2).
\end{equation}
Let $\psi_k(\boldsymbol x)$ be the noise function
\begin{equation*}
\psi_k(\boldsymbol x)=0.2 \sin(2k\pi x_1)  \sin(2k\pi x_2),
\end{equation*}
where $k$ is the frequency of the noise. The final target function $u(\boldsymbol x)$ is then given by $u(\boldsymbol x)=u_0(\boldsymbol x)+\psi_k(\boldsymbol x)$. 
We use two different shallow neural network models the same as the two used in Experiment \ref{2Ddirect-ex}.
The MSE loss function is computed by 4000 sampling points from the uniform distribution $\mathcal U([0,1]^2)$.   

The validation error is computed by
\begin{equation*}
L_2(f,u_0) = \frac1m \left(\sum_{i,j=1}^{m}\big(f(x_{1,i},x_{2,j})-u_0(x_{1,i},x_{2,j})\big)^2\right)^{\frac12},
\end{equation*}
on a uniform grid points of $[0,1]^2$ with $m=2^7$.  Both models
are trained with a learning rate of 0.001 and decreasing to its $\frac34$ for each 250 epochs. All parameters are 
initialized following a uniform distribution $\mathcal U(-0.3,0.3)$. 
\begin{figure}[!htbp]
		\centering		
		\includegraphics[scale=0.24]{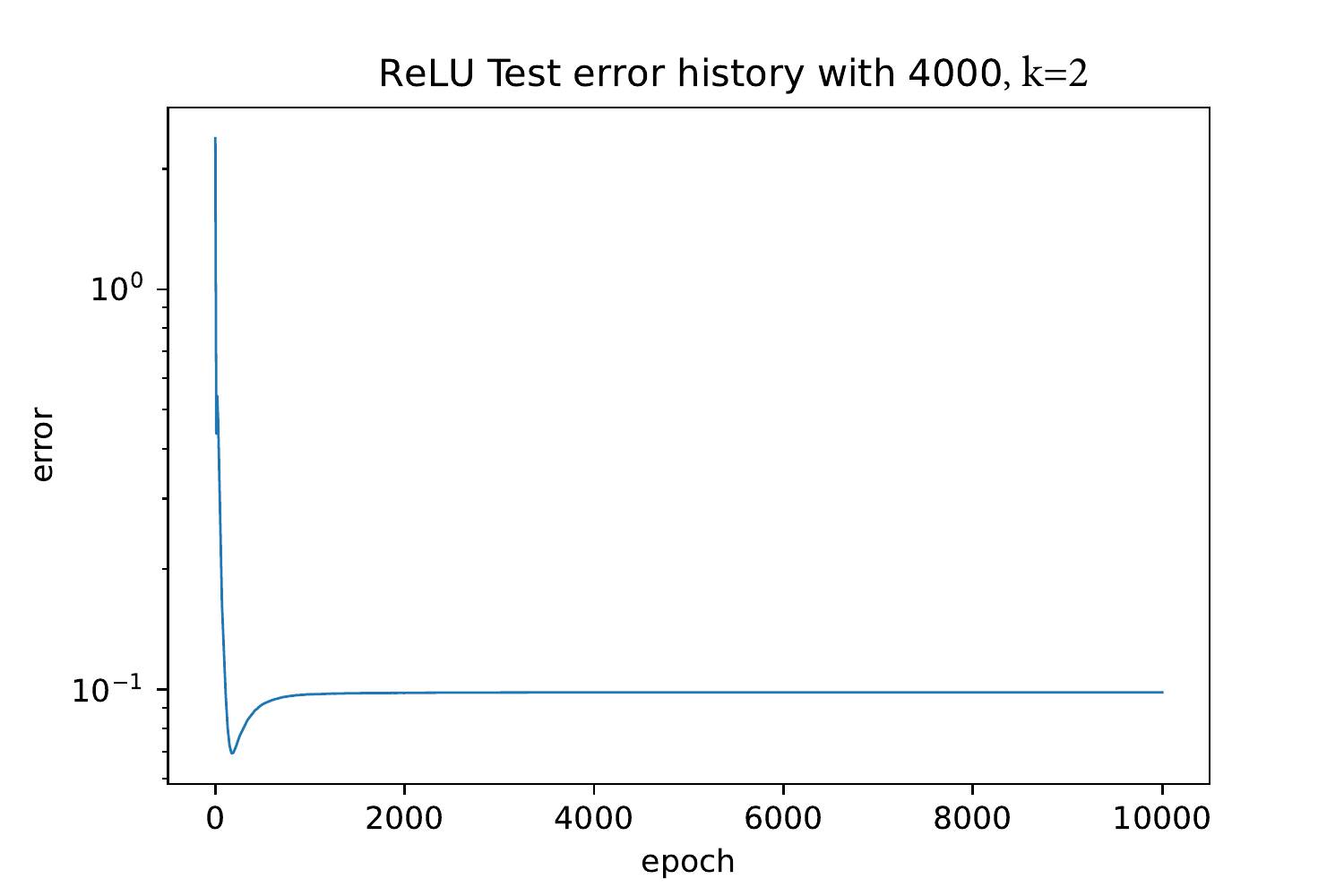}\hspace{-11pt}
		\includegraphics[scale=0.24]{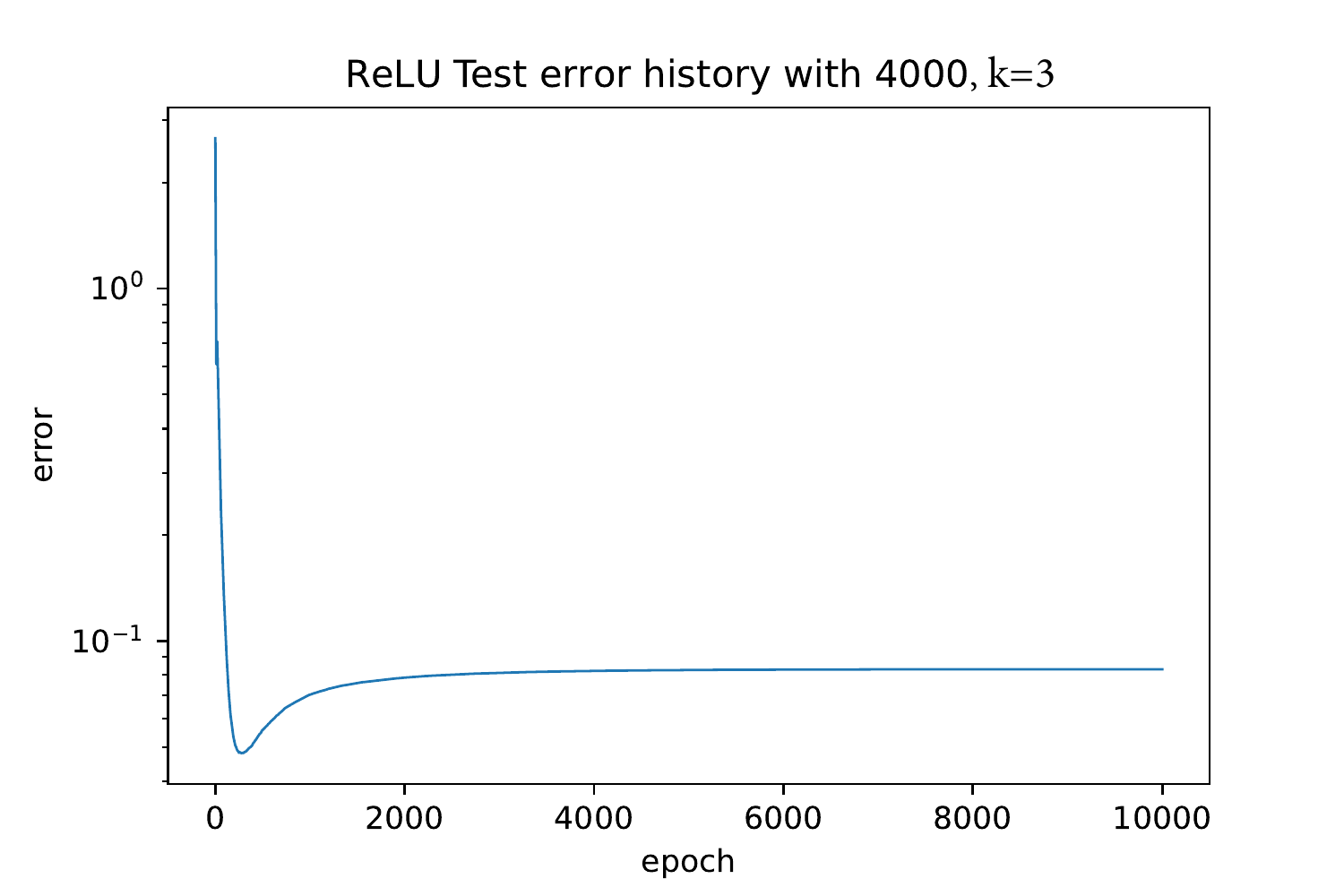}\hspace{-11pt}
		\includegraphics[scale=0.24]{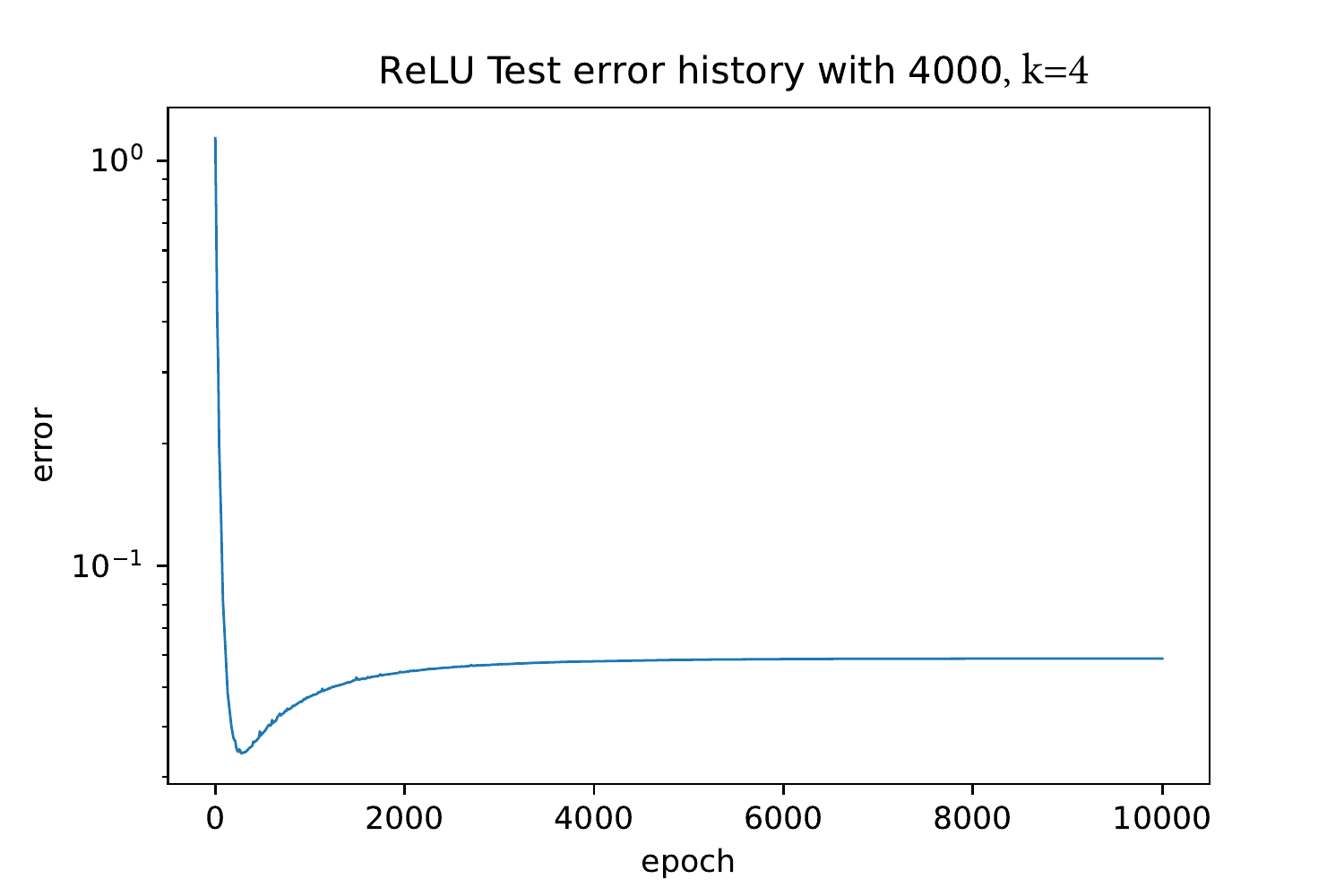}\hspace{-11pt}
		\includegraphics[scale=0.24]{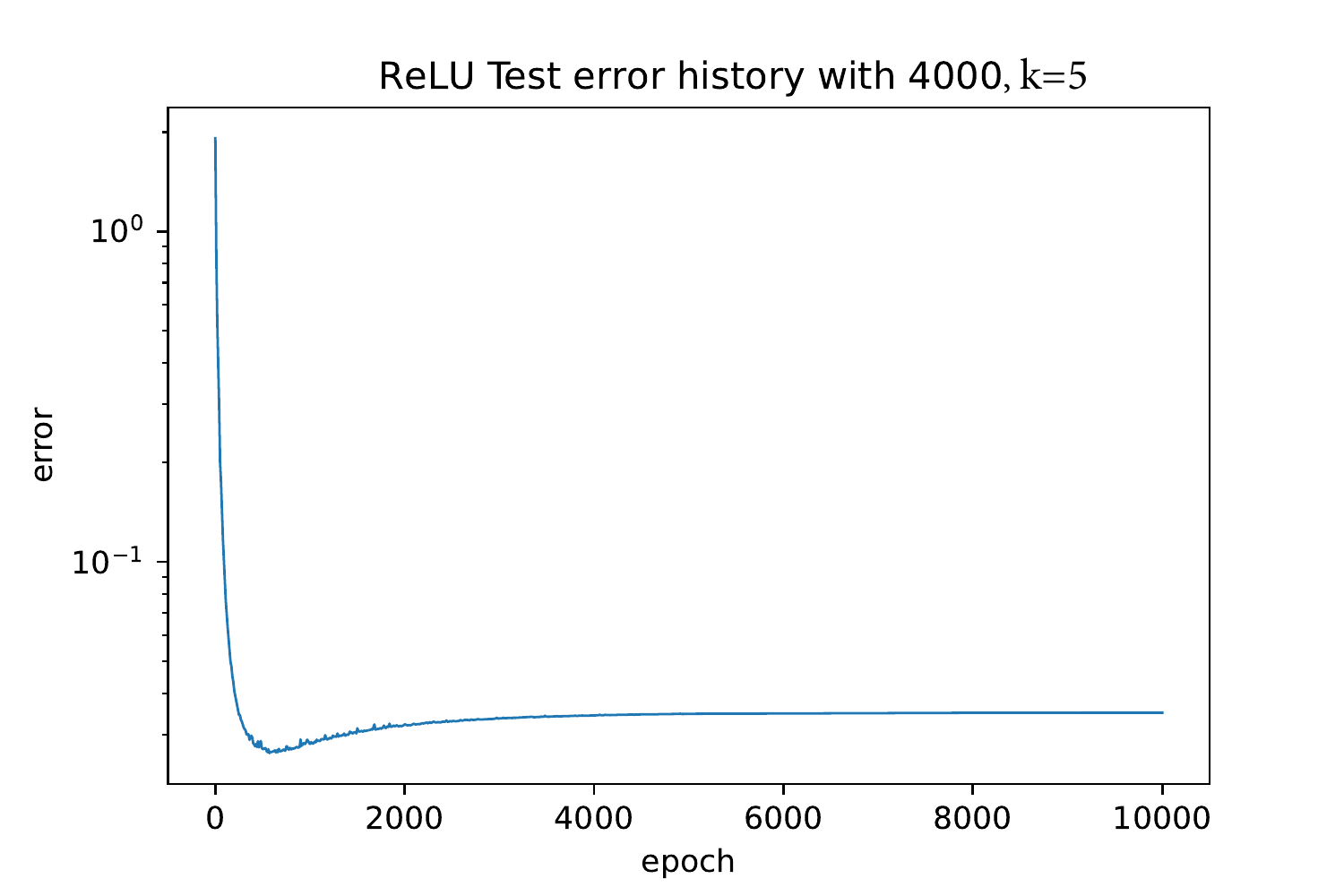}
		\caption{Validation error history for $\sigma(p)=\text{ReLU}(p)$ with $k=2,3,4,5$.}
		\label{2DReLU}
\end{figure}
\vspace{-18pt}

\begin{figure}[!htbp]
		\centering		
		\includegraphics[scale=0.24]{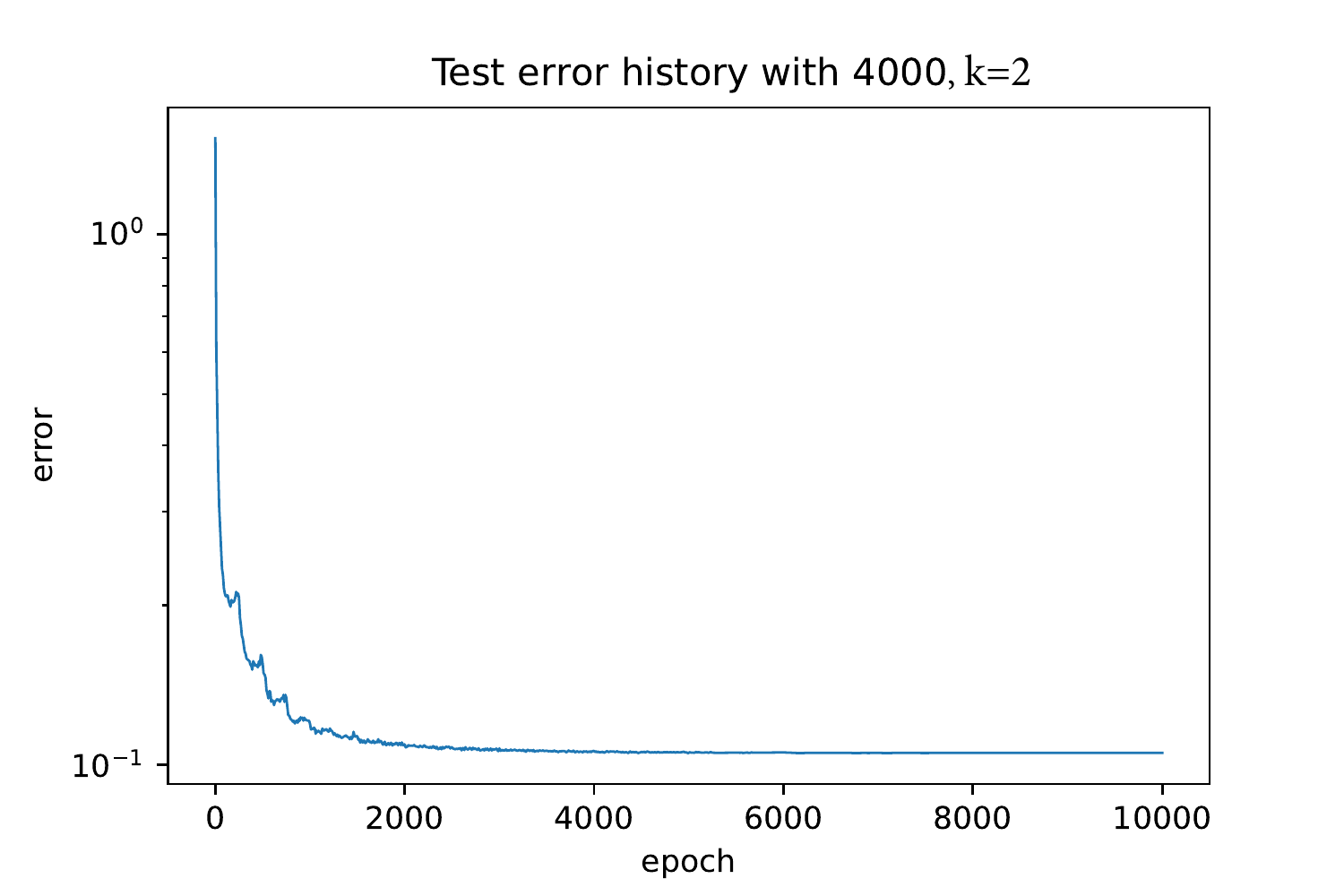}\hspace{-11pt}
		\includegraphics[scale=0.24]{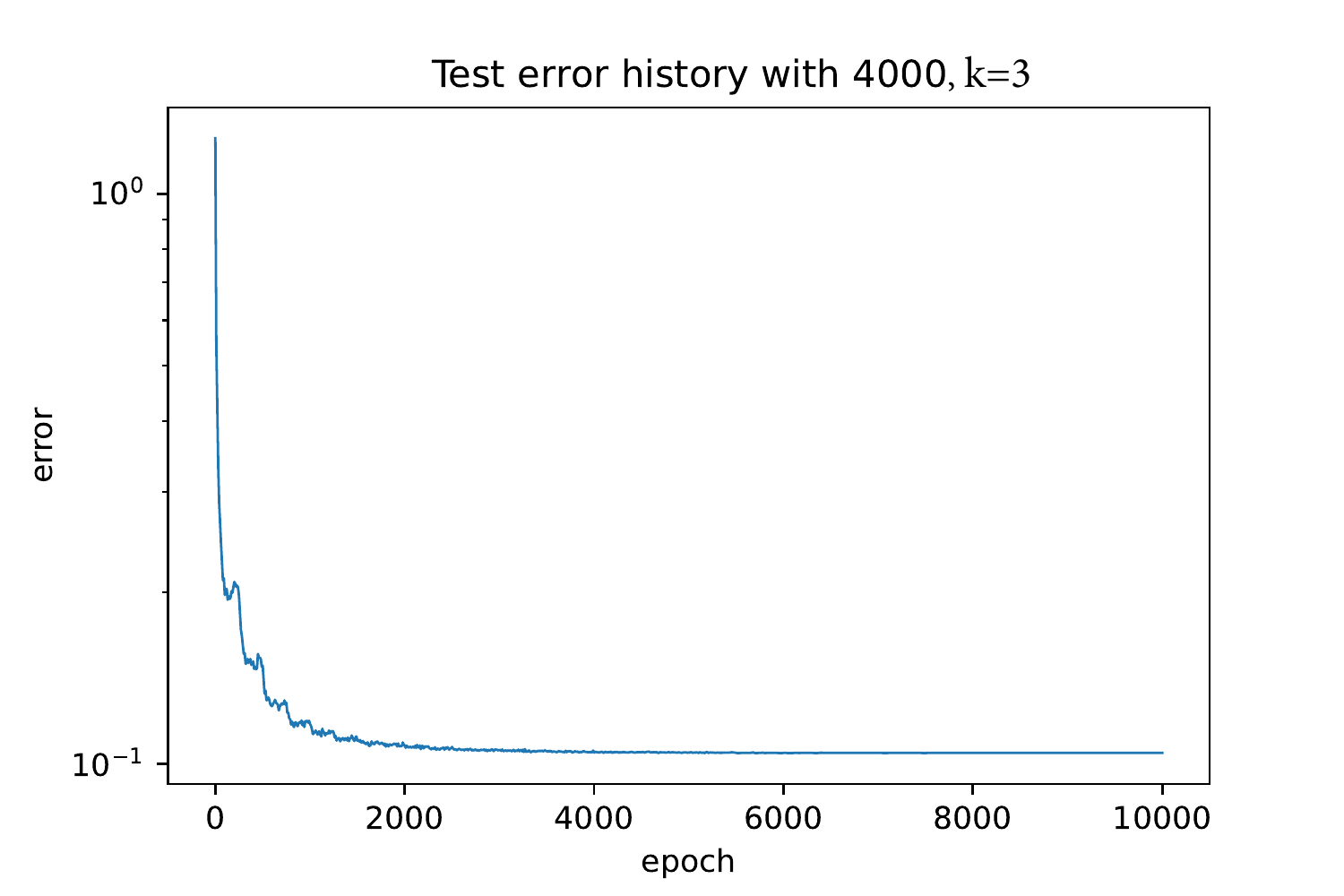}\hspace{-11pt}
		\includegraphics[scale=0.24]{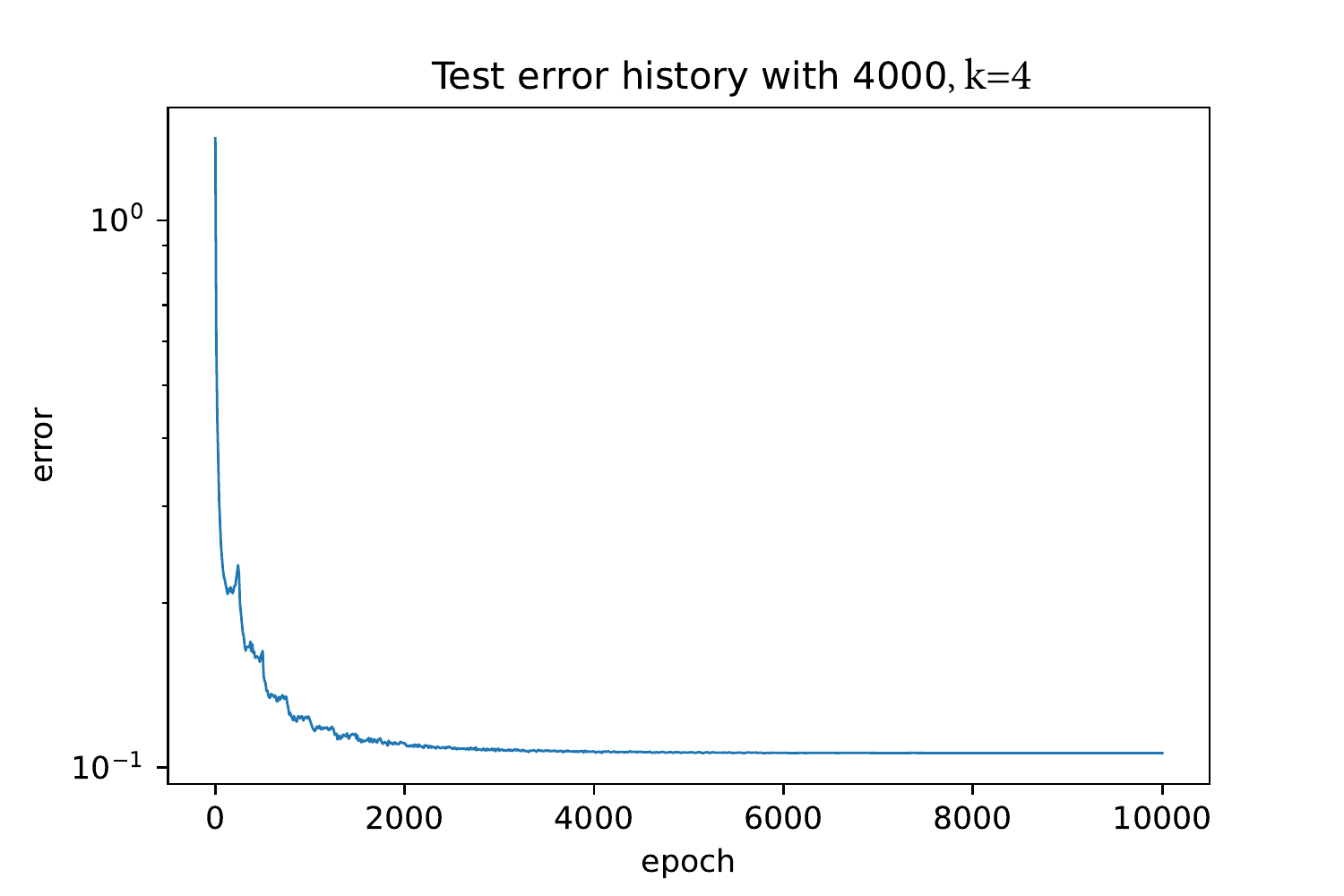}\hspace{-11pt}
		\includegraphics[scale=0.24]{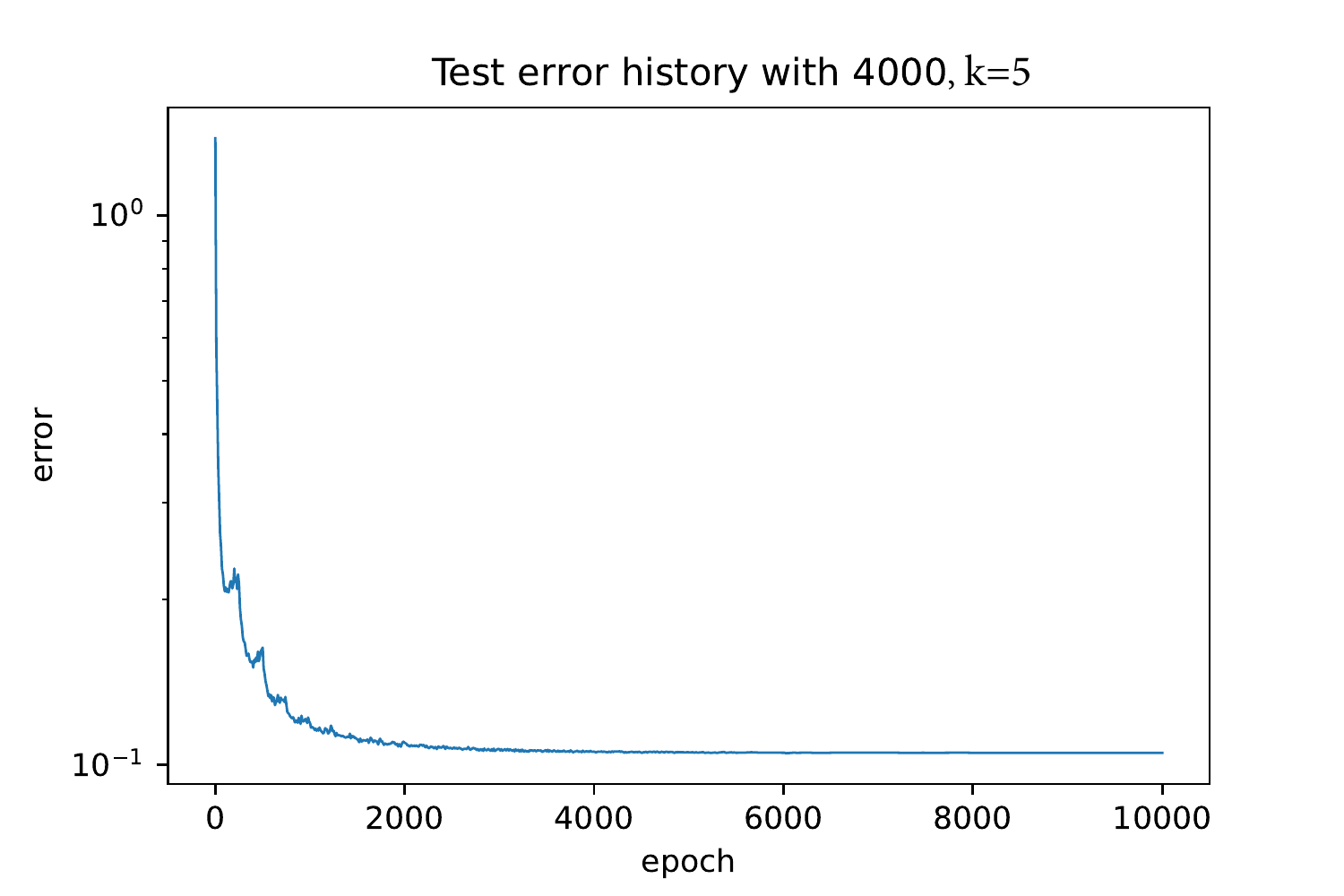}
		\caption{Validation error history for $\sigma(p)=\text{ Hat}(100p)$ with $k=2,3,4,5$.}
		\label{2DHat}
\end{figure}

From the Figure \ref{2DReLU}, we can see that the profile of the loss curves varies significantly with the frequency of noise added to the target when ReLU neural network is used. This is explained by the fact that the ReLU neural network readily fits the noise signal if it is low frequency, whereas the higher frequency noise is only fit later in the training. In the latter case, the dip in validation score early in the training is when the network has learned the low frequency true target function $u_0(\boldsymbol x)$; the remainder of the training is spent learning the higher-frequencies in the training target $u(\boldsymbol x)$. When the frequency is higher, the ReLU neural network fits the target function slower indicated by the loss curves.  From the Figure \ref{2DHat}, we can see that the profile of the loss curves are very much the same with respect to the frequency of noise added to the target when Hat neural network is used. This is explained by the fact that Hat neural network does not have frequency bias. 
\end{experiment}

\begin{experiment}\label{3Dcase1}
\normalfont
In this experiment, we investigate how the validation performance depends on the frequency of noise added to the training target in three dimension case. 
We consider the target function 
\begin{equation}
u_0(\boldsymbol x)=\sin(2\pi x_1) \sin(2\pi x_2) \sin(2\pi x_3),
\end{equation}
where $\boldsymbol x=(x_1,x_2,x_3)\in [0,1]^3$. 
Let $\psi_k(\boldsymbol x)$ be the noise function
\begin{equation*}
\psi_k(\boldsymbol x)=0.5 \sin(2k\pi x_1)  \sin(2k\pi x_2) \sin(2k\pi x_3),
\end{equation*}
where $k$ is the frequency of the noise. The final target function $u$ is then given by $u(\boldsymbol x)=u_0(\boldsymbol x)+\psi_k(\boldsymbol x)$. 

We use shallow neural network models
with two different activation functions to fit $u(\boldsymbol x)$. 
One is $\sigma(p)=\text{ReLU}(p)$, and the other one is the scaled Hat function $\sigma(p)={\rm  Hat}(100p)$.
Both two models have only one hidden layer with size 3-30000-1.

The training MSE loss function is computed by 100000 sampling points from the uniform distribution $\mathcal U([0,1]^2)$.   

The validation error is computed by
\begin{equation}\label{test3dloss}
L_2(f,u) = \left(\frac1m\sum_{i=1}^{m} (f(\boldsymbol x_i)-u_0(\boldsymbol x_i))^2\right)^{\frac12},
\end{equation}
where $\{\boldsymbol x_i\}^{m}_{i=1}, m=100000$ are sampling points from the uniform distribution $\mathcal U([0,1]^3)$. 
The ReLU neural network are trained by Adam optimizer with learning rate of 0.001 and decreasing to its $0.85$ for each 300 epochs. 
The Hat neural network are trained with learning rate of 0.001 and decreasing to its $0.75$ for each 250 epochs. All parameters are 
initialized following a uniform distribution $\mathcal U(-0.3,0.3)$. 

\begin{figure*}[h]
		\centering		
		\includegraphics[scale=0.24]{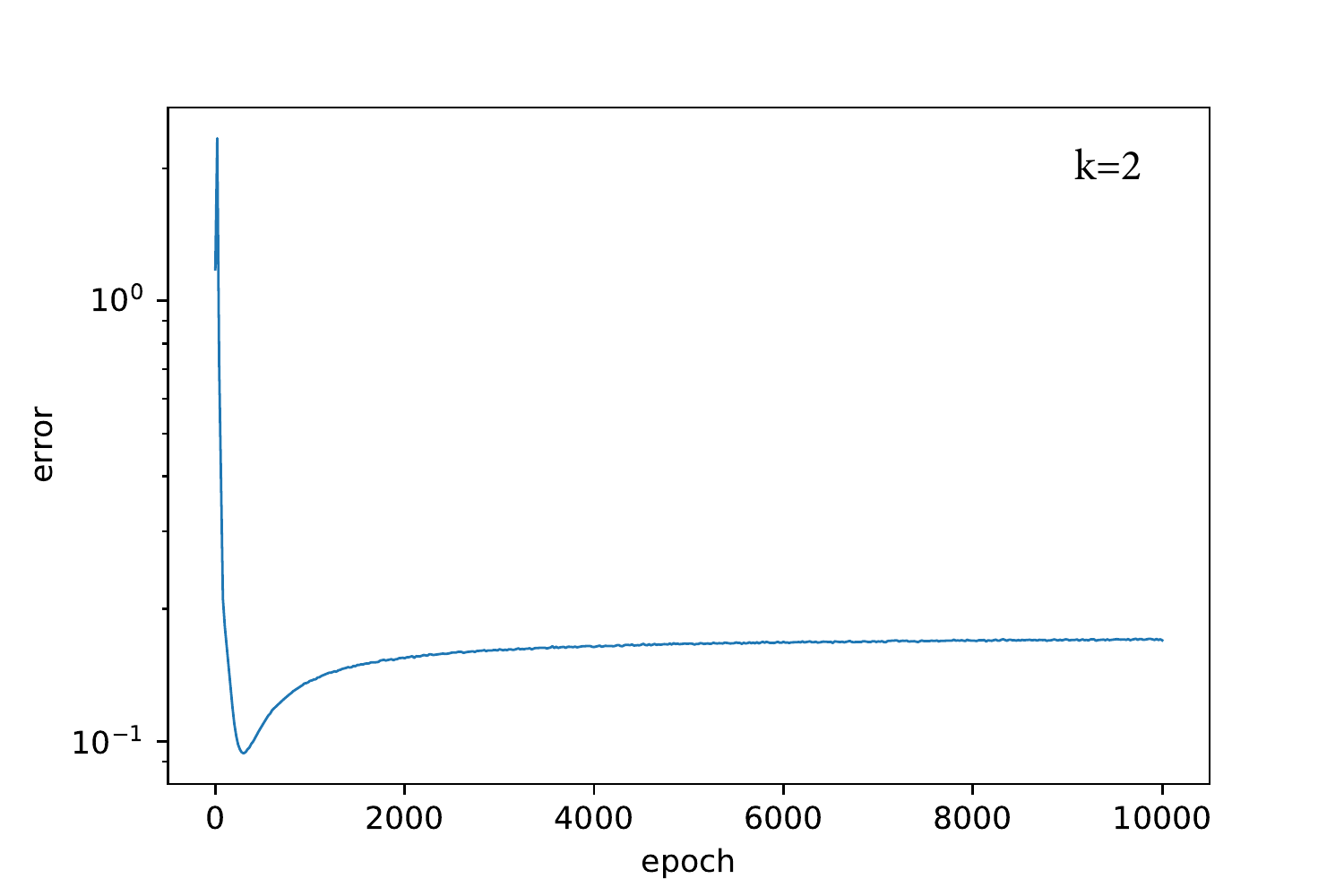}\hspace{-11pt}
		\includegraphics[scale=0.24]{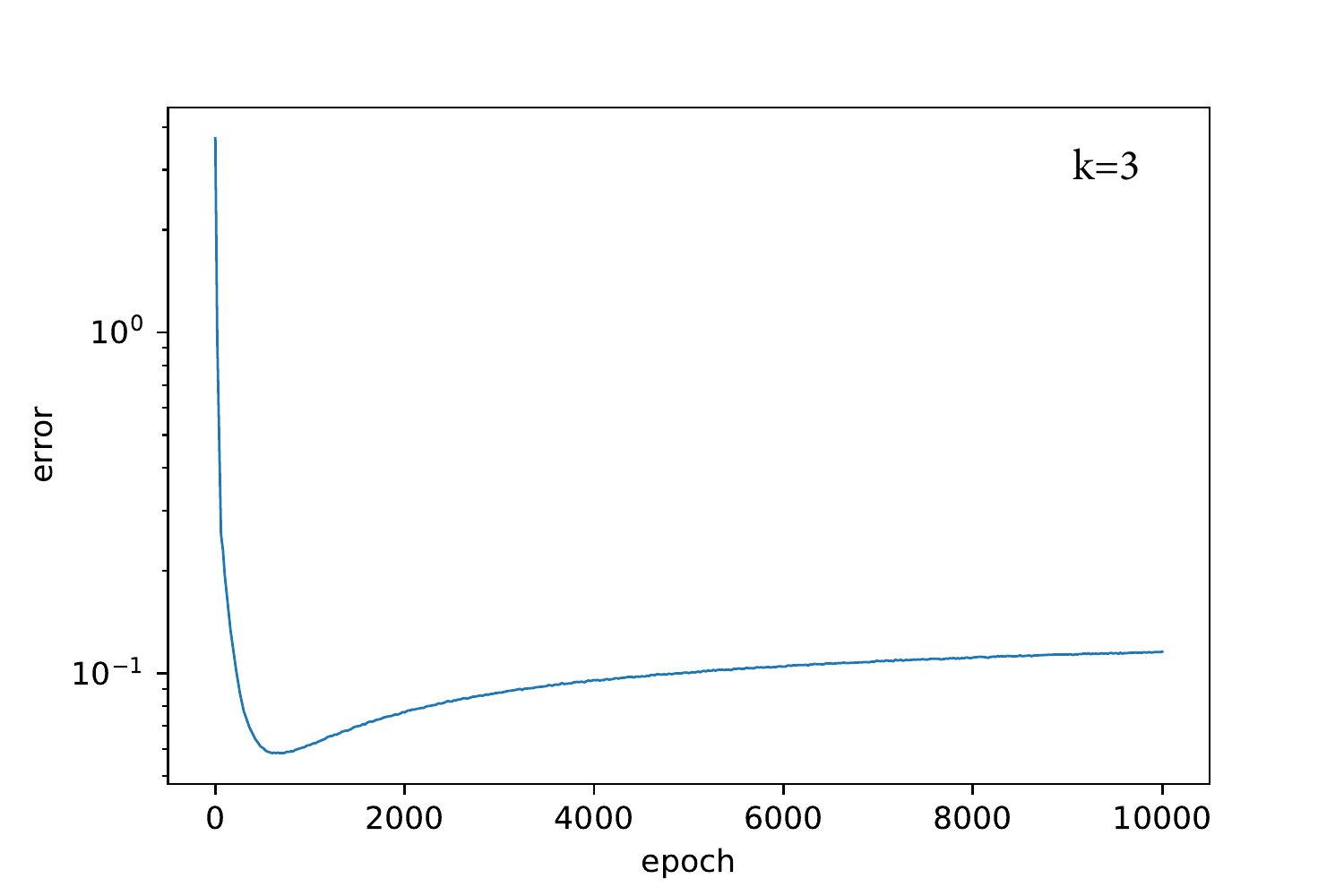}\hspace{-11pt}
		\includegraphics[scale=0.24]{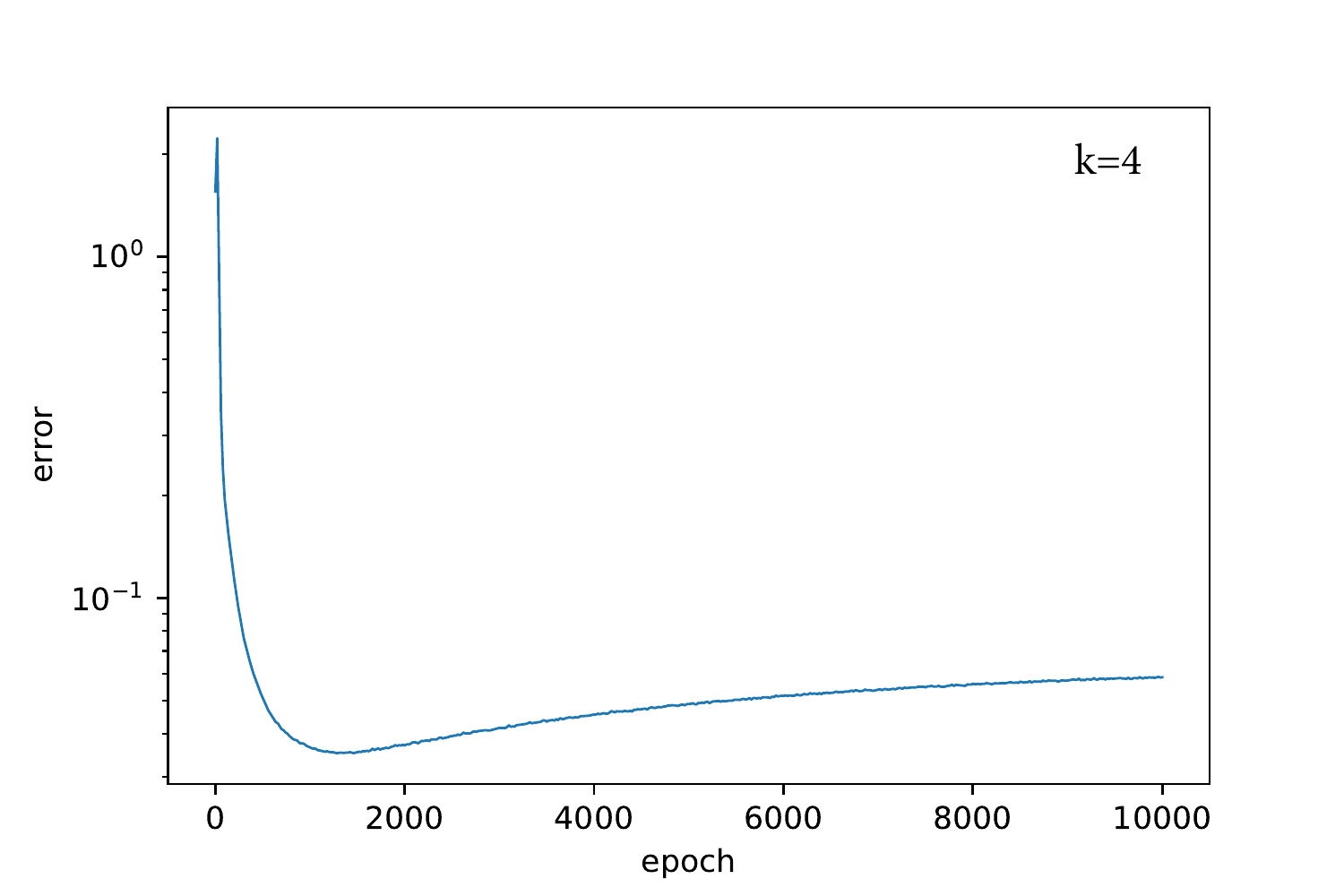}\hspace{-11pt}
		\includegraphics[scale=0.24]{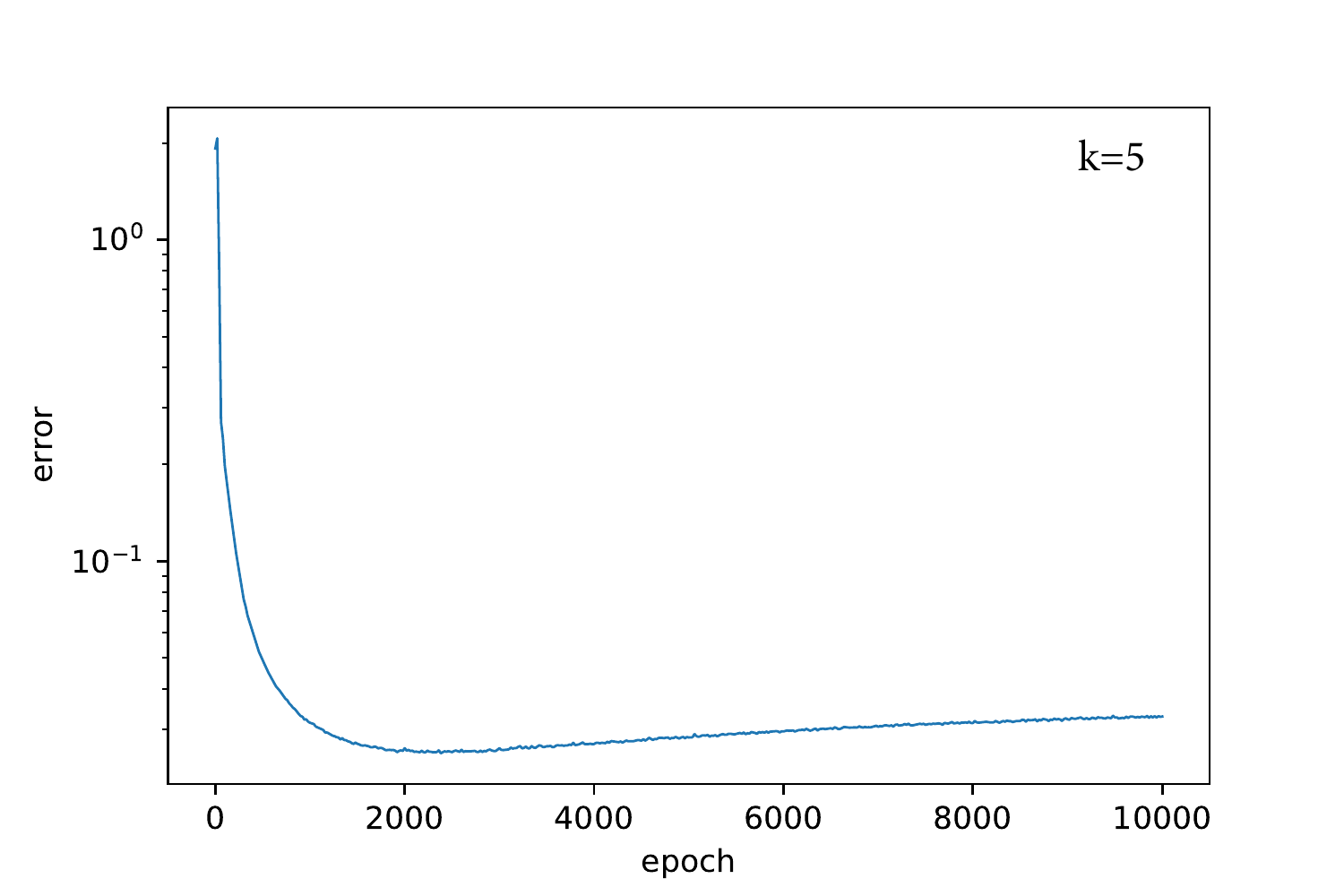}
		\caption{Validation error history for $\sigma(p)=\text{ReLU}(p)$ with $k=2,3,4,5$.}
		\label{3DReLU}
\end{figure*}
\vspace{-18pt}
\begin{figure*}[!htbp]
		\centering		
		\includegraphics[scale=0.24]{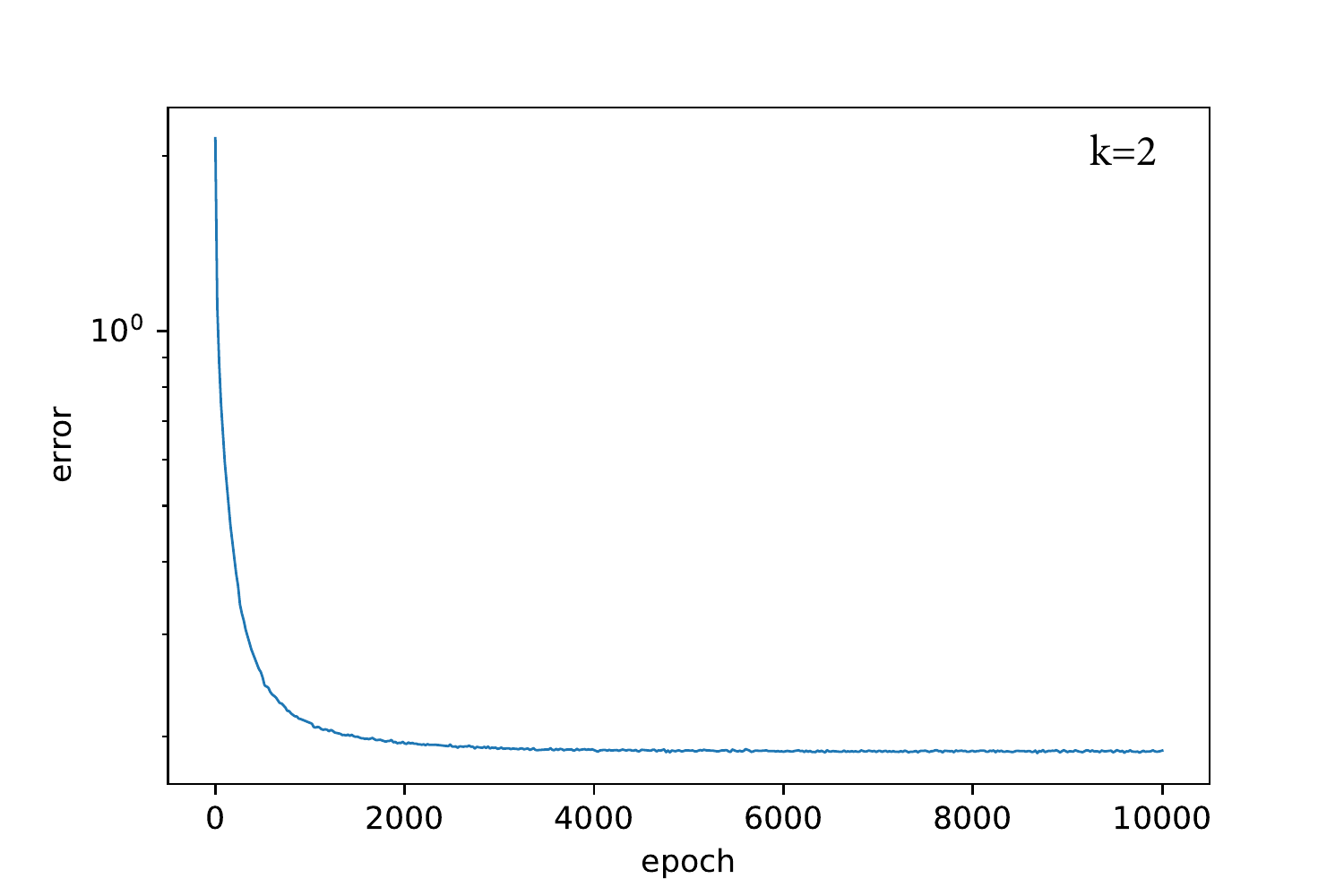}\hspace{-11pt}
		\includegraphics[scale=0.24]{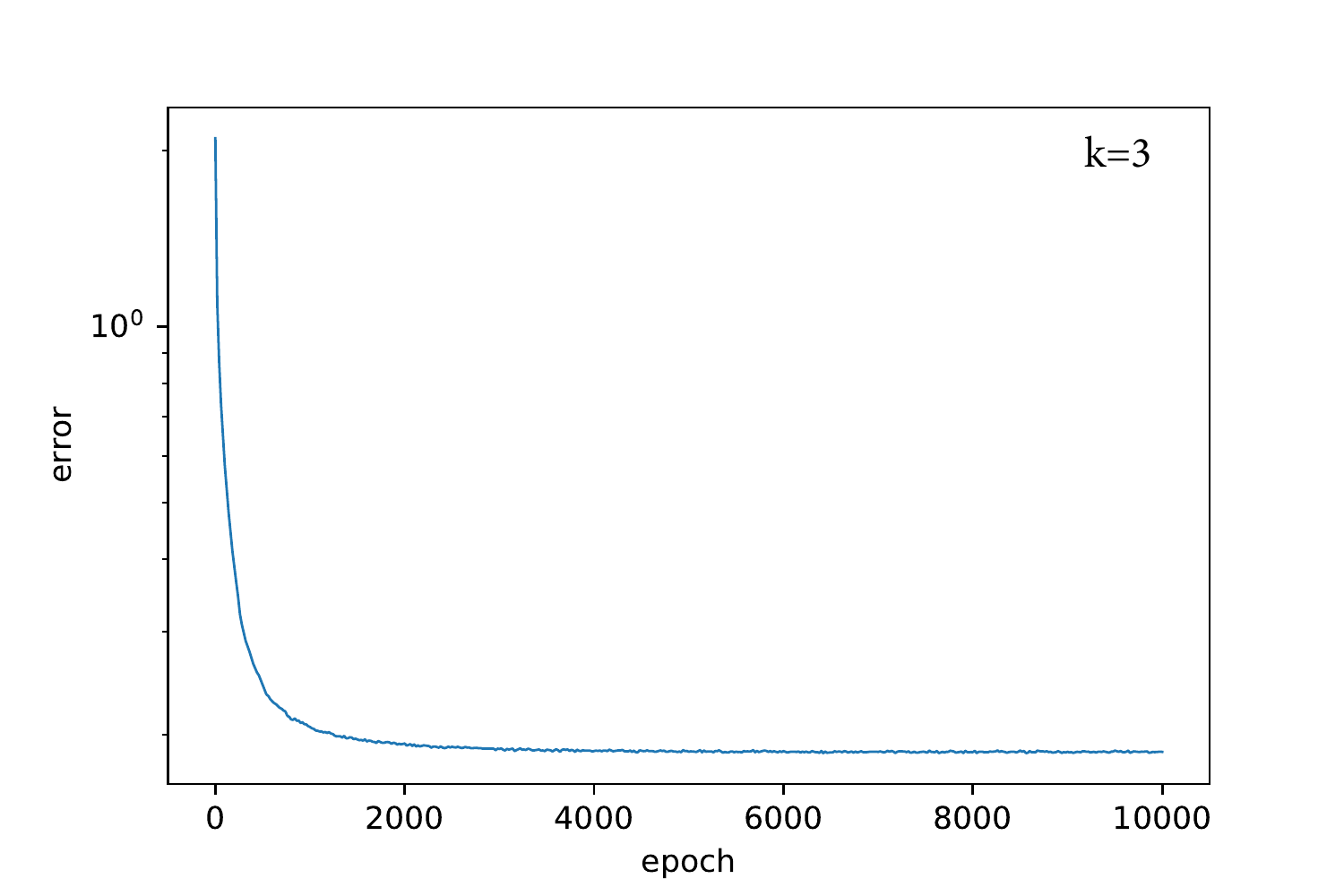}\hspace{-11pt}
		\includegraphics[scale=0.24]{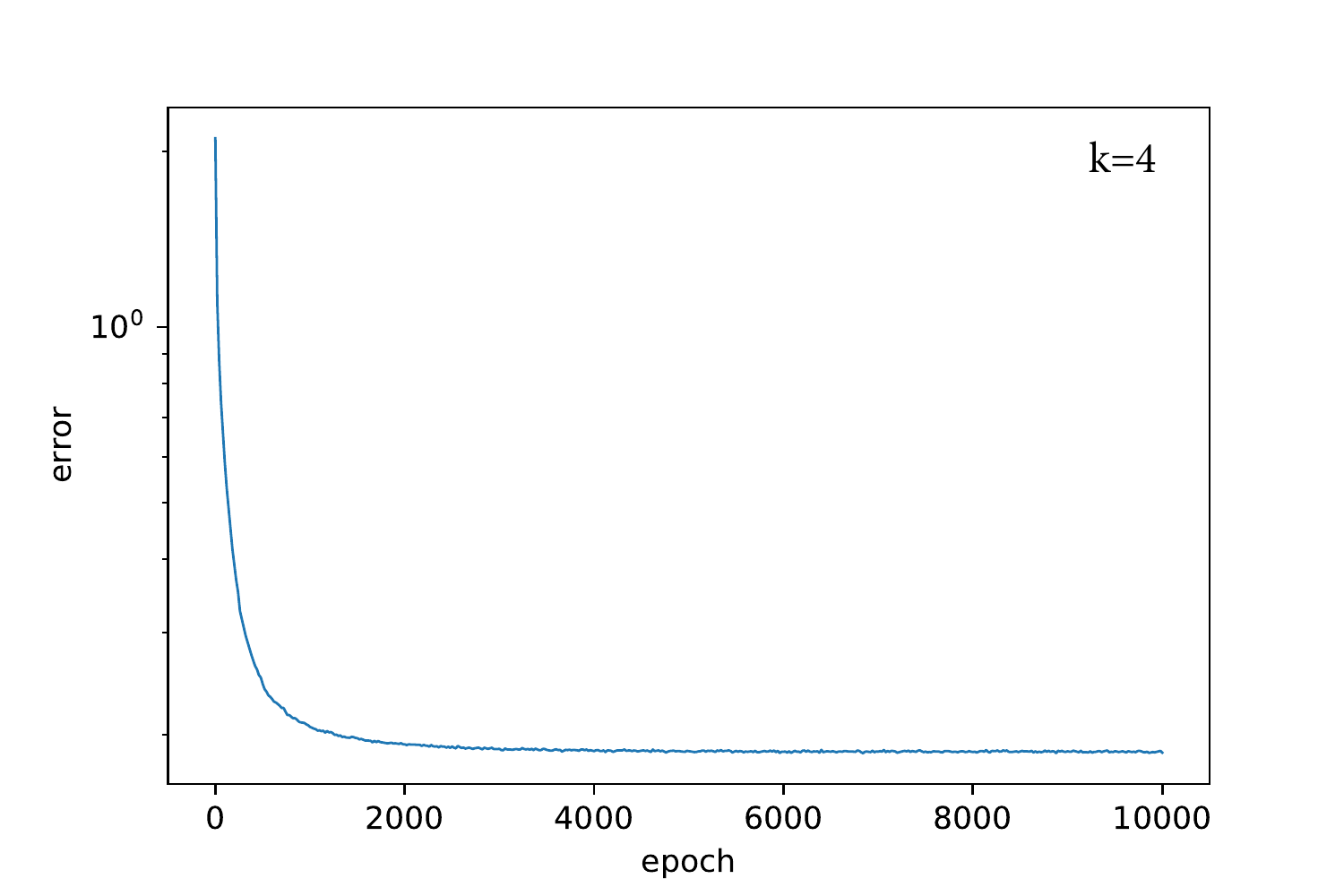}\hspace{-11pt}
		\includegraphics[scale=0.24]{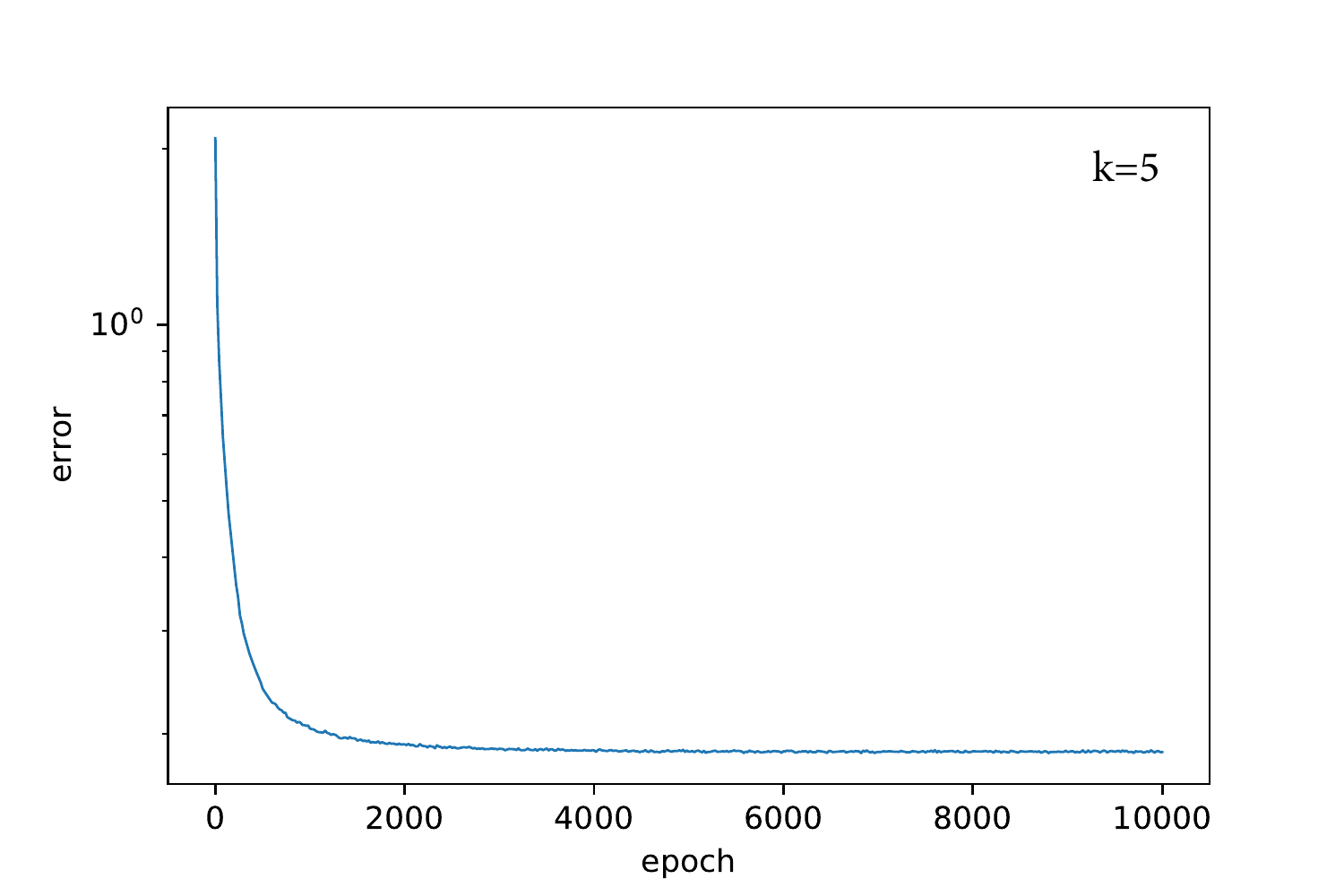}
		\caption{Validation error history for $\sigma(p)=\text{ Hat}(100p)$ with $k=2,3,4,5$.}
		\label{3DHat}
\end{figure*}
From the Figure \ref{3DReLU} and Figure \ref{3DHat}, we can see that the numerical results are similar to the two dimension case as shown in Experiment \ref{2Dinde-ex}. 
\end{experiment}

\begin{experiment}\label{3Dcase2}
\normalfont
In this experiment, we investigate how the validation performance depends on the frequency of noise added to the training target in three dimension case. 
We consider the ground target function 
\begin{equation}
u_0(\boldsymbol x)=\sin(2\pi x_1) \sin(2\pi x_2) \sin(2\pi x_3).
\end{equation}
Let $\psi_{1,k}(\boldsymbol x)$ and $\psi_{2,k}(\boldsymbol x)$ be the noise functions
\begin{equation*}
\psi_{1,k}(\boldsymbol x)=0.5  \sin(2\pi k\|\boldsymbol x\|), ~ \psi_{2,k}(\boldsymbol x)=  \frac{0.5 \sin(2\pi k\|\boldsymbol x\|)}{\|\boldsymbol x\|},
\end{equation*}
where $k$ is the frequency of the noise. The final target function $u(\boldsymbol x)$ is then given by $u(\boldsymbol x)=u_0(\boldsymbol x)+\psi_{j,k}(\boldsymbol x),~j=1,2$. 

We use models 
$$
f(\boldsymbol x) = W_2 \sigma \left(W_1 \boldsymbol x
 + \boldsymbol b_1\right)
$$ 
with two different activation functions to fit $u(\boldsymbol x)$. 
One is $\sigma(p)=\text{ReLU}(p)$, and the other one is the scaled Hat function $\sigma(p)={\rm Hat}(100p)$.
Both two models have only one hidden layer with size 3-30000-1.
The training error is computed by 
\begin{equation}
L_1(f,u) = \left(\frac1N\sum_{i=1}^{N} (f(\boldsymbol x_i)-u(\boldsymbol x_i))^2\right)^{\frac12},
\end{equation}
where $\{\boldsymbol x_i\}^{n}_{i=1},N=100000$ are sampling points from the uniform distribution $\mathcal U([0,1]^3)$.   

The validation error is computed by
\begin{equation}
L_2(f,u) = \left(\frac1m\sum_{i=1}^{m} (f(\boldsymbol x_i)-u_0(\boldsymbol x_i))^2\right)^{\frac12},
\end{equation}
where $\{\boldsymbol x_i\}^{m}_{i=1}, m=100000$ are sampling points from the uniform distribution $\mathcal U([0,1]^3)$. 
Both models are trained by Adam optimizer with a learning rate of 0.001 and decreasing to its $0.85$ for each 300 epochs. All parameters are 
initialized following a uniform distribution $\mathcal U(-0.3,0.3)$. 

\begin{figure*}[!htbp]
		\centering		
		\includegraphics[scale=0.24]{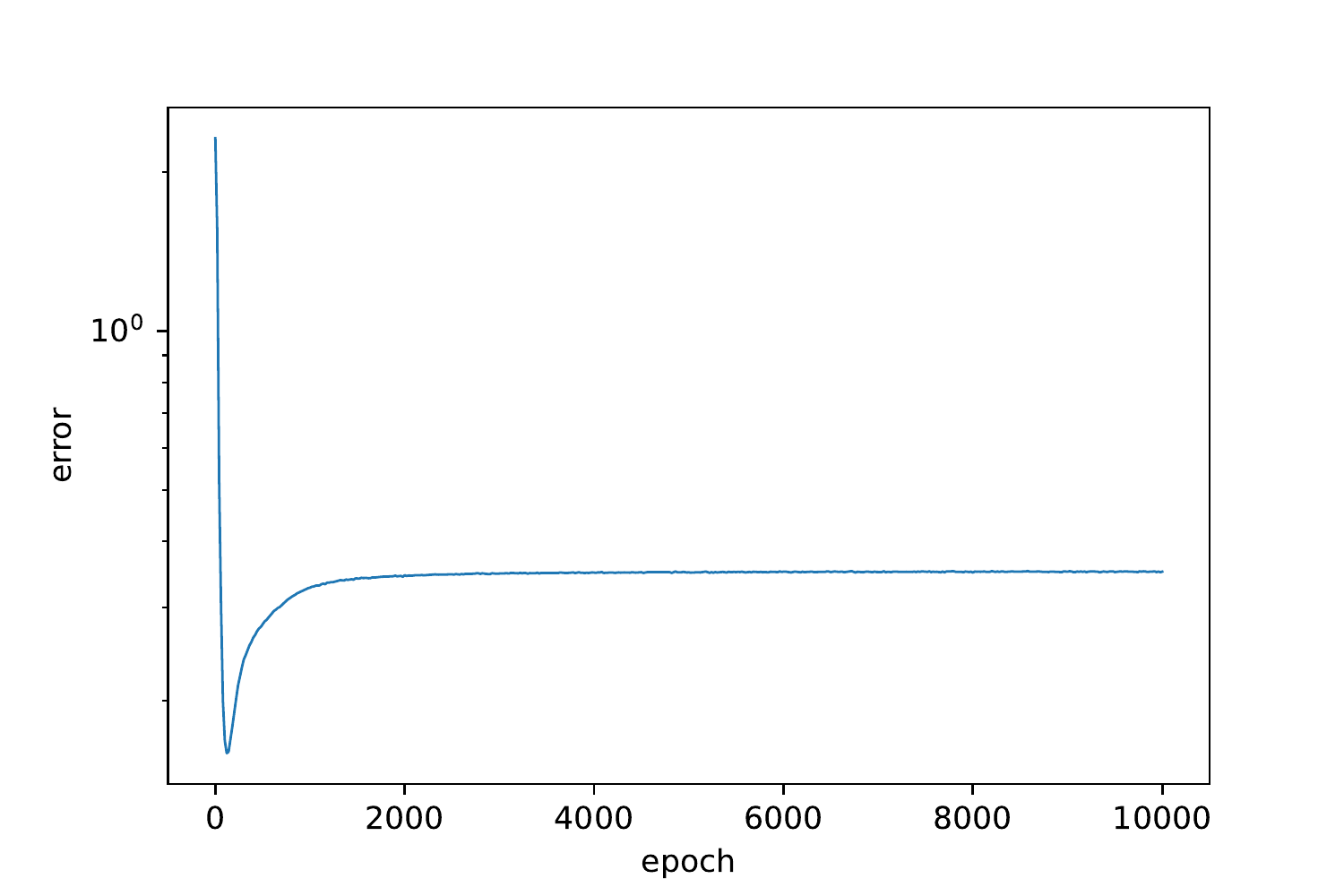}\hspace{-11pt}
		\includegraphics[scale=0.24]{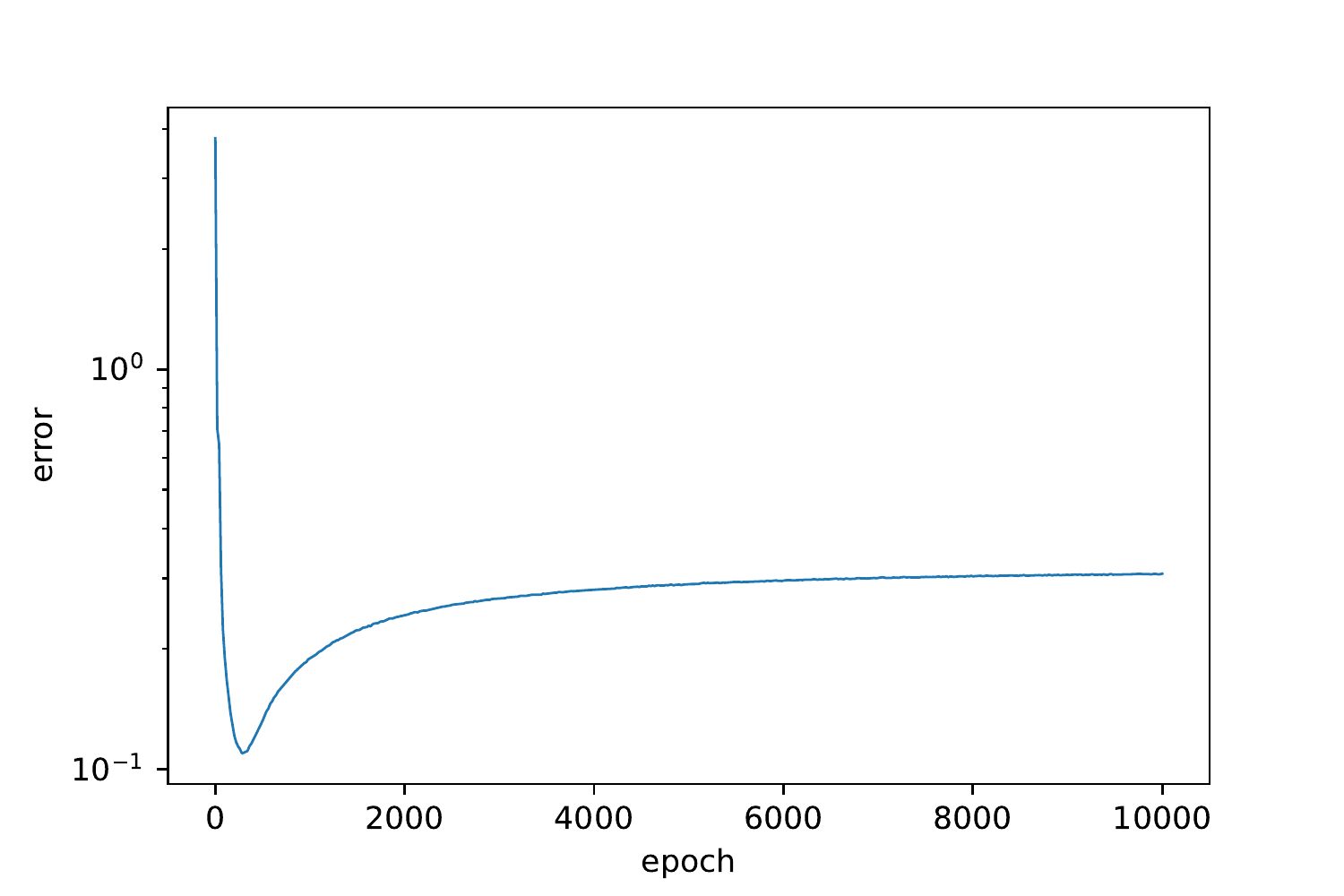} \hspace{-11pt}
		\includegraphics[scale=0.24]{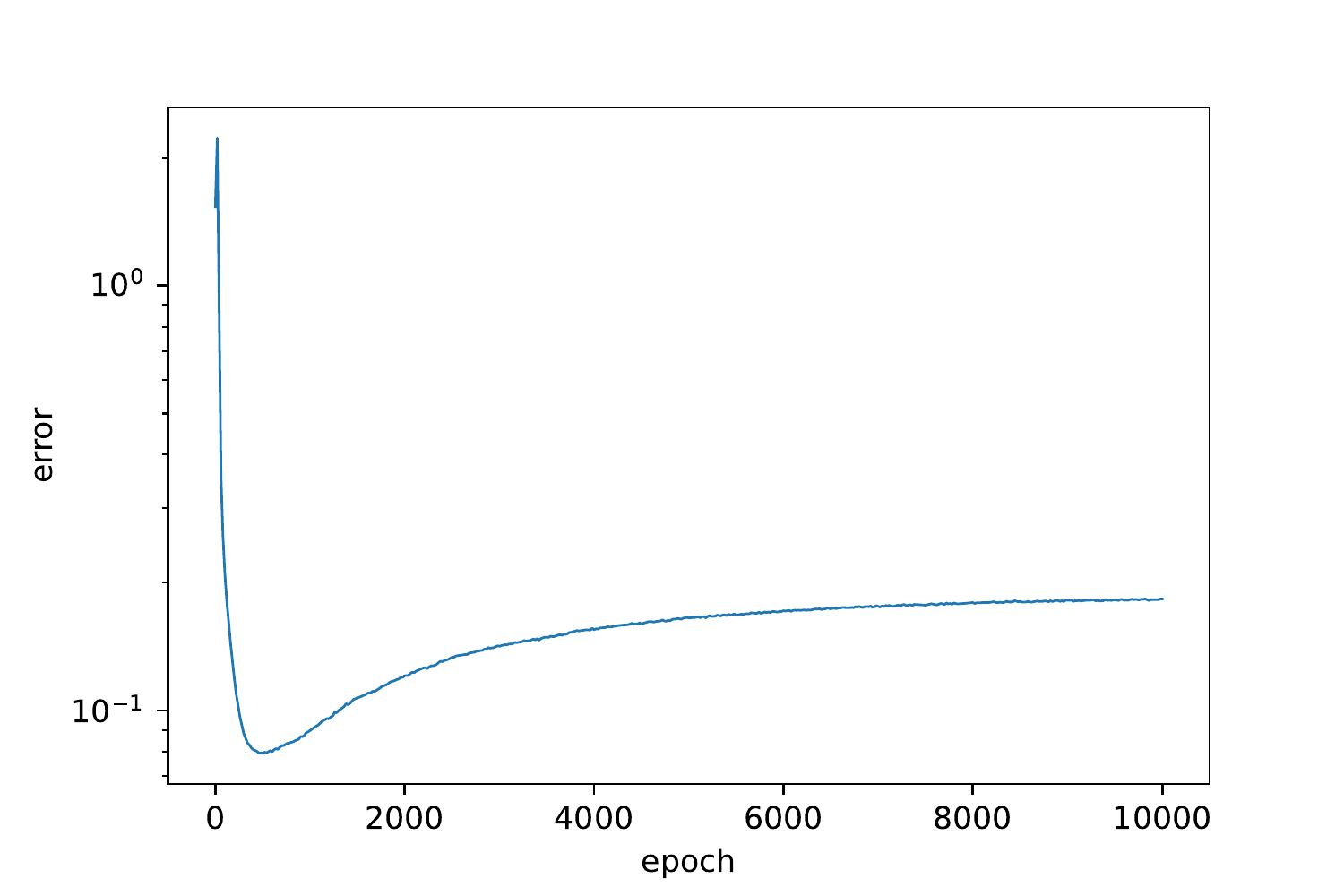}\hspace{-11pt}
		\includegraphics[scale=0.24]{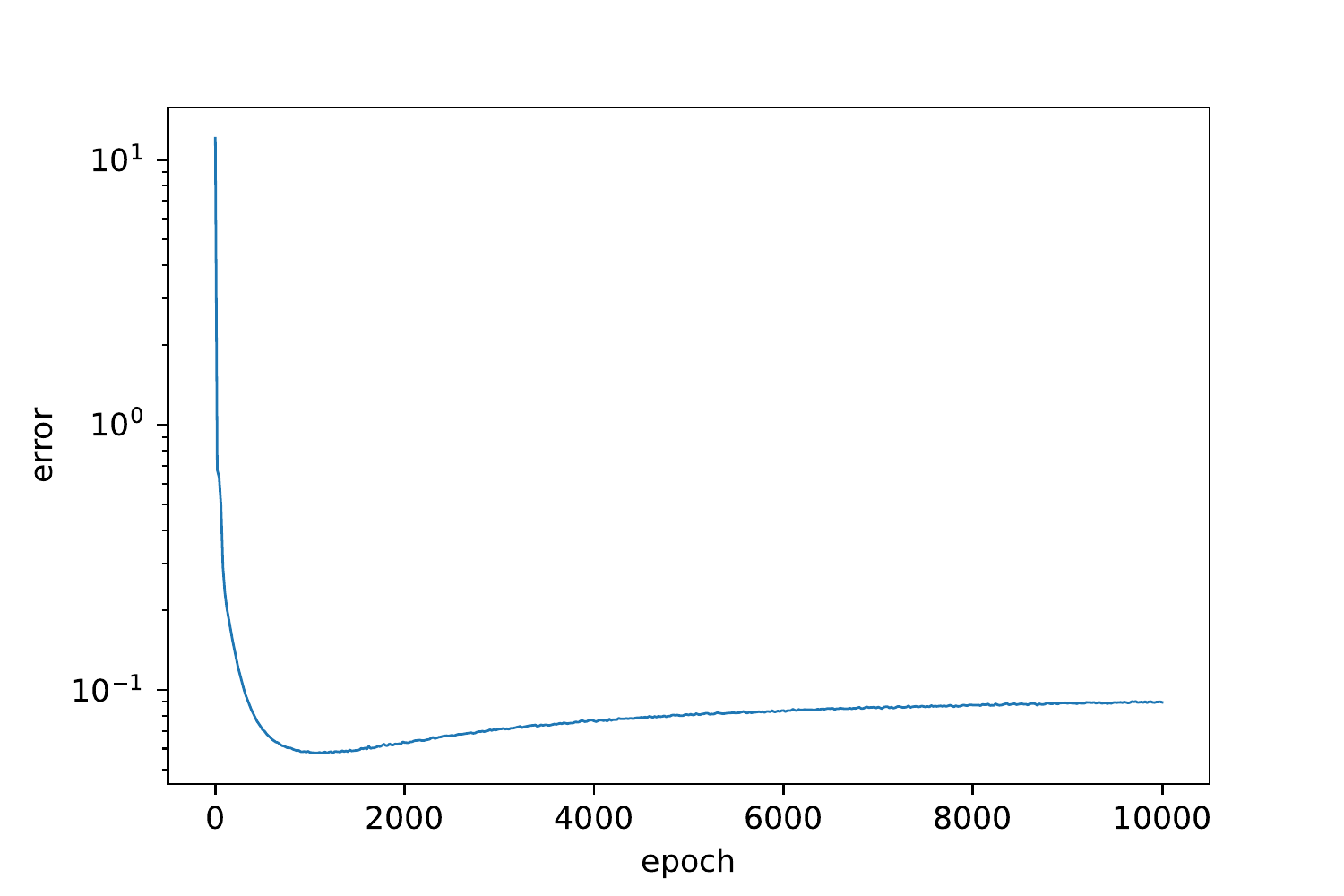}
		\caption{Validation error history for $\sigma(p)=\text{ReLU}(p)$ with $\psi_{1,k}(\boldsymbol x)$ and $k=2,3,4,5$.}
		\label{3DReLUBen}
\end{figure*}

\begin{figure*}[!htbp]
		\centering		
		\includegraphics[scale=0.24]{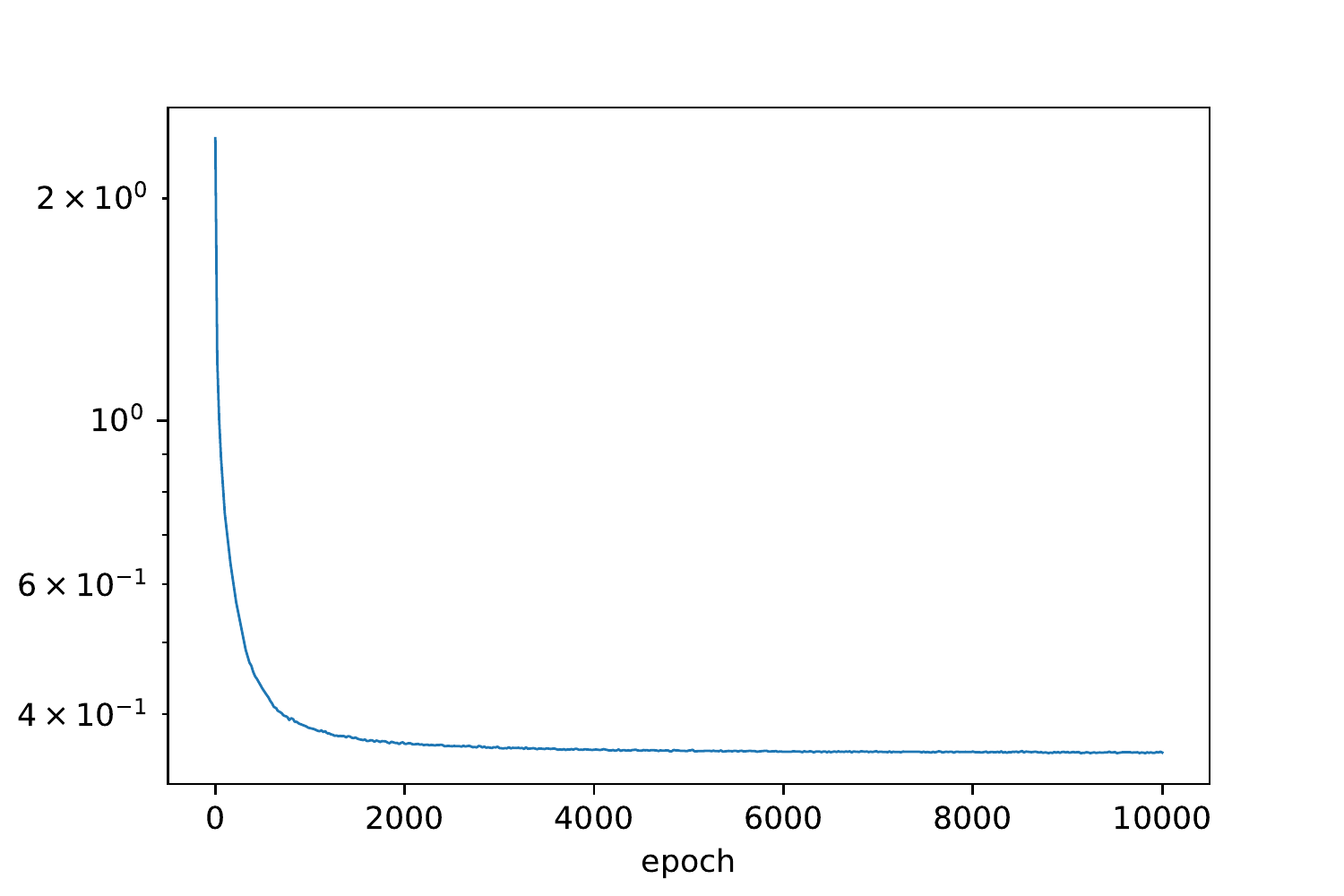}\hspace{-11pt}
		\includegraphics[scale=0.24]{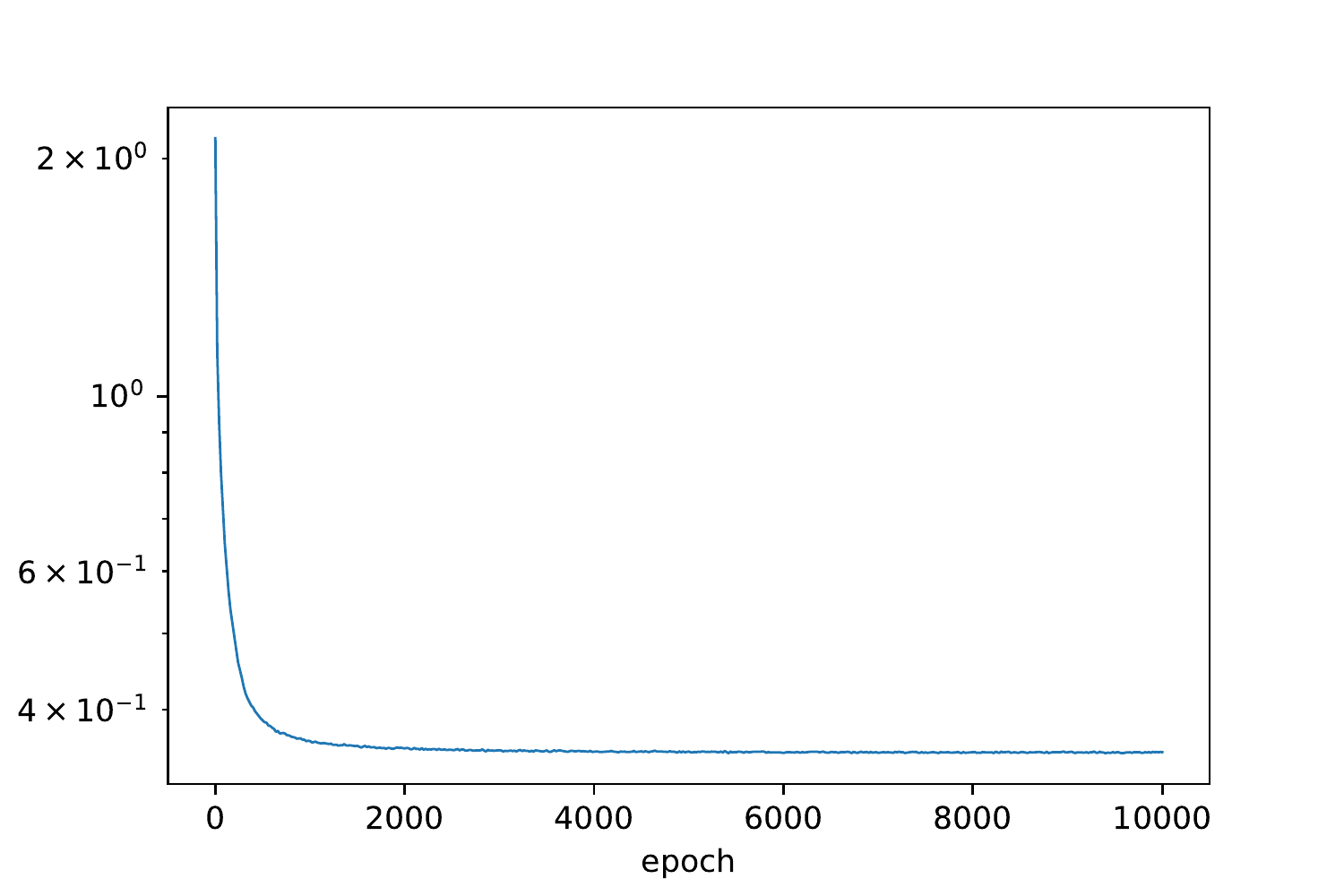} \hspace{-11pt}
		\includegraphics[scale=0.24]{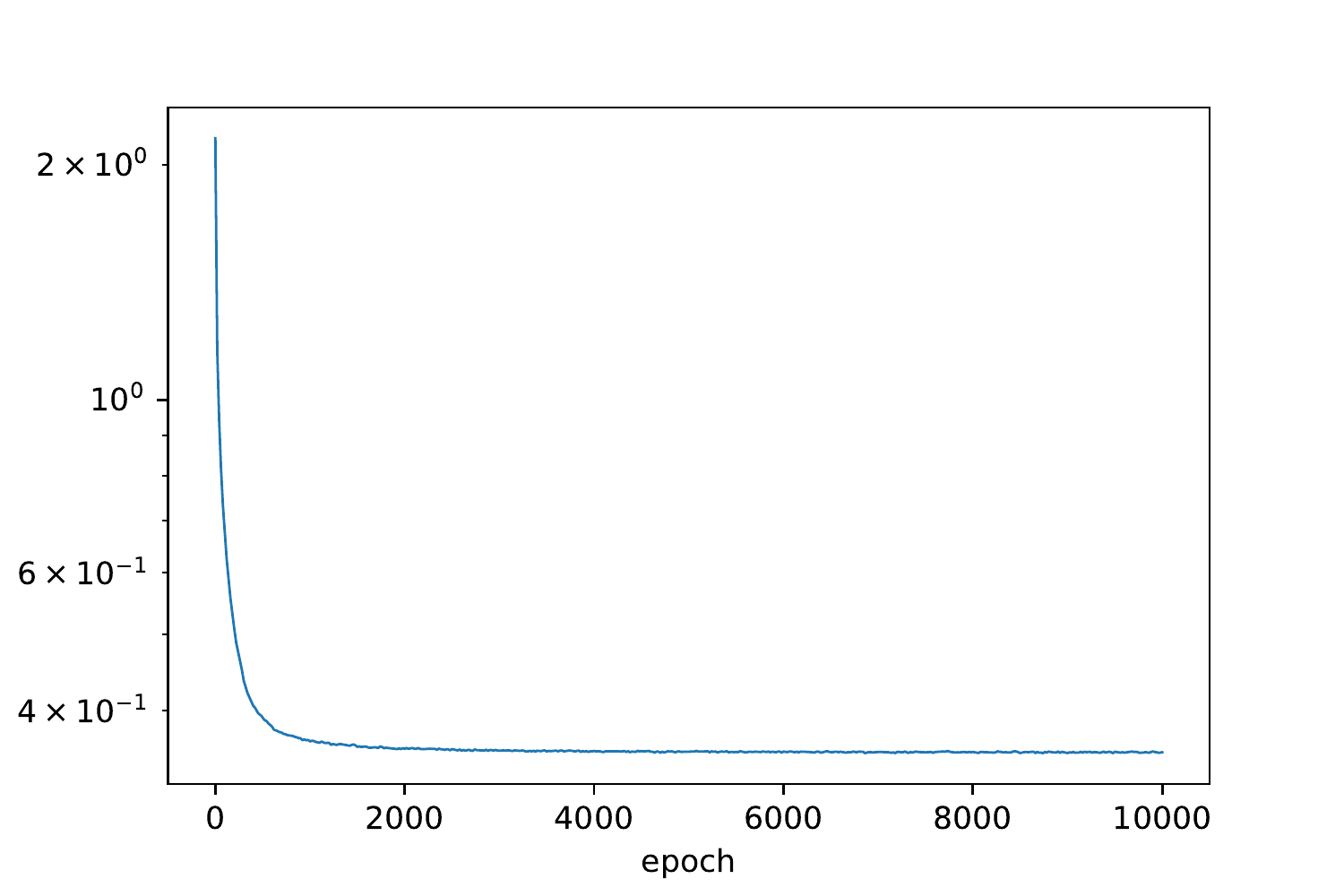}\hspace{-11pt}
		\includegraphics[scale=0.24]{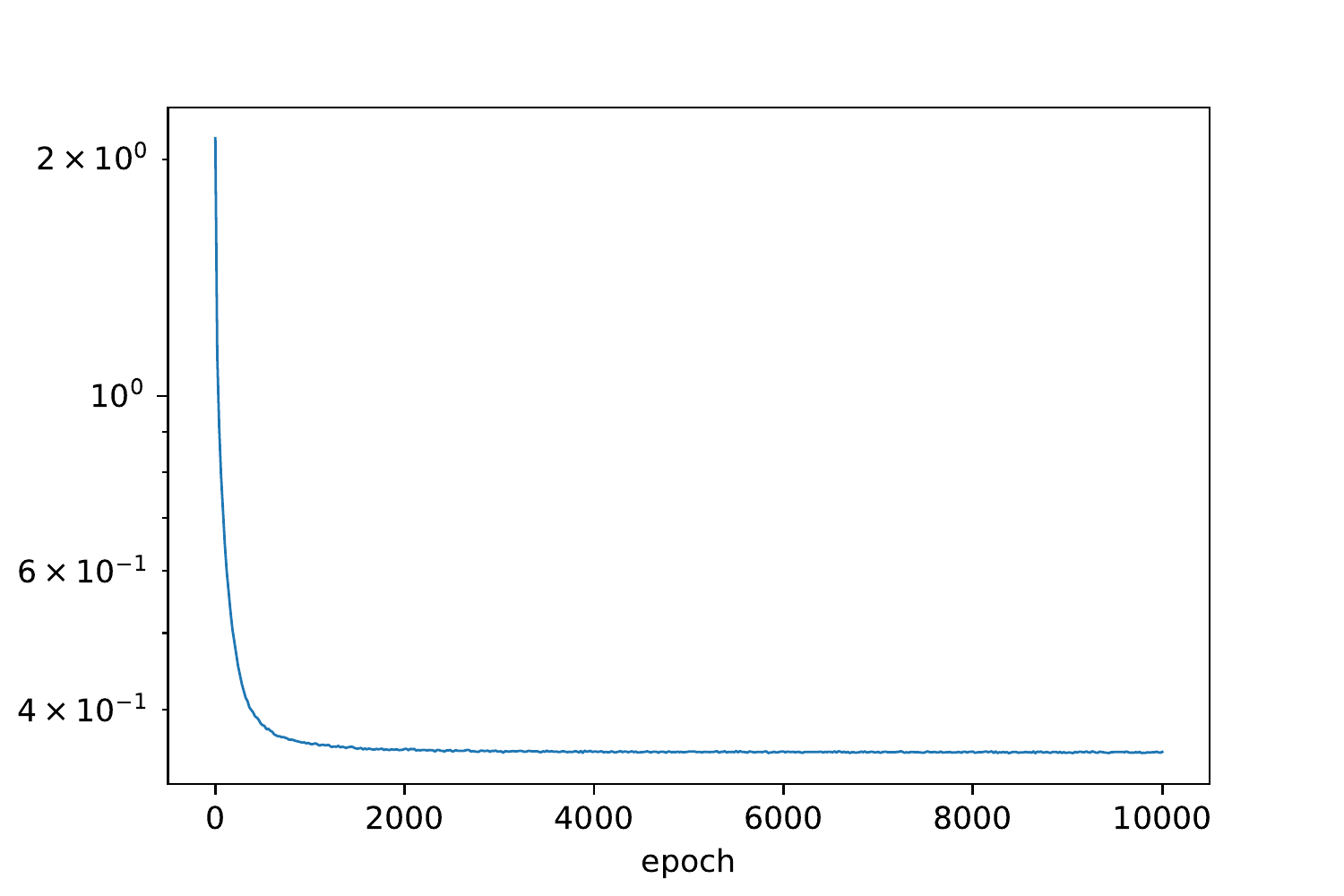}
		\caption{Validation error history for $\sigma(p)=\text{Hat}(100p)$ with $\psi_{1,k}(\boldsymbol x)$ and $k=2,3,4,5$.}
		\label{3DHatBen}
\end{figure*}

\begin{figure*}[!htbp]
		\centering		
		\includegraphics[scale=0.24]{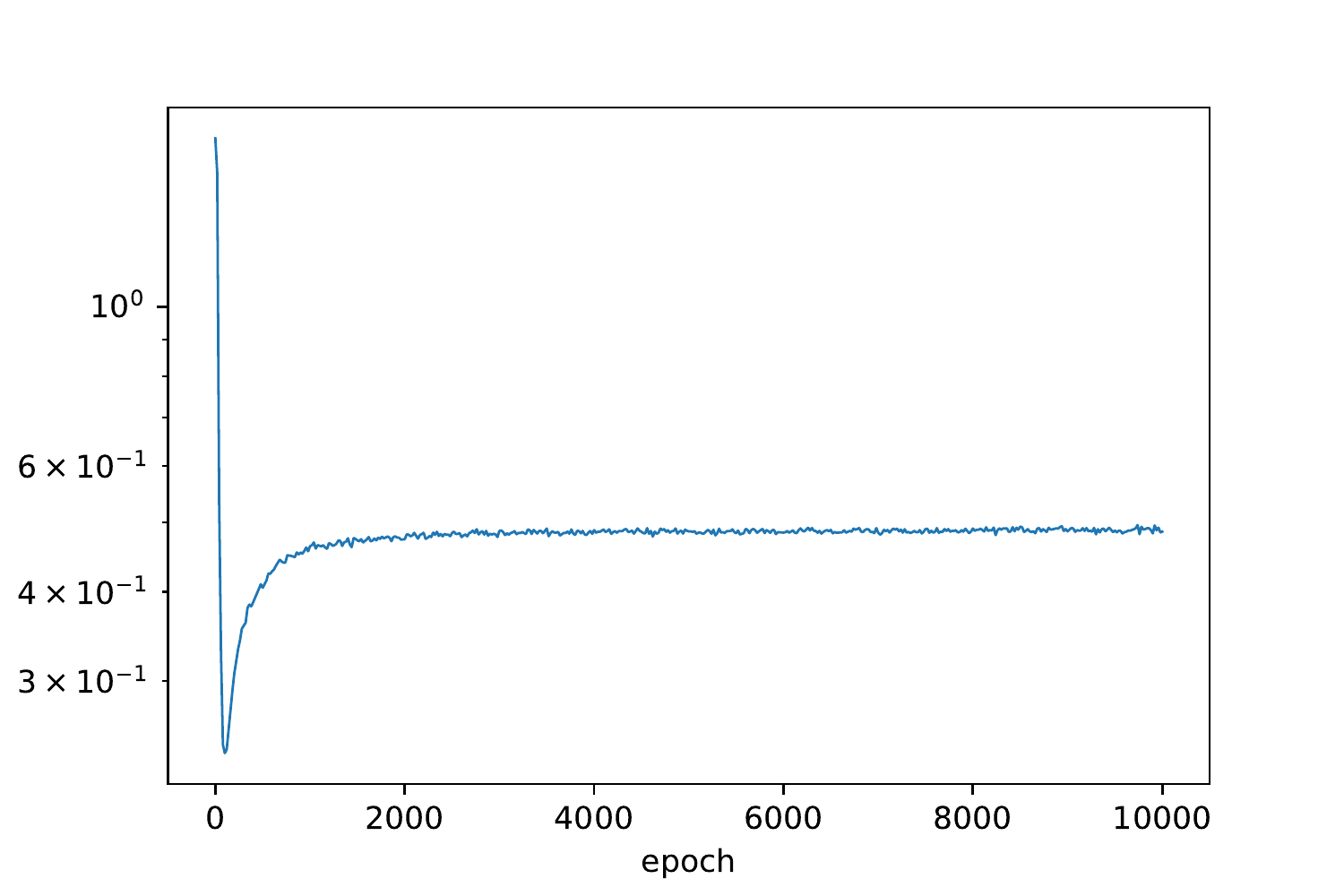}\hspace{-11pt}
		\includegraphics[scale=0.24]{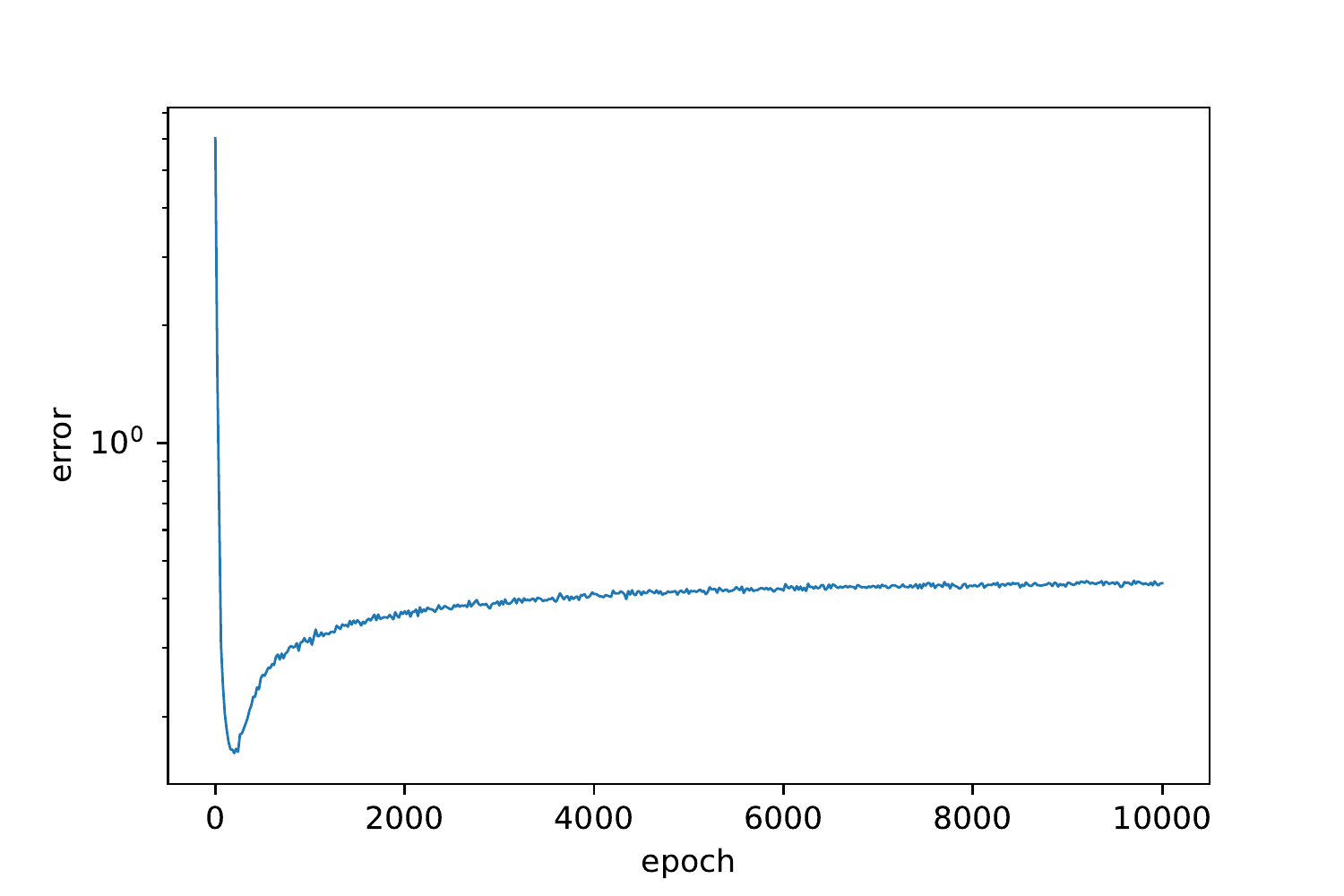}\hspace{-11pt}
		\includegraphics[scale=0.24]{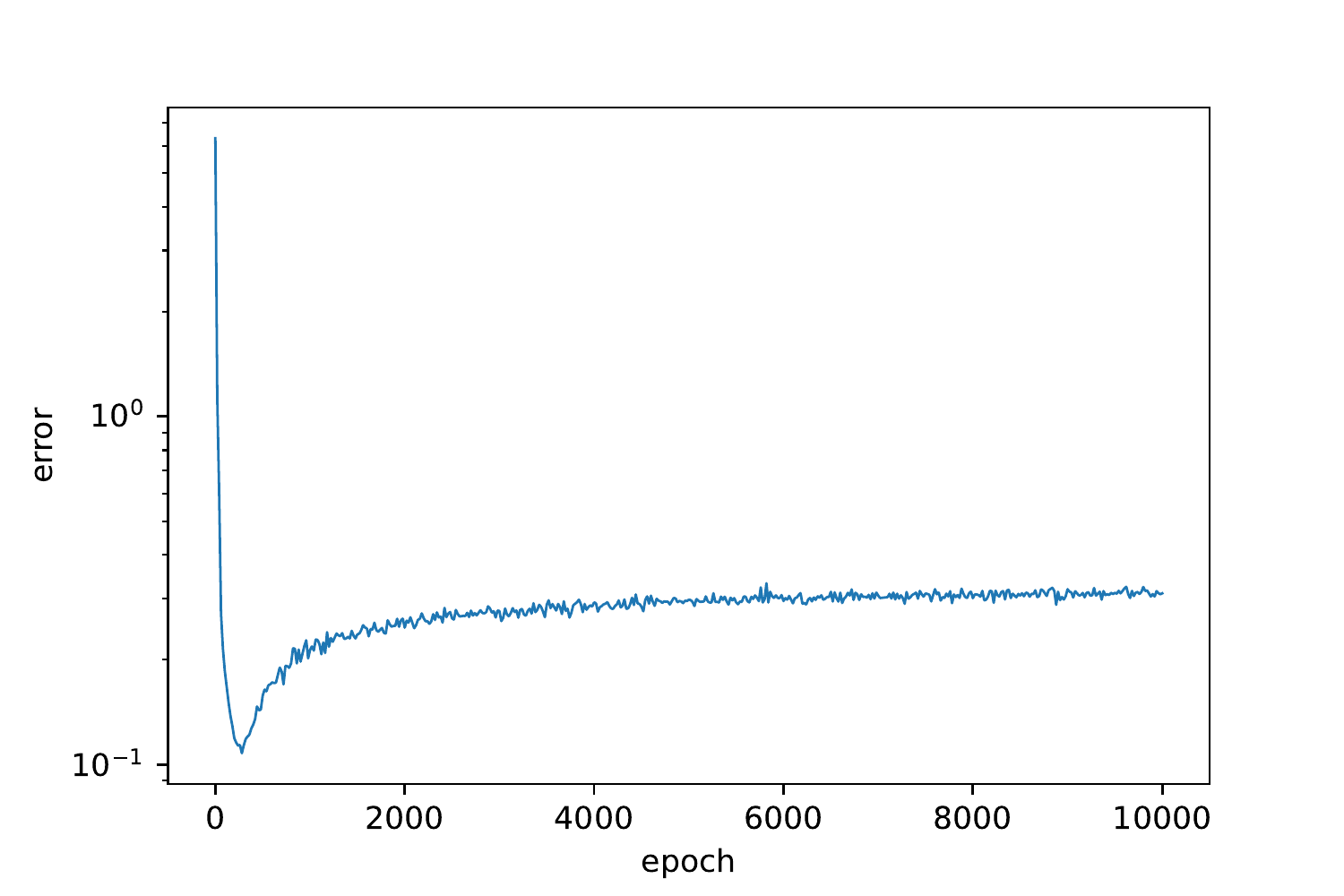}\hspace{-11pt}
		\includegraphics[scale=0.24]{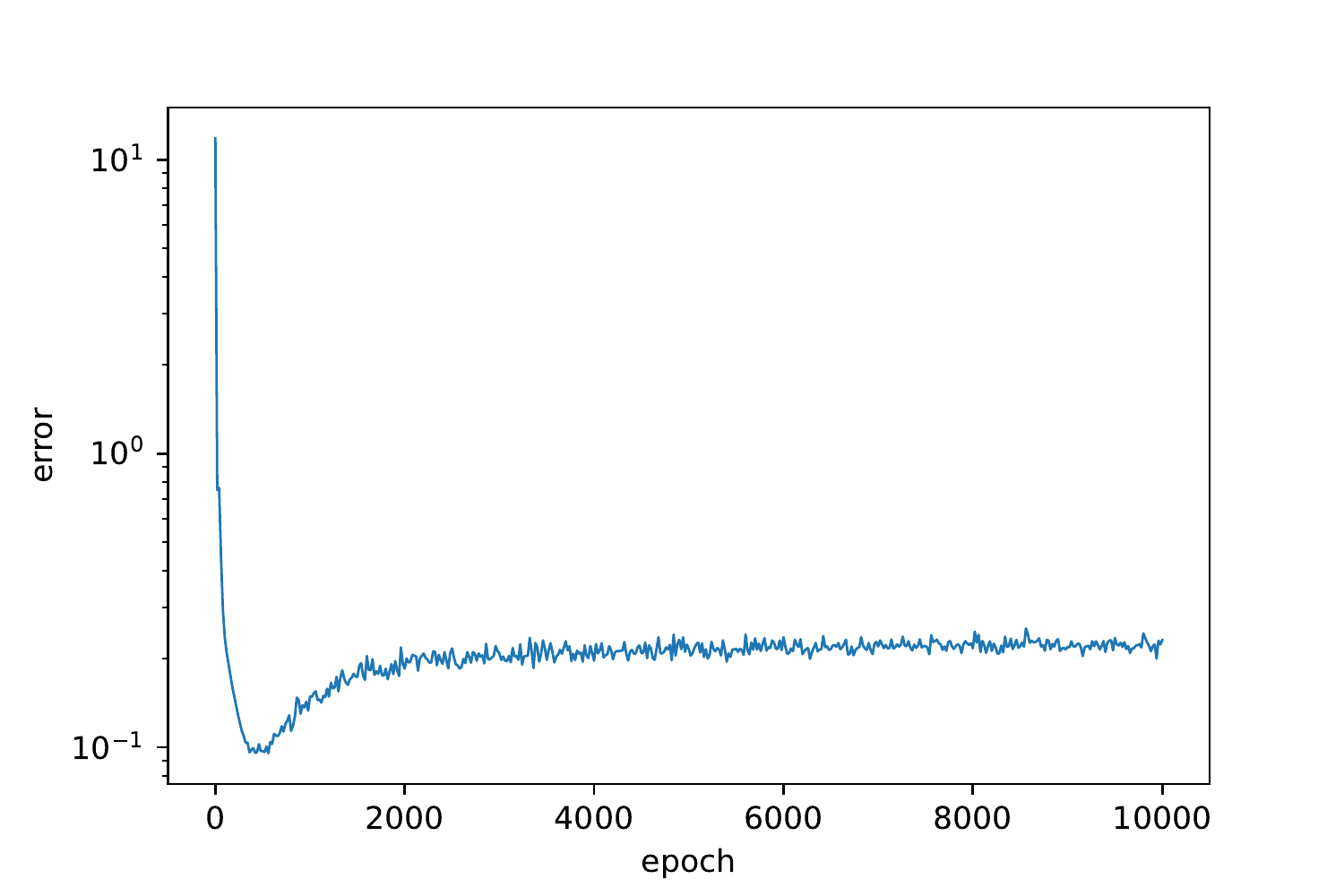}
		\caption{Validation error history for $\sigma(p)=\text{ReLU}(p)$ with $\psi_{2,k}(\boldsymbol x)$ and $k=2,3,4,5$.}
		\label{3DReLUBen1}
\end{figure*}

\begin{figure*}[!htbp]
		\centering		
		\includegraphics[scale=0.24]{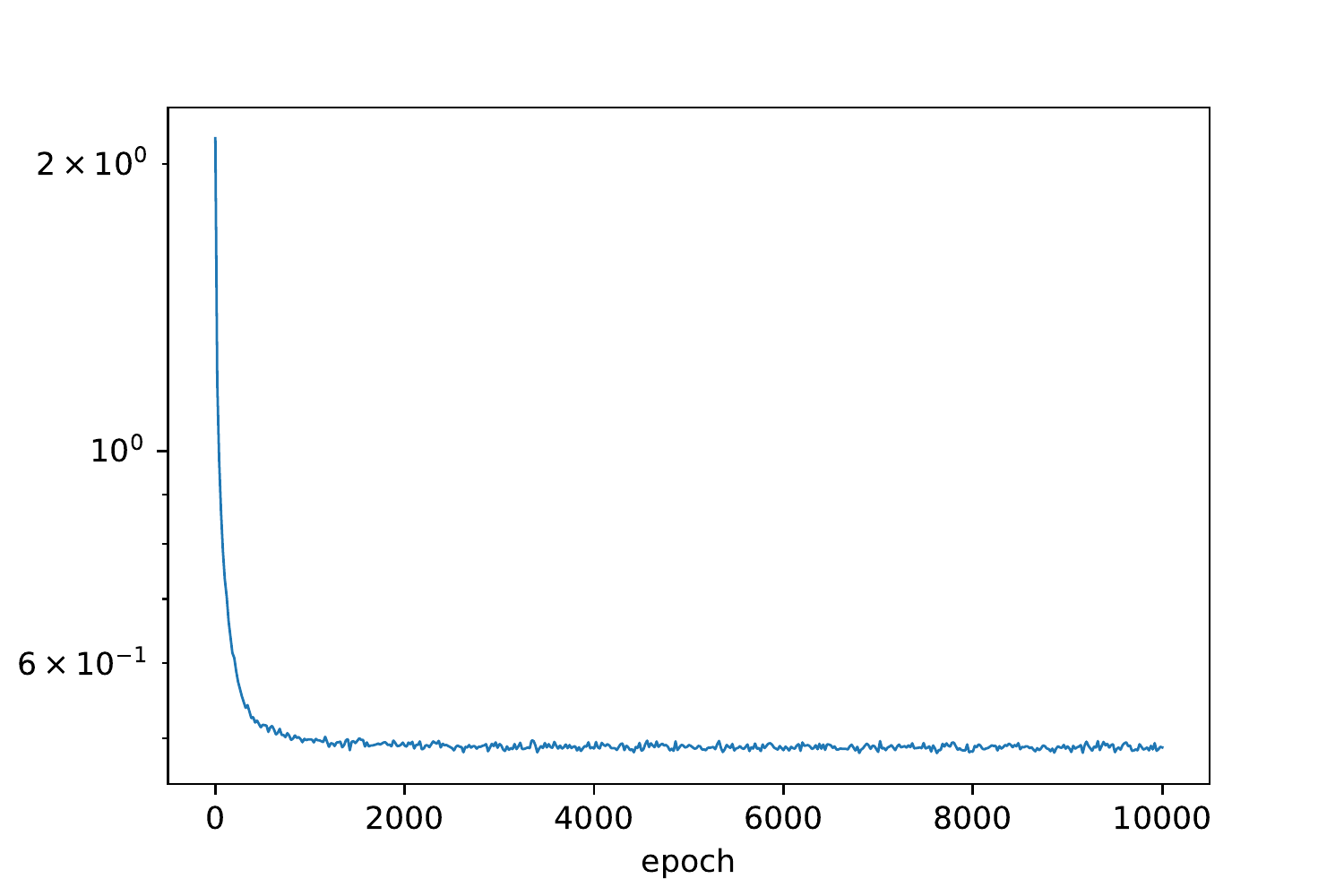}\hspace{-11pt}
		\includegraphics[scale=0.24]{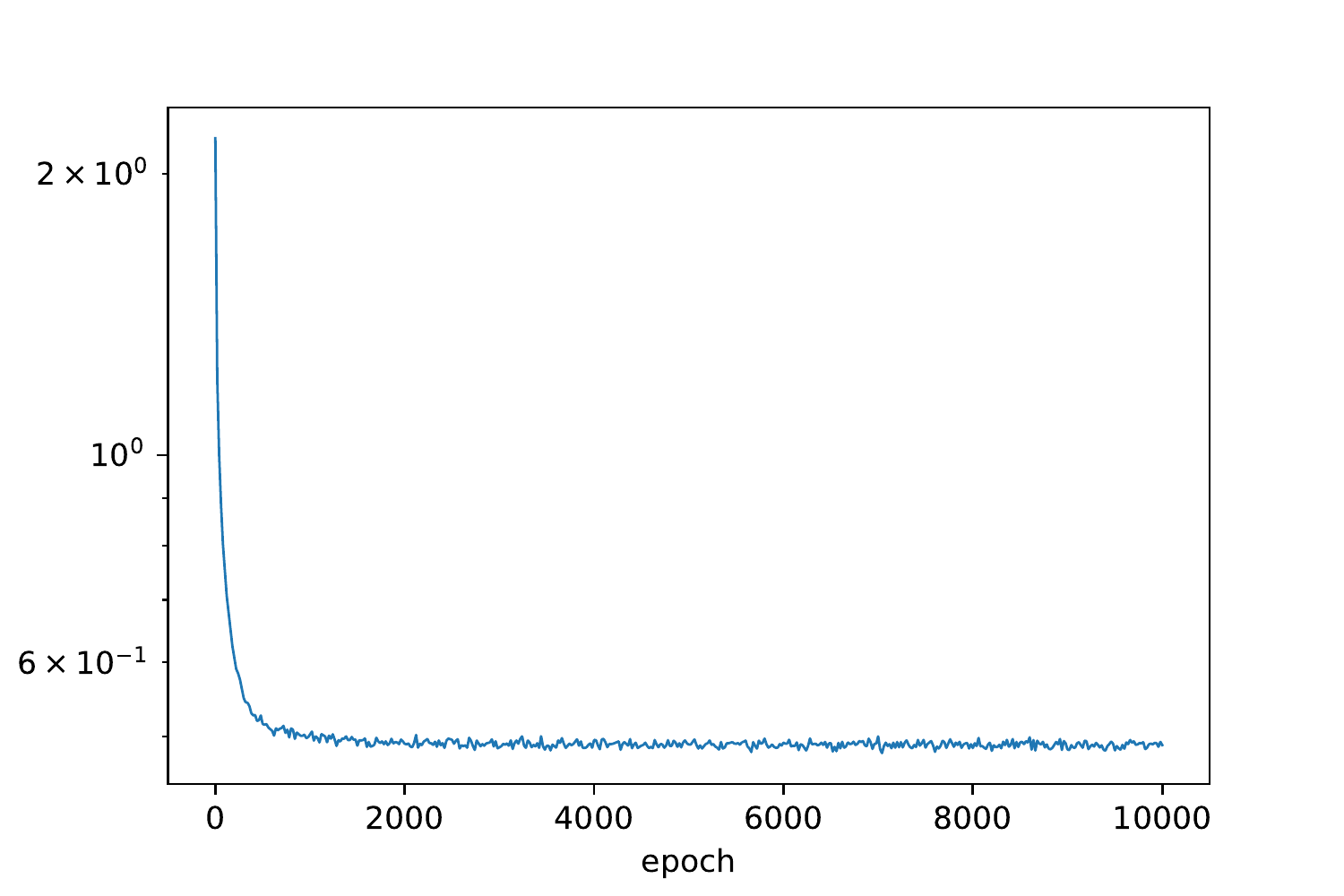} \hspace{-11pt}
		\includegraphics[scale=0.24]{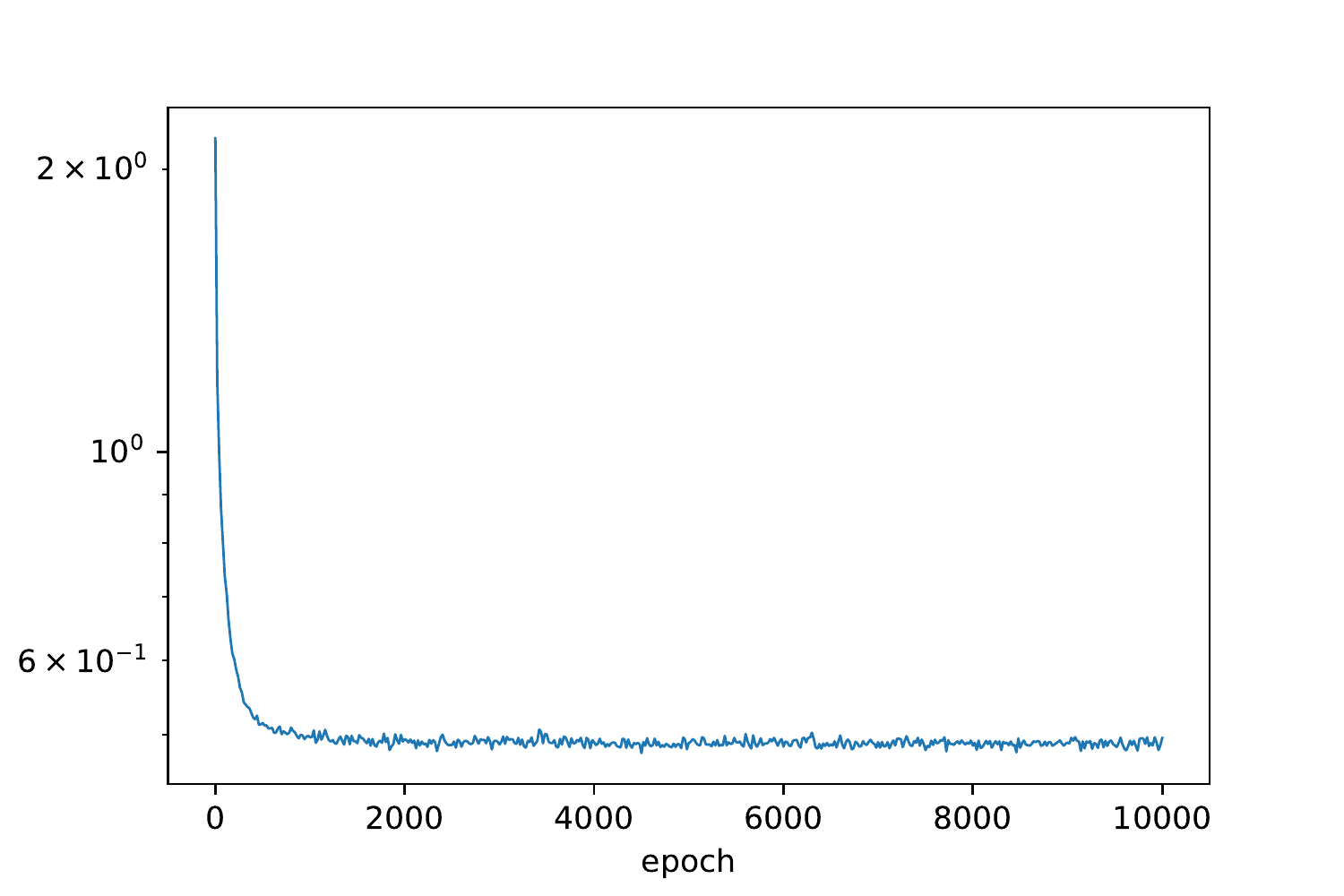}\hspace{-11pt}
		\includegraphics[scale=0.24]{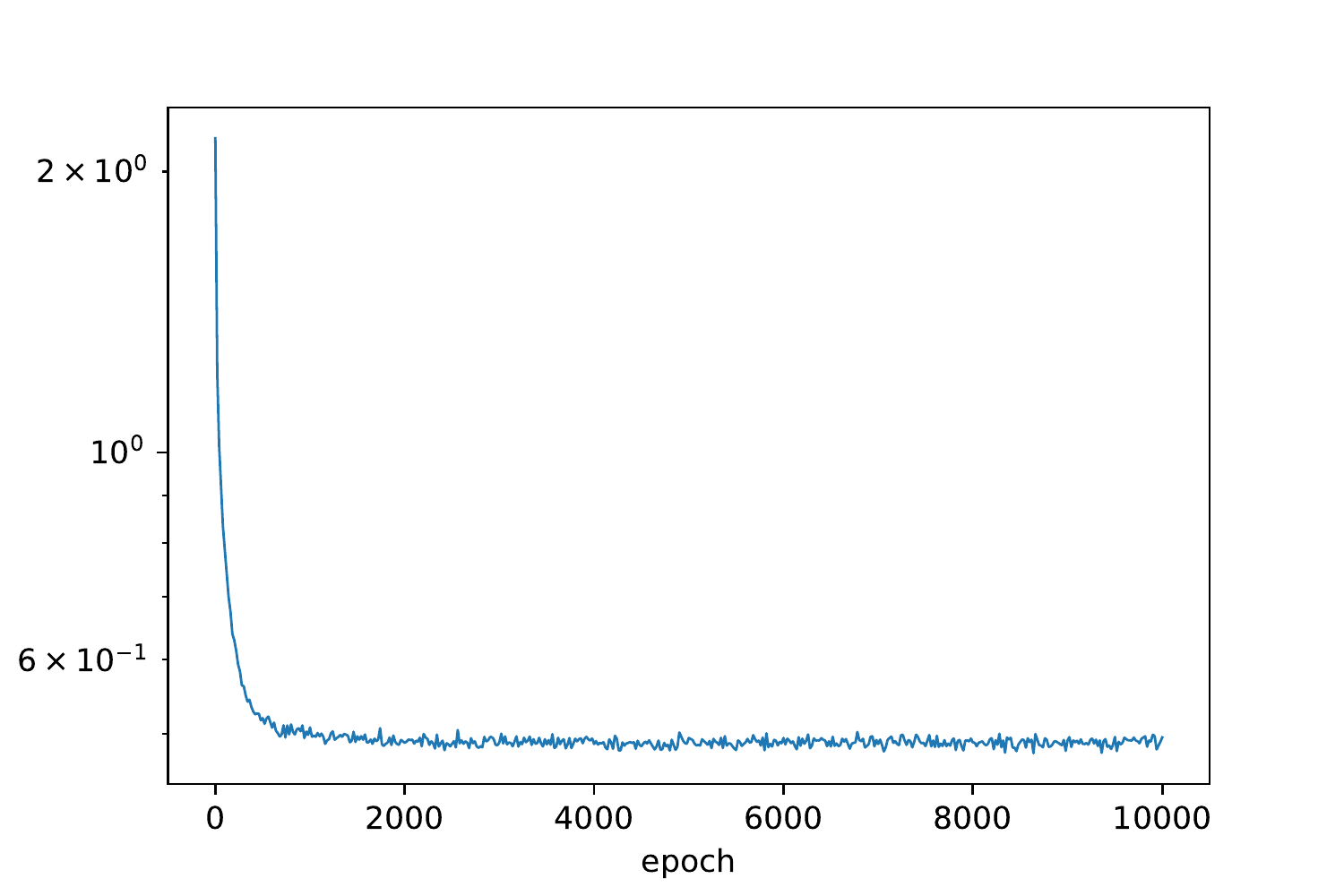}	
		\caption{Validation error history for $\sigma(p)=\text{Hat}(100p)$ with $\psi_{2,k}(\boldsymbol x)$ and $k=2,3,4,5$.}
		\label{3DHatBen1}
\end{figure*}

From the Figure \ref{3DReLUBen}, Figure  \ref{3DHatBen}, Figure \ref{3DReLUBen1} and Figure  \ref{3DHatBen1}, we can see that the numerical results are similar to the two dimension case and the three dimension case in Experiment \ref{3Dcase1}. 
\end{experiment}

\begin{experiment}
\normalfont
We consider to fit the target function  \cite{cai2020phase}
\begin{equation*}
	u(x) = \left\{
	\begin{aligned}
		10(\sin(x) + \sin(3x)),  x\in[-\pi,0];\\
		10(\sin(23x)+\sin(137x)+\sin(203x)),  x\in[0,\pi].
	\end{aligned}
	\right.
\end{equation*}


\begin{itemize}
\item ReLU-DNN with phase shift by Cai et al. in \cite{cai2020phase}: 16 models of ReLU-DNN which have the size of 1-40-40-40-40-1 and Fourier transform are used to train different frequency components of the target function. Training data are the evenly mesh with 1000 grids from $-\pi$ to $\pi$. They train the mean square loss with Adam and the learning rate is 0.002.

\item Our method:  We use only one model with one hidden layer and width of 25000. In our model, the activation function is $\sigma(p)=\text{hat}(200p)$. Learning rate is 0.001 and decrease to its half every 250 epochs. The weights of outer layers are sampled following the uniform distribution $\mathcal{U}(-7,7)$, and all the other parameters are sampled following the uniform distribution $\mathcal{U}(-0.95,0.95)$. Training data are the evenly mesh with 1000 grids from $-\pi$ to $\pi$. We train the mean square loss with Adam.
\end{itemize}
\begin{figure}[!htbp]
		\centering		
		\includegraphics[scale=0.24]{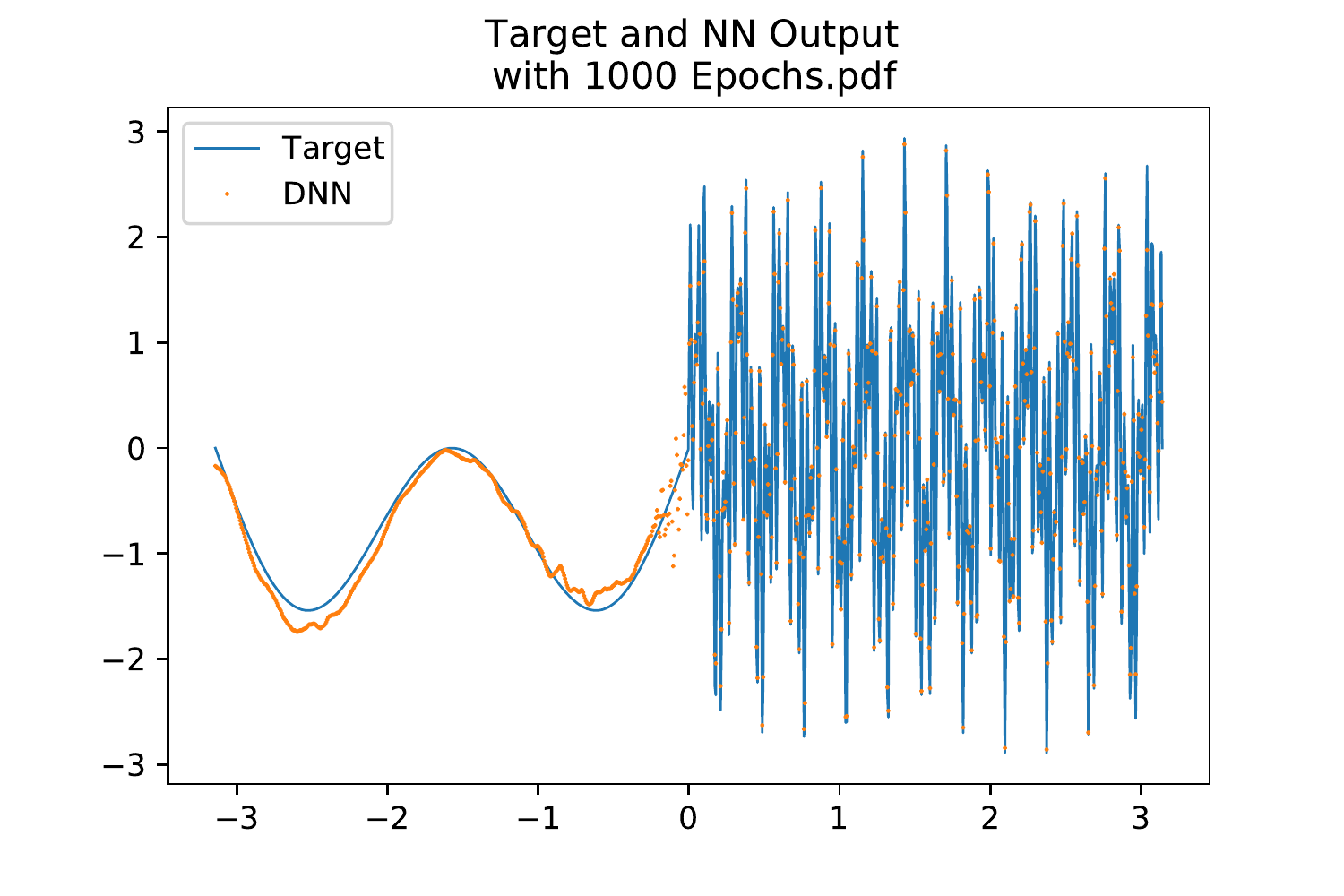}\hspace{-11pt}
		\includegraphics[scale=0.24]{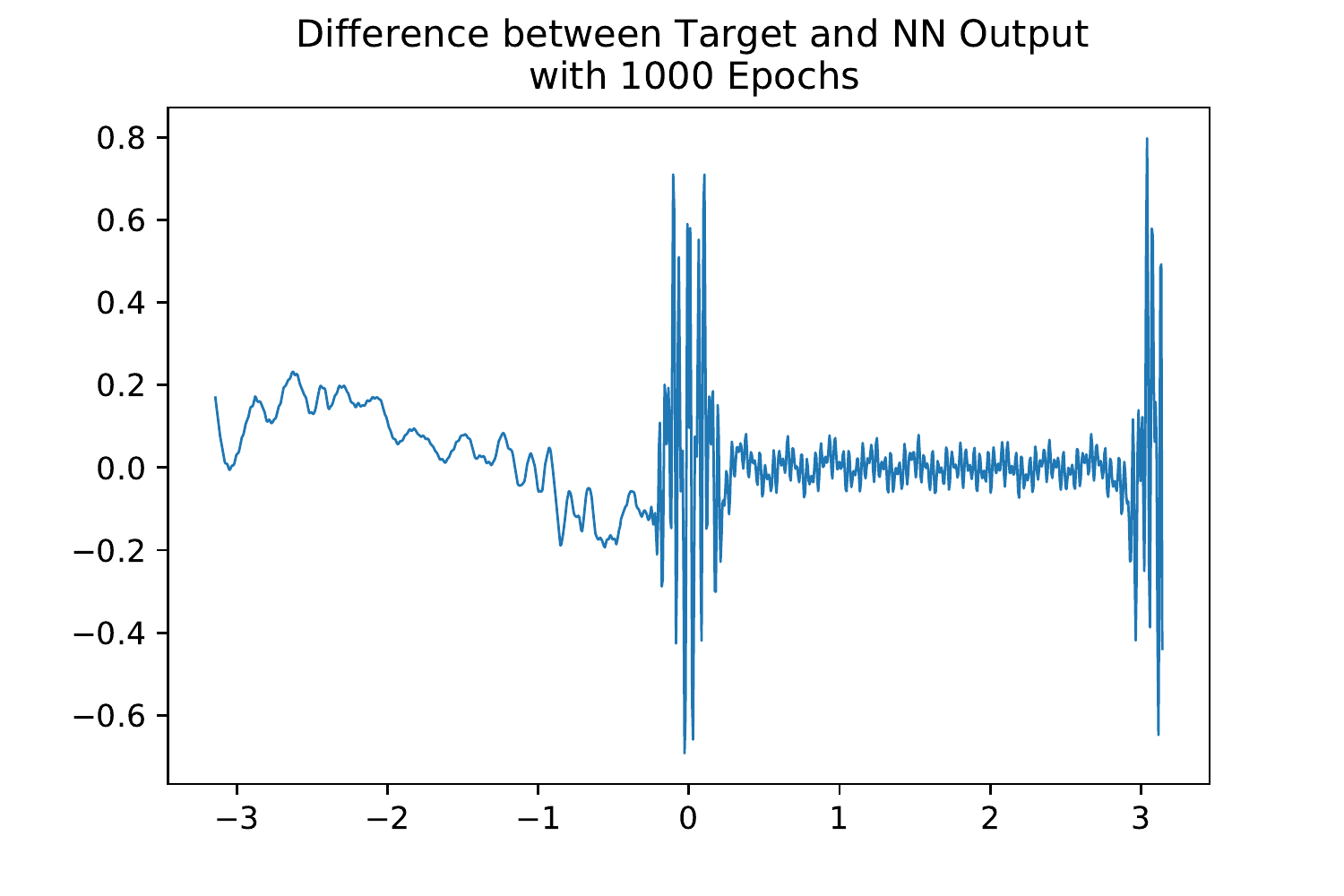} \hspace{-11pt}
		\includegraphics[scale=0.24]{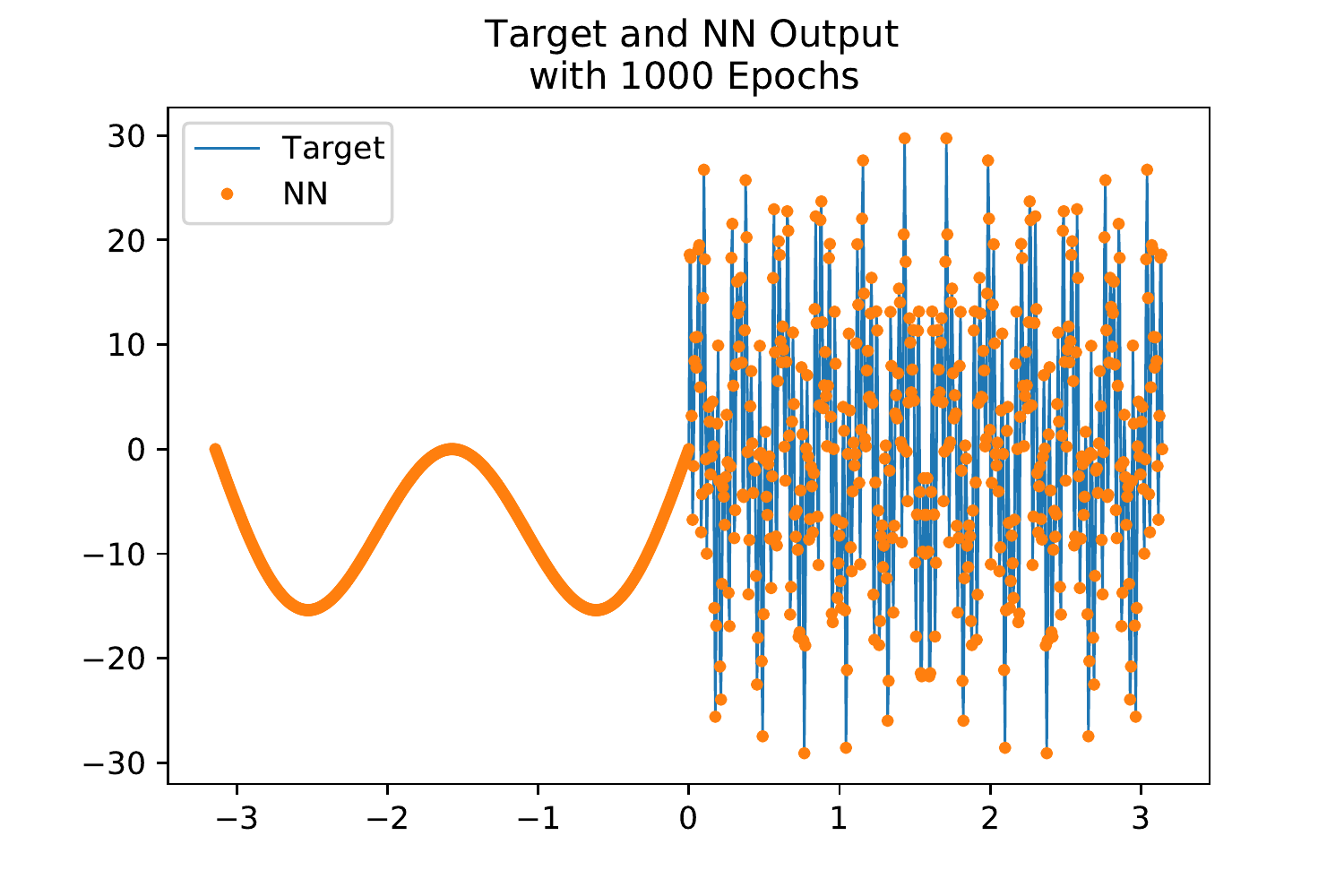}\hspace{-11pt}
		\includegraphics[scale=0.24]{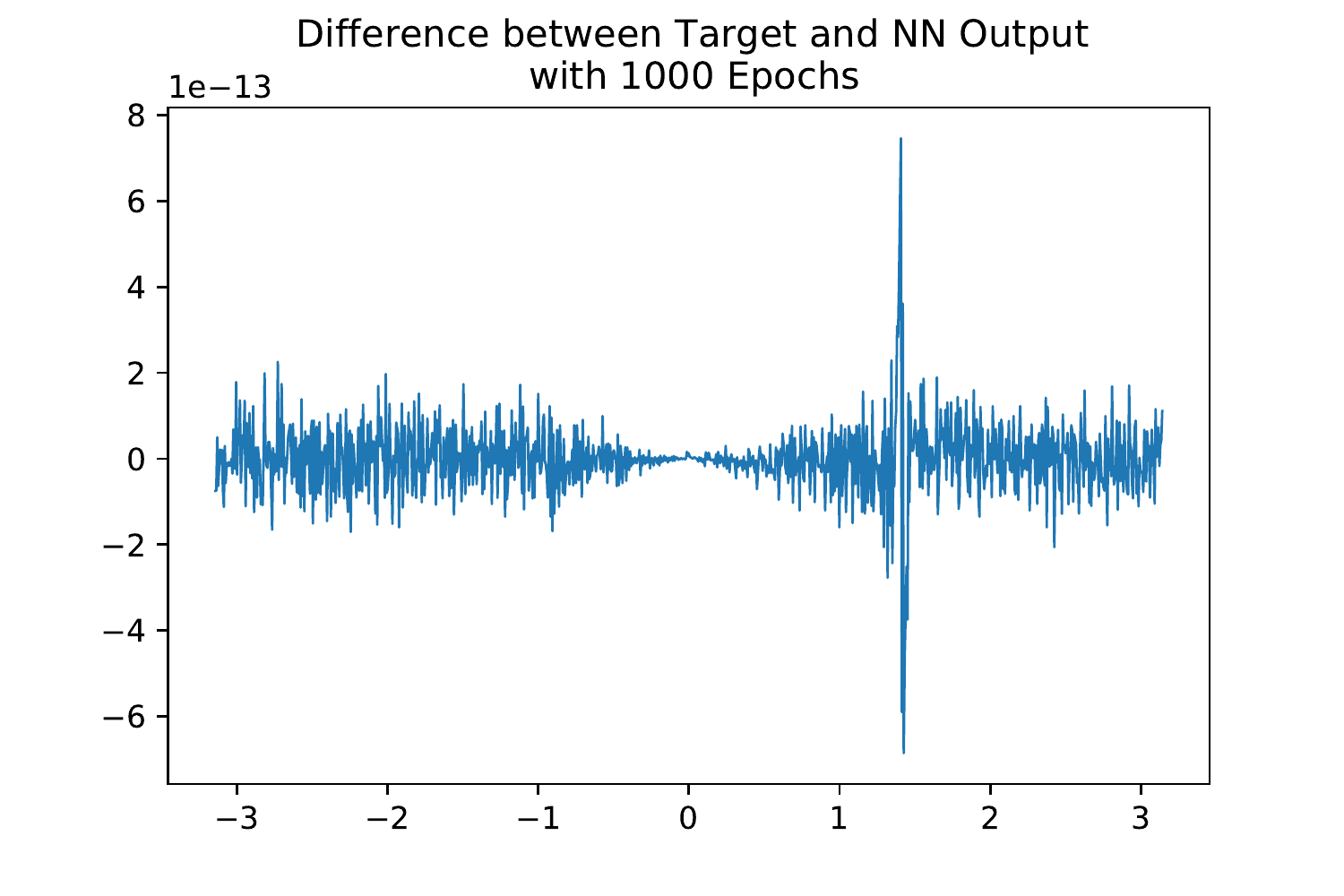}
		\caption{ReLU-DNN with Phase Shift in \cite{cai2020phase} (left two), Our method (right two).}
		\label{Caiexample}
\end{figure}
The model of our method has almost the same amount of parameters in the method of the ReLU-DNN with phase shift in \cite{cai2020phase} which includes 16 models. From the Figure \ref{Caiexample}, after 1000 epochs for both methods, we can see that the difference  between the output neural network function obtained by method of the ReLU-DNN with phase shift in \cite{cai2020phase} 
and the target function $u(x)$ is around $10^{-1}$, while the difference  between the output neural network function obtained by our method  
and the target function $u(x)$ is around $10^{-13}$.  
\end{experiment}

\begin{experiment}
We consider fitting the  target function on $[0,1]$: 
$$
u(x) = \sin(2\pi x)+\sin(6\pi x)+\sin(10\pi x).
$$ 
We define the loss $L(f,u) = \int_{0}^{1} \left| f(x)-u(x) \right|^2 dx$ and evaluate this integral with accurate approximation. We use one-hidden-layer network with width of 128. The activation functions are ReLU and $hat(10x)$.  Both models
are trained with a learning rate of 0.0005. For the ReLU network $\displaystyle f(x)=\sum_{i=1}^n a_i \sigma(\omega_ix+b_i)$ with $\sigma(p)=\text{ReLU}(p)$, we initialize $\omega_i$'s as constant 1, and initialize $a_i$'s and $b_i$'s following a  uniform distribution on $[-1,1]$.   And all parameters are 
initialized following a uniform distribution on $[-1,1]$ for hat neural network.  Denote $\Delta_{f,u} (k) = \left| \hat{f}_k - \hat{u}_k 
\right|/\left|\hat{u}_k \right|$, where $k$ represents the frequency, $\left| \cdot \right|$ represents the 
norm of a complex number and $\hat{\cdot}$ represents Fourier 
transform. We also can evaluate the accurate Fourier transform for frequencies $k=1, 3, 5$. From the Fourier transform of $u(x)$, we select $\Delta_{f,u} (1)$, $\Delta_{f,u} (3)$, 
and $\Delta_{f,u} (5)$  to observe the convergent rate. 
And we apply the same thing to $f(x)$.
	\begin{figure}[!htbp]
 \centering
  \includegraphics[scale=0.28]{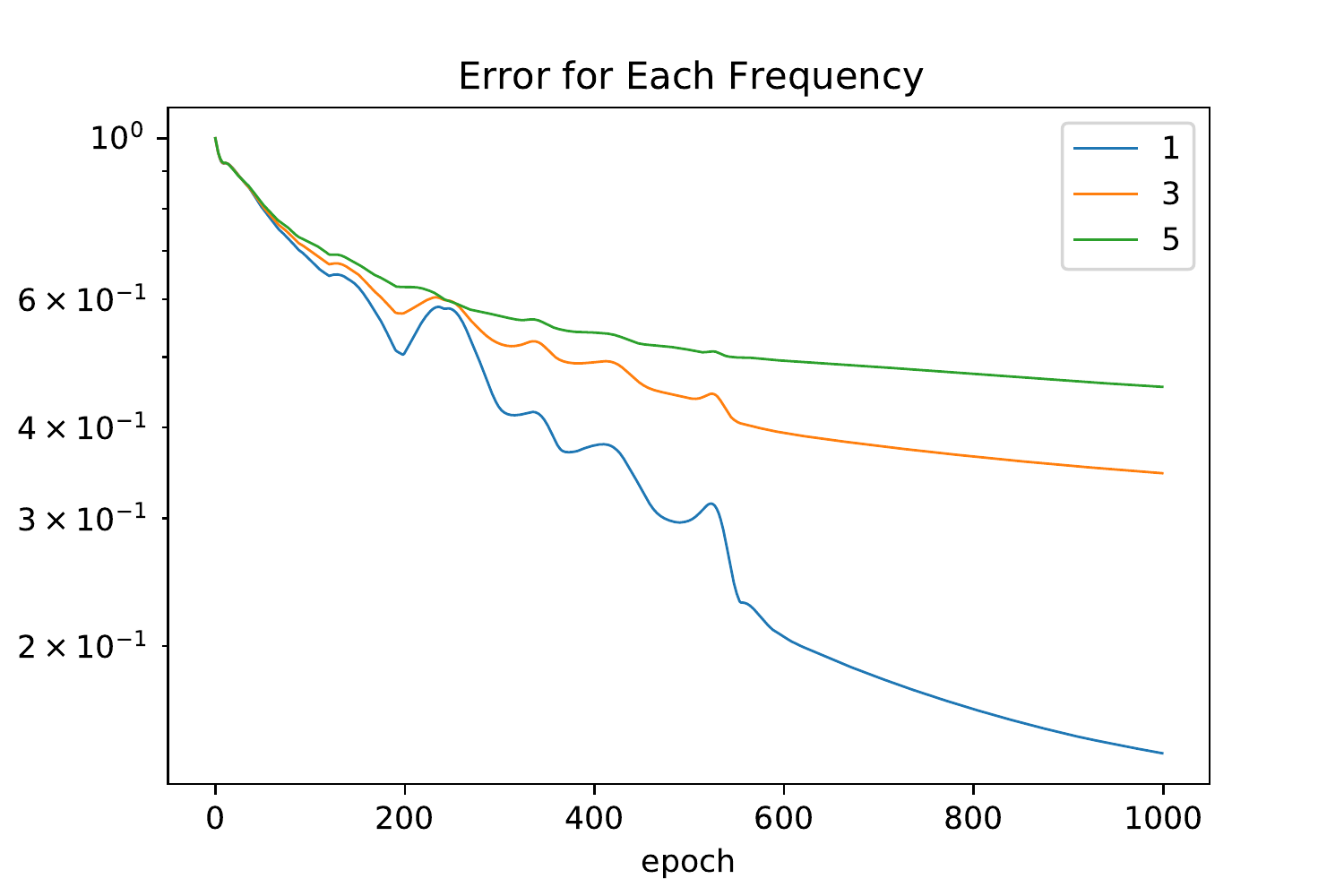}\hspace{-8pt}
 \includegraphics[scale=0.28]{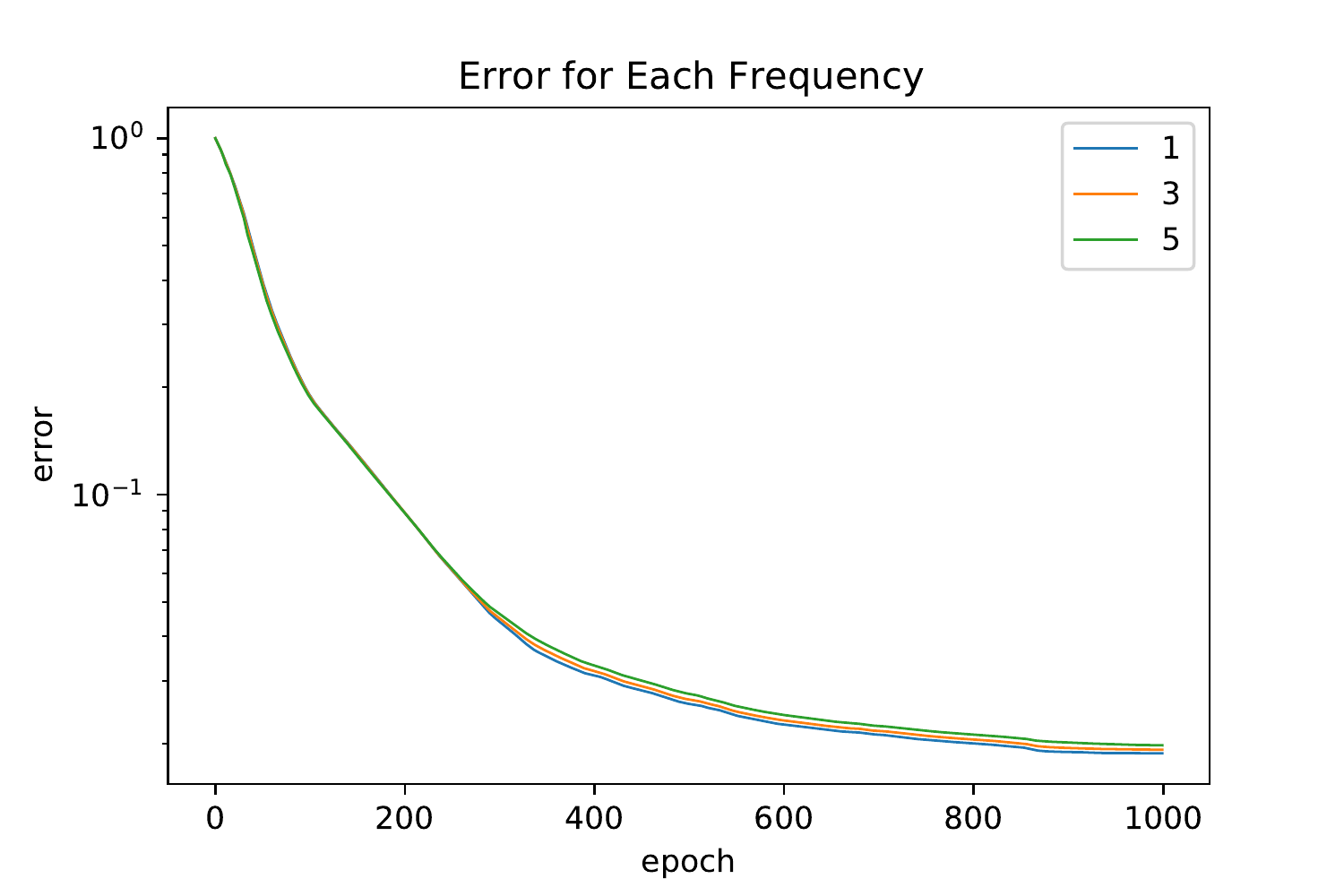}
 	\vspace{-3pt}
	\caption{$\sigma(p)=\text{ReLU}(p)$ (left), $\sigma(p)=\text{ Hat}(10p)$ (right)}
	\label{exactreluvshat}
\end{figure}
Figure \ref{exactreluvshat} shows the convergent process of each frequency component of 
ReLU neural network and Hat neural network.  Green line, yellow line, and blue line denote the convergent process of $\Delta_{f,u} (1)$, $\Delta_{f,u} (3)$, and $\Delta_{f,u} (5)$ respectively. 
 Spectral bias is obviously 
observed for ReLU neural network (see left of Figure \ref{exactreluvshat}), 
while the spectral bias does not hold for Hat neural network as shown in right of Figure \ref{exactreluvshat}. 
\end{experiment}

\begin{experiment}
We consider fitting the  target function on $[-\pi, \pi]$:
\begin{equation}
u(x) = \sin(x) + \sin(3x)+\sin(5x).
\end{equation}
We use shallow neural network models with activation function $\sigma(p)=sin(p)$ and $\sigma(p)=hat(p)$. 
Both models have only one hidden layer with size 8000.
The mean square error (MSE) function is defined as
\begin{equation}\label{MSEloss}
L(f,u) = \frac1N\sum_{i=1}^{N} (f(x_i)-u(x_i))^2,
\end{equation}
where $x_i$ is the uniform grid of size $N=201$ from $[-\pi,\pi]$. Both models
are trained with a learning rate of 0.0002. And all parameters are 
initialized following a Gaussian distribution with mean 0 and 
standard deviation 0.8 for both sin neural network and hat neural network. 
Denote $\Delta_{f,u} (k) = \left| \hat{f}_k - \hat{u}_k 
\right|/\left|\hat{u}_k \right|$, where $k$ represents the frequency, $\left| \cdot \right|$ represents the 
norm of a complex number and $\hat{\cdot}$ represents discrete Fourier 
transform. From the Fourier transform of $u(x)$, we select $\Delta_{f,u} (1)$, $\Delta_{f,u} (3)$, 
and $\Delta_{f,u} (5)$  to observe the convergent rate. 
And we apply the same thing to $f(x)$.
\begin{figure}[!htbp]
\centering
\begin{minipage}{.33\textwidth}
  \centering
  \centering
	\includegraphics[scale=0.28]{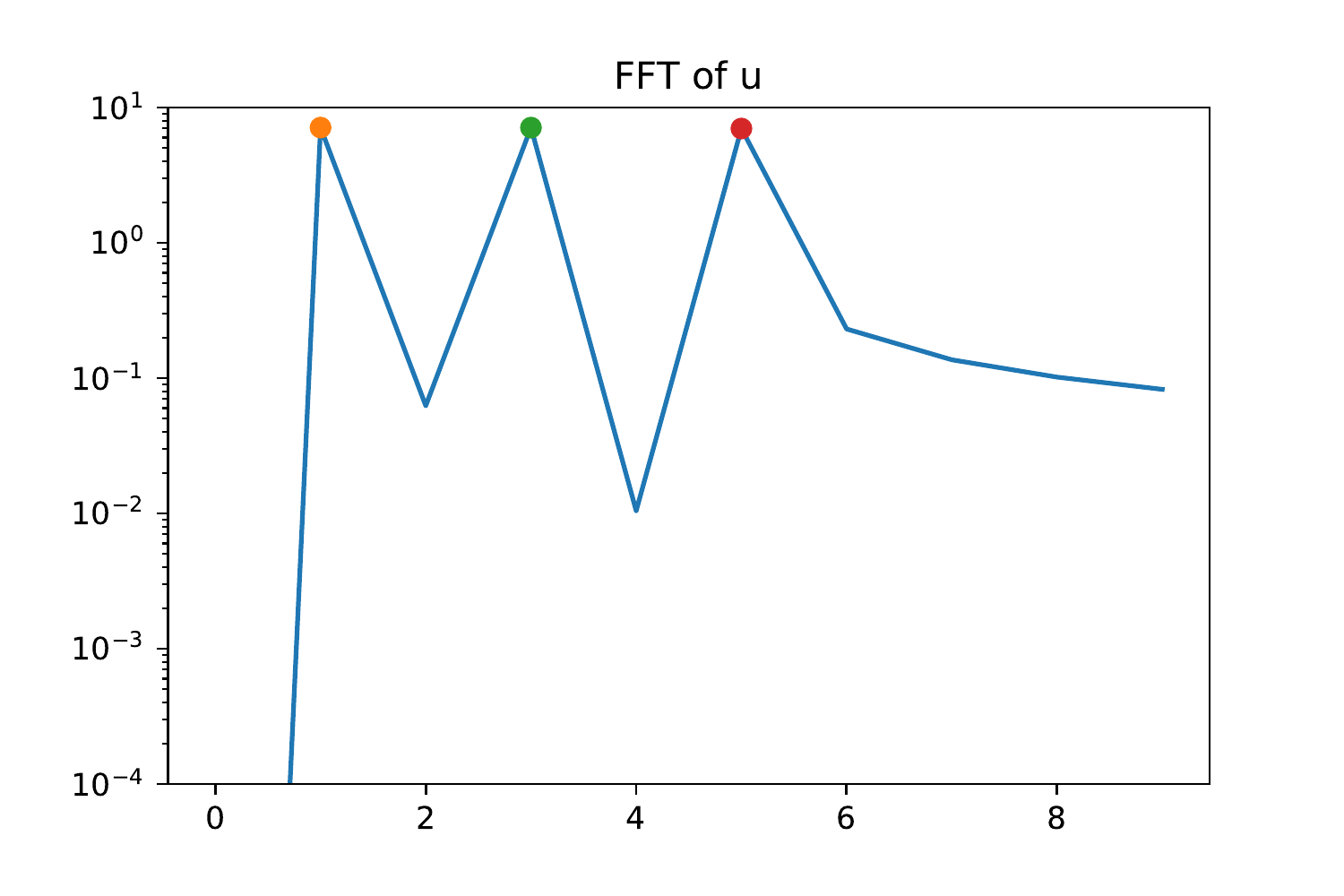}
	\vspace{-3pt}
	\caption{Fourier transform of $u$}
	\end{minipage}%
\begin{minipage}{.77\textwidth}
 \centering
  \includegraphics[scale=0.28]{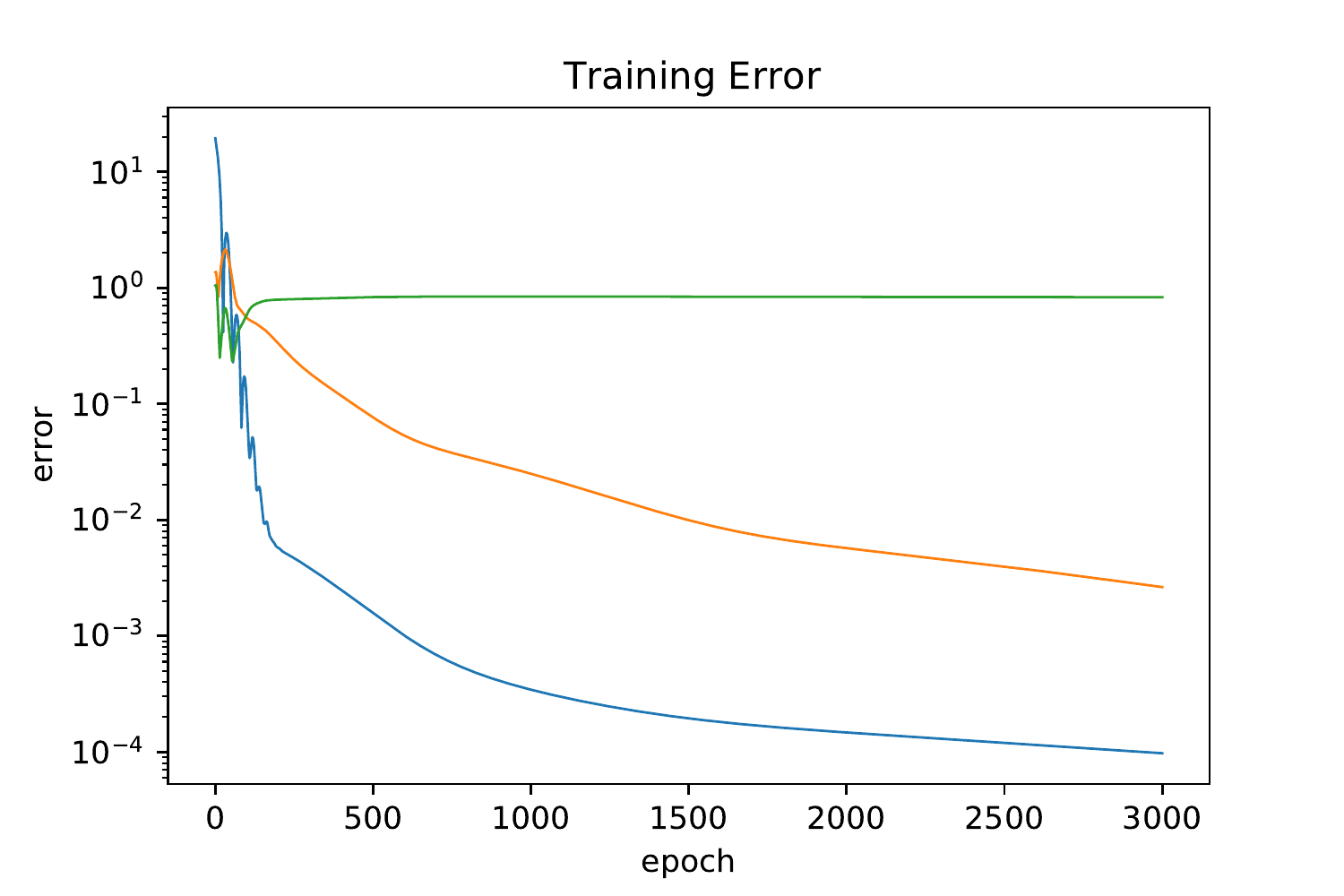}\hspace{-8pt}
 \includegraphics[scale=0.28]{1013/TrainingError.pdf}
 	\vspace{-3pt}
	\caption{$\sigma(p)=\text{sin}(p)$ (left), $\sigma(p)=\text{ Hat}(p)$ (right)}
	\label{sinvshat}
\end{minipage}
\end{figure}
Figure \ref{sinvshat} shows the convergent process of each frequency component of 
sin neural network and Hat neural network.  Green line, yellow line, and blue line denote the convergent process of $\Delta_{f,u} (1)$, $\Delta_{f,u} (3)$, and $\Delta_{f,u} (5)$ respectively. 
 Spectral bias is obviously 
observed for sin neural network (see left of Figure \ref{sinvshat}), 
while the spectral bias does not hold for Hat neural network as shown in right of Figure \ref{sinvshat}. 
\end{experiment}

\begin{experiment}
We run the Experiment 1 with fewer neurons, namely size 3000 (all the other settings are the same as Experiment 1), and obtained similar results as follows (See Figure \ref{tanhvshatfewer}): 
\begin{figure}[!htbp]
 \centering
  \includegraphics[scale=0.28]{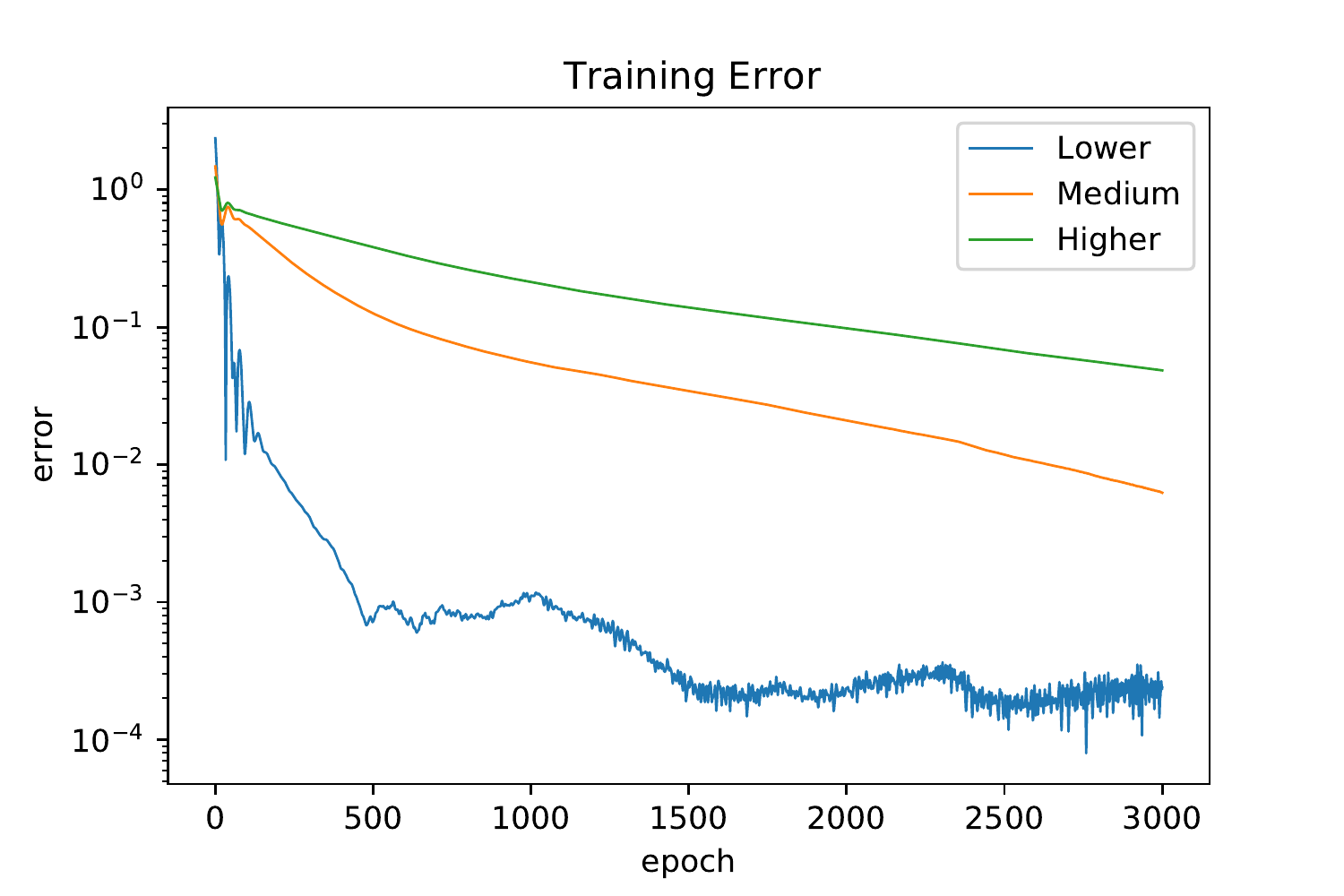}\hspace{-8pt}
  \includegraphics[scale=0.28]{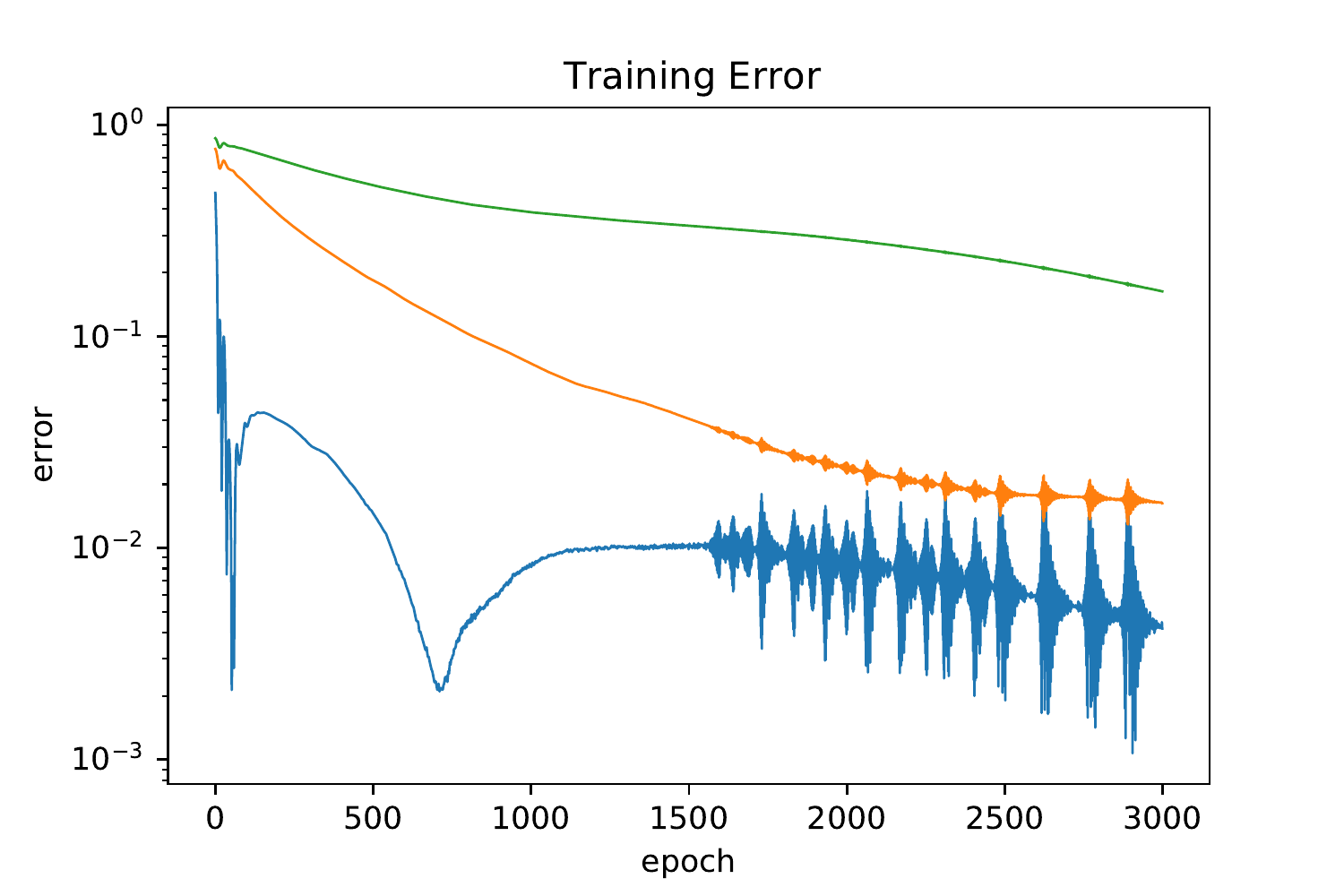}\hspace{-8pt}
  \includegraphics[scale=0.28]{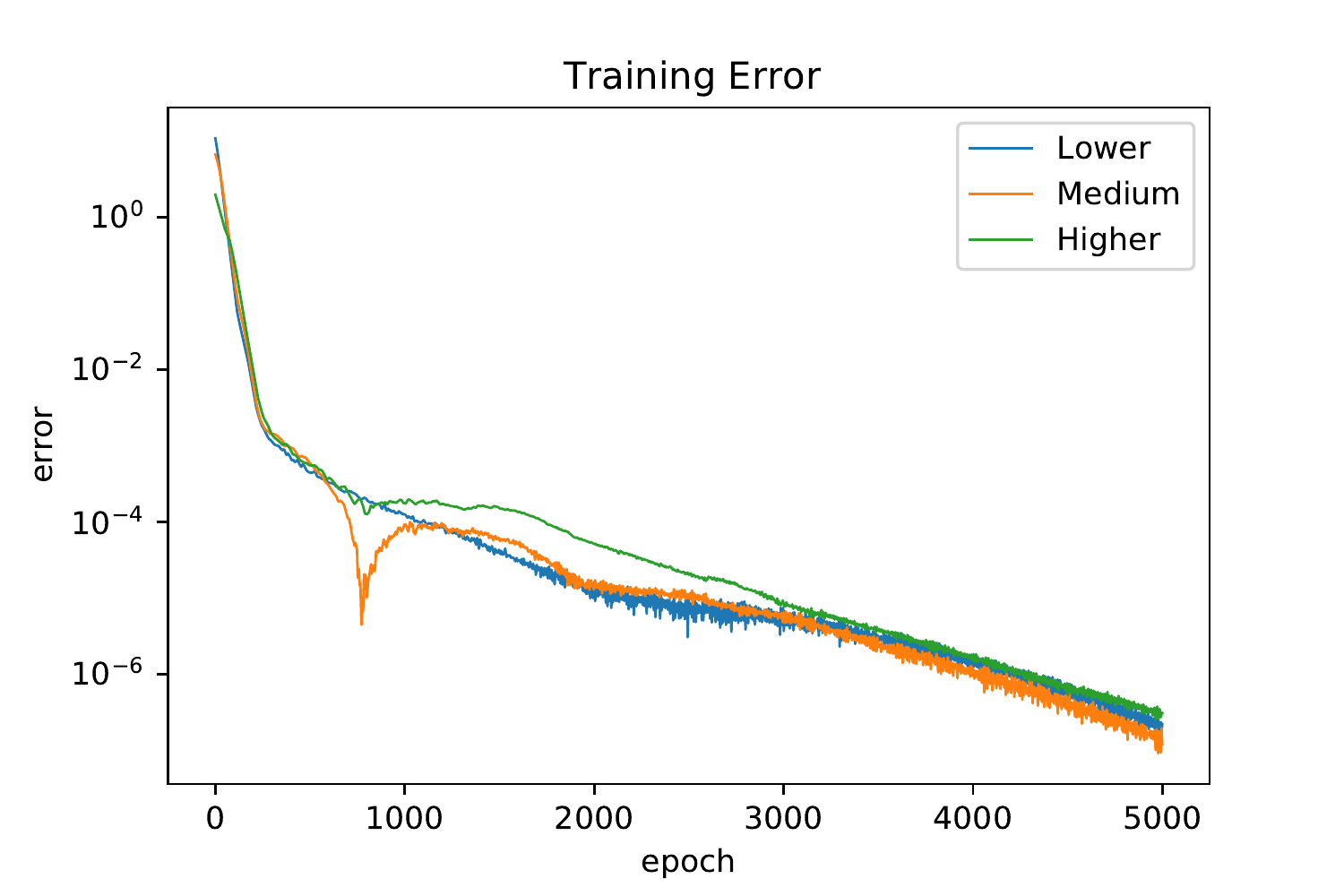}
 	\vspace{-3pt}
	\caption{$\sigma(p)=\text{tanh}(p)$ (left), $\sigma(p)=\text{ReLU}(p)$ (middle), $\sigma(p)=\text{ Hat}(p)$ (right)}
	\label{tanhvshatfewer}
\end{figure}
\\
We also run the Experiment 2 with fewer neurons, namely size 2-5000-1 (all the other settings are the same as Experiment 2), and obtained similar results as follows 
(See Figure \ref{reluvshat2dfewer}): 
\begin{figure}[!htbp]
 \centering
  \includegraphics[scale=0.28]{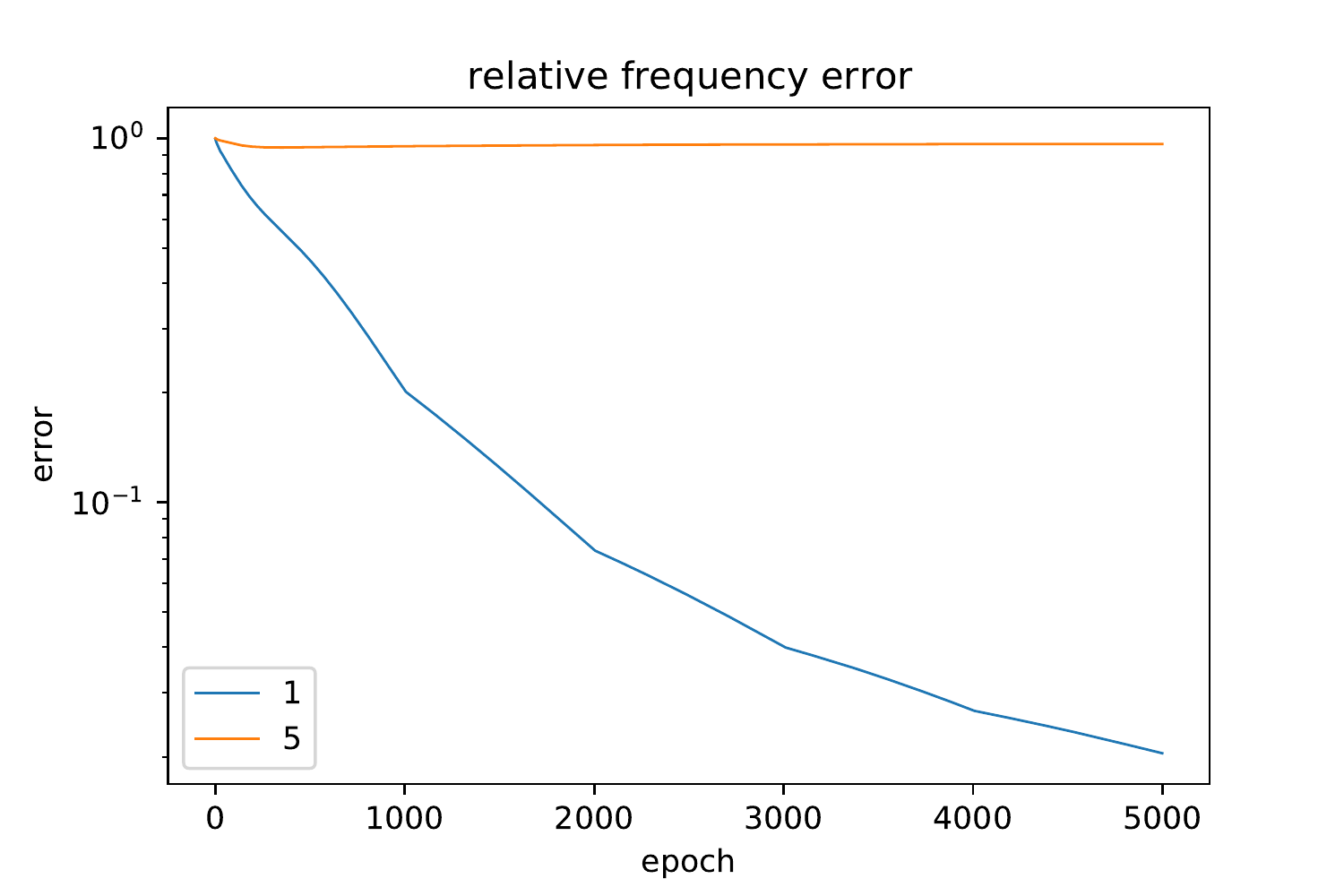}\hspace{-8pt}
  \includegraphics[scale=0.28]{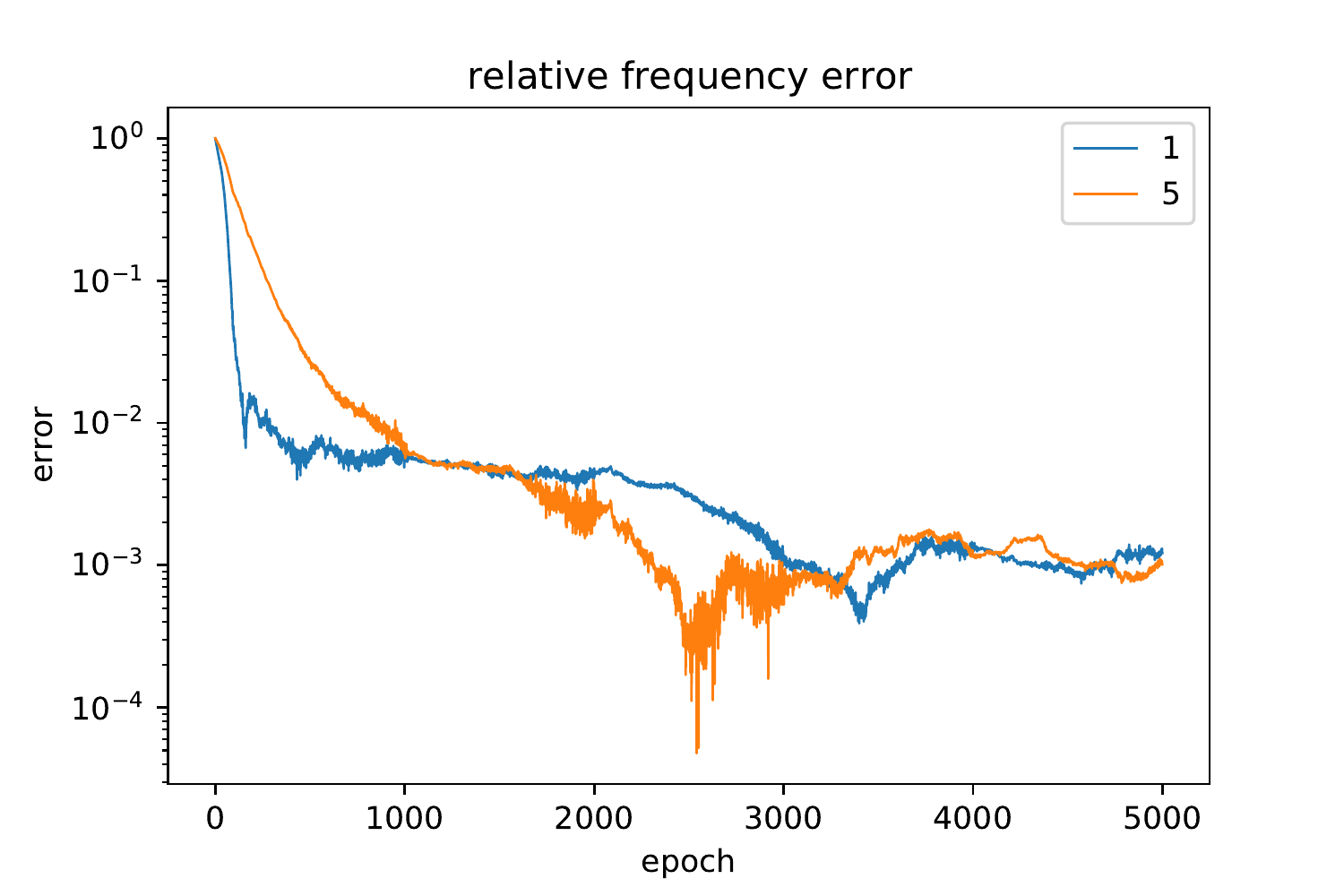}\hspace{-8pt}
 	\vspace{-3pt}
	\caption{$\sigma(p)=\text{relu}(p)$ (left), $\sigma(p)=\text{ Hat}(100p)$ (right)}
	\label{reluvshat2dfewer}
\end{figure}
\end{experiment}

\begin{experiment}
\normalfont
We use shallow neural networks with each of the activation functions in \eqref{activation-functions-experiments} to fit the following target function $u$ on $[-\pi, \pi]$:
\begin{equation}
u(x) = \sin(x) + \sin(3x)+\sin(5x).
\end{equation}
For each activation function our network has one hidden layer with size 8000.
The mean square error (MSE) loss is given by
\begin{equation}\label{MSEloss}
L(f,u) = \frac1N\sum_{i=1}^{N} (f(x_i)-u(x_i))^2,
\end{equation}
where $x_i$ is a uniform grid of size $N=201$ on $[-\pi,\pi]$. The three networks
are trained using SDG with a learning rate of 0.001. When using a tanh or ReLU activation function, all parameters are 
initialized following a Gaussian distribution with mean 0 and 
standard deviation 0.1, while the network with Hat activation function is 
initialized following a Gaussian distribution with mean 0 and 
standard deviation 0.8. 

Denote 
\begin{equation}\label{delta-definition-experiments}
\Delta_{f,u} (k) = \left| \hat{f}_k - \hat{u}_k 
\right|/\left|\hat{u}_k \right|,
\end{equation}
where $k$ represents the frequency, $\left| \cdot \right|$ represents the 
norm of a complex number and $\hat{\cdot}$ represents discrete Fourier 
transform. We plot $\Delta_{f,u} (1)$, $\Delta_{f,u} (3)$, 
and $\Delta_{f,u} (5)$ in Figure \ref{tanhvshatGD} for each of the three networks.
\begin{figure}[!htbp]
 \centering
 \includegraphics[scale=0.28]{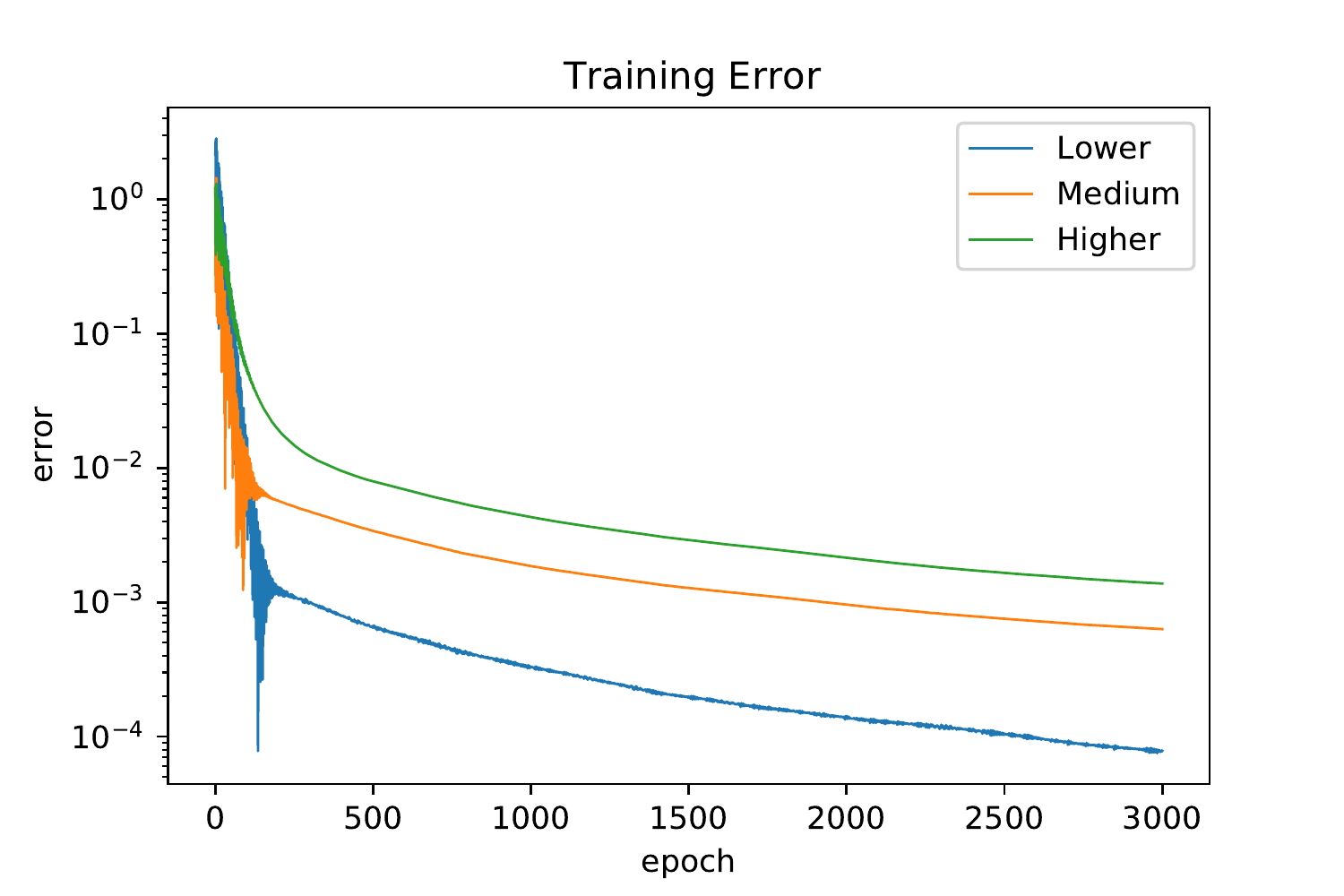}\hspace{-8pt}
  \includegraphics[scale=0.28]{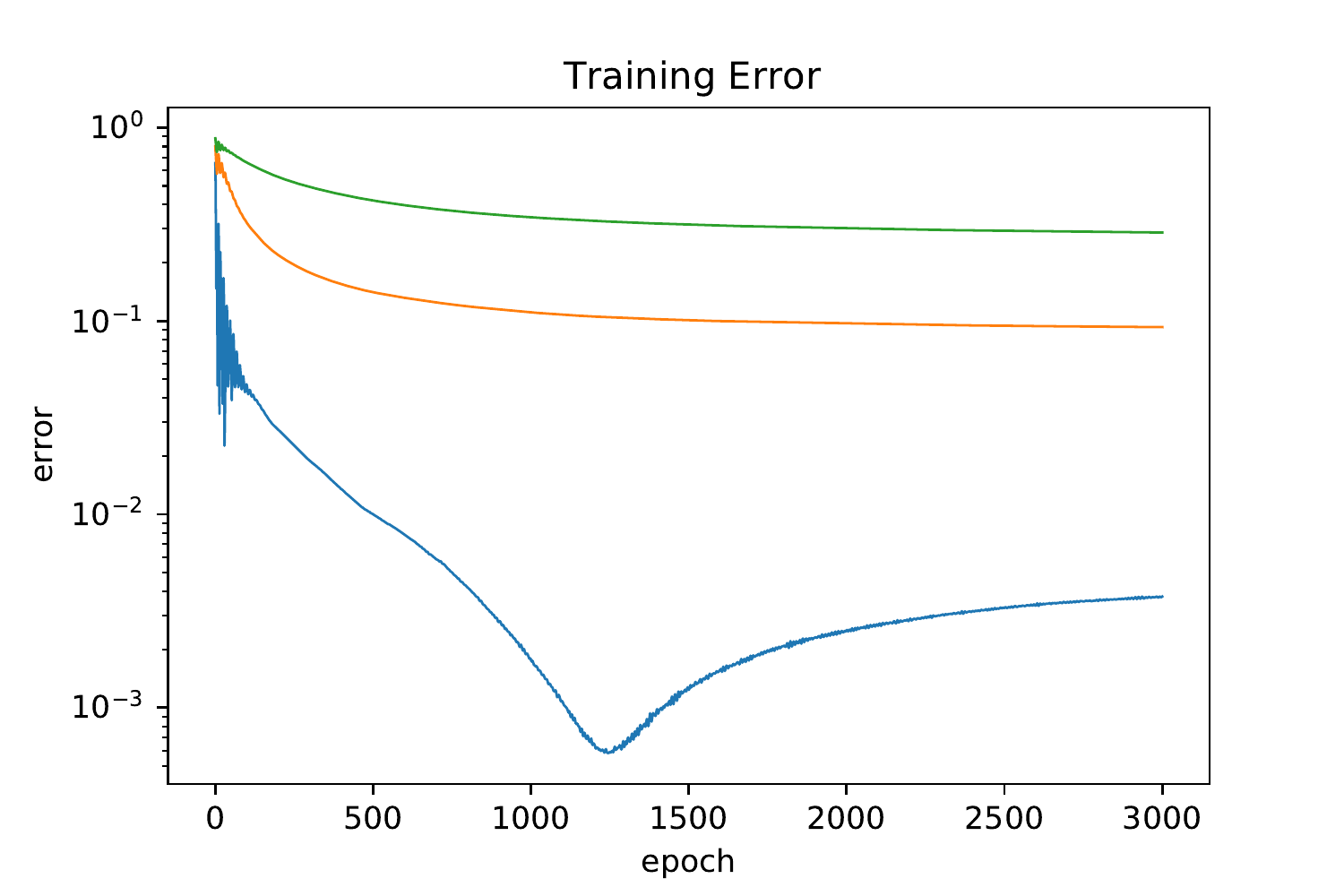}\hspace{-8pt}
 \includegraphics[scale=0.28]{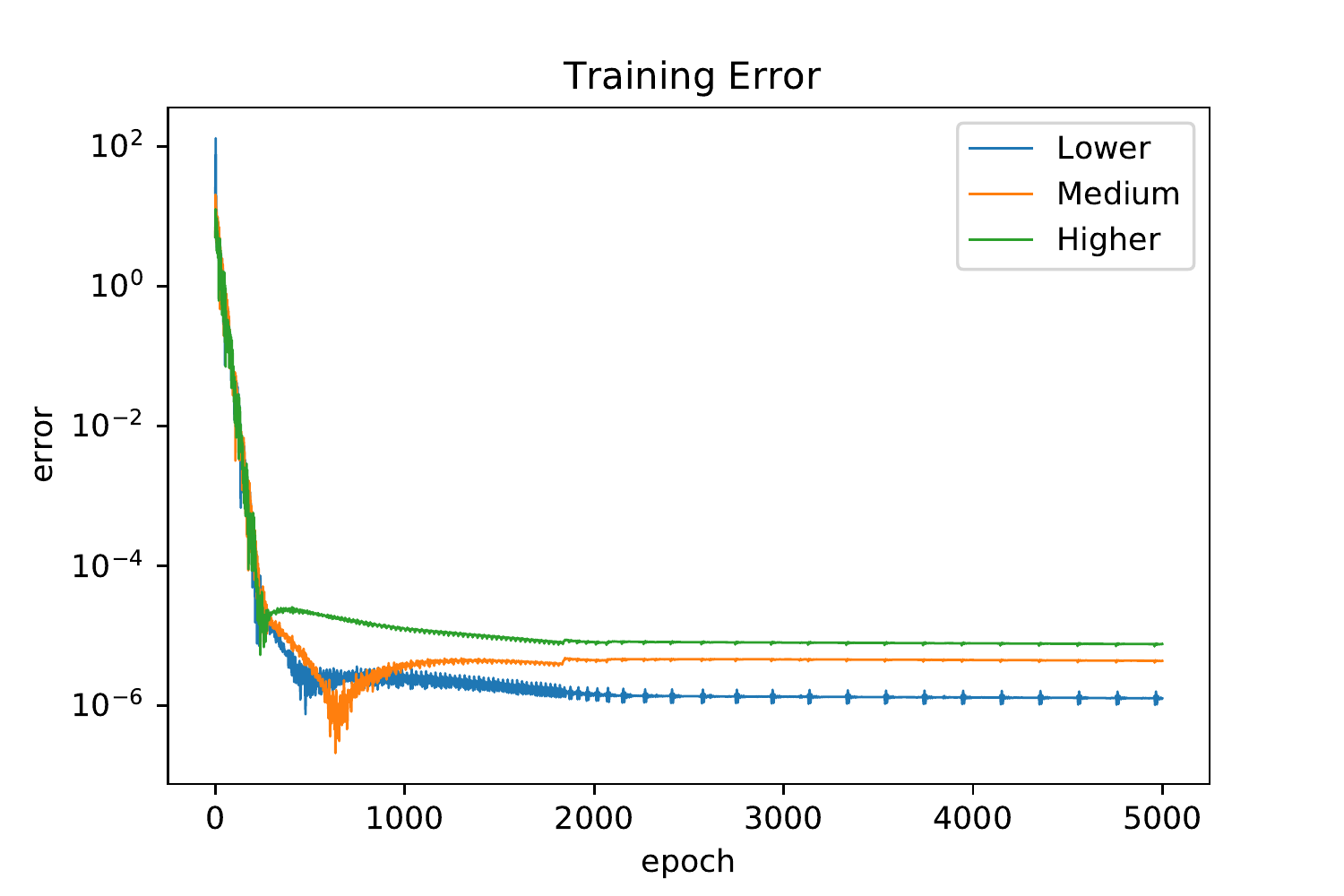}
 	\vspace{-3pt}
	\caption{$\sigma(p)=\text{tanh}(p)$ (left), $\sigma(p)=\text{ReLU}(p)$ (middle), $\sigma(p)=\text{ Hat}(p)$ (right)}
	\label{tanhvshatGD}
\end{figure}
From these results, we 
observe the spectral bias for tanh and ReLU neural networks (see left and middle of Figure \ref{tanhvshatGD}), 
while there is no spectral bias for the Hat neural network as shown in right of Figure \ref{tanhvshatGD}. These observations are similar to the results in Experiment 1 when ADAM is used to training the network. 
\end{experiment}

\begin{experiment}
We run the Experiment 1 using shallow ReLU neuron network with different sizes $n=500, n=2000$ and $n=8000$ (all the other settings are the same as Experiment 1), and obtained results as follows (See Figure \ref{ReLUwidth}): 
\begin{figure}[!htbp]
 \centering
  \includegraphics[scale=0.28]{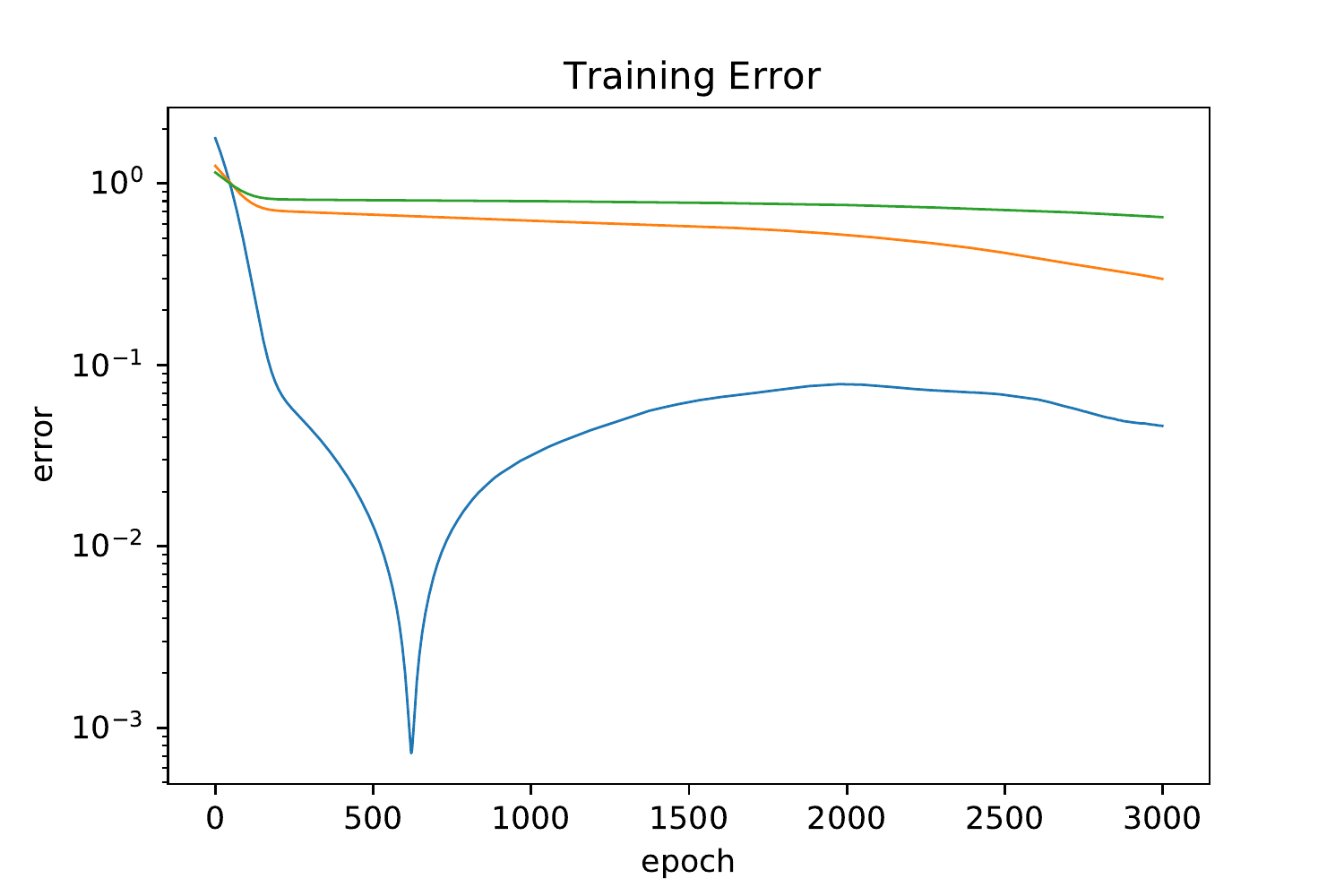}\hspace{-8pt}
  \includegraphics[scale=0.28]{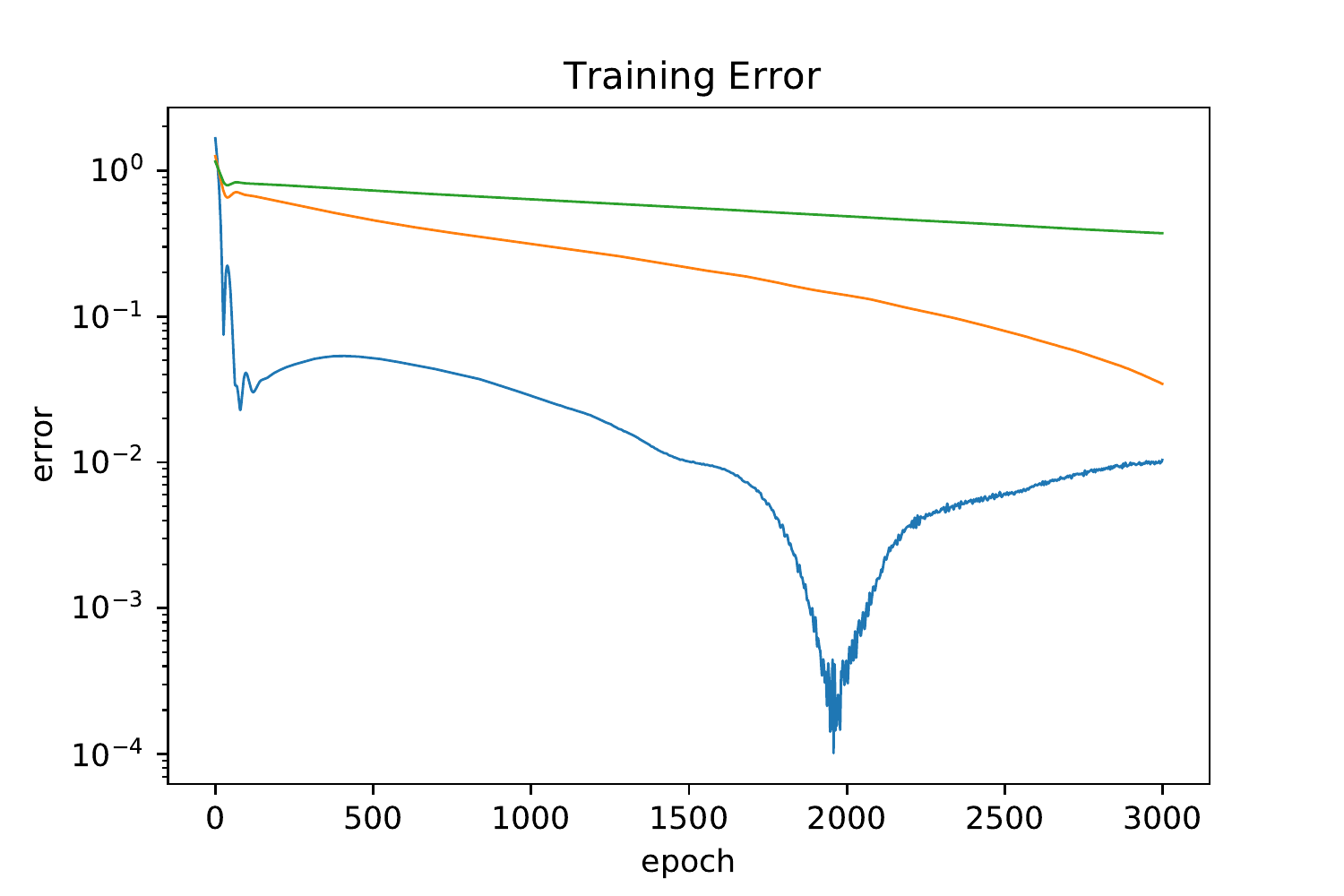}\hspace{-8pt}
 \includegraphics[scale=0.28]{220731/Exp1_TrainingError_relu.pdf}
 	\vspace{-3pt}
	\caption{$n=500$ (left), $n=2000$ (middle), $n=8000$ (right)}
	\label{ReLUwidth}
\end{figure}
From Figure \ref{ReLUwidth}, we can see that when the size of shallow ReLU neuron network increases, the spectral bias becomes stronger. 
\end{experiment}
\subsection{Finite element bases}
Let ${\cal T}_h$ be a uniform mesh on $[0,1]$ with $n$ grid points and mesh size $h=\frac{1}{n}$. Define
\begin{equation*}
\begin{aligned}
V_n=\{v_h:\; v_h~\mbox{is continuous and}~
\mbox{piecewise linear w.r.t.} {\cal T}_h, v_h(0)=0\}.
\end{aligned}
\end{equation*}
The space $V_n$ is a standard linear finite element space in one dimension and has been well-studied, see \cite{ciarlet2002finite}, for instance.

We denote two bases of $V_n$ (see Figure \ref{twobases}), as follows: 
\begin{itemize}
\item ReLU basis: $\psi_i(x)={\rm ReLU}(\frac{x-x_{i-1}}{h})
$
\begin{equation}
	\label{Relu-basis}
	\psi_i(x)=\left\{\begin{array}{cl}
	\frac{x-x_{i-1}}{h}, & x>x_{i-1};\\
	0, &x\le x_{i-1}.
\end{array}\right.
\end{equation}
where ${\rm ReLU(x)}=\max\{0,x\}$.	
\item Hat basis: $\varphi_i(x) ={\rm Hat}(nx-i+1)
$
\begin{equation}
	  \label{Vh-basis}
	\varphi_i(x)=\left\{\begin{array}{cl}
		\frac{x-x_{i-1}}{h}, & x\in[x_{i-1},x_i];\\
		\frac{x_{i+1}-x}{h}, & x\in[x_{i},x_{i+1}];\\
		0, &\mbox{elsewhere}.
	\end{array}\right.
\end{equation}
where ${\rm Hat}(x)=\left\{\begin{array}{cl}
		x, & x\in[0,1];\\
		2-x, & x\in[1,2];\\
		0, &\mbox{elsewhere}.
	\end{array}\right.$
\end{itemize}
\begin{figure}
		\centering
		\includegraphics[scale=0.45]{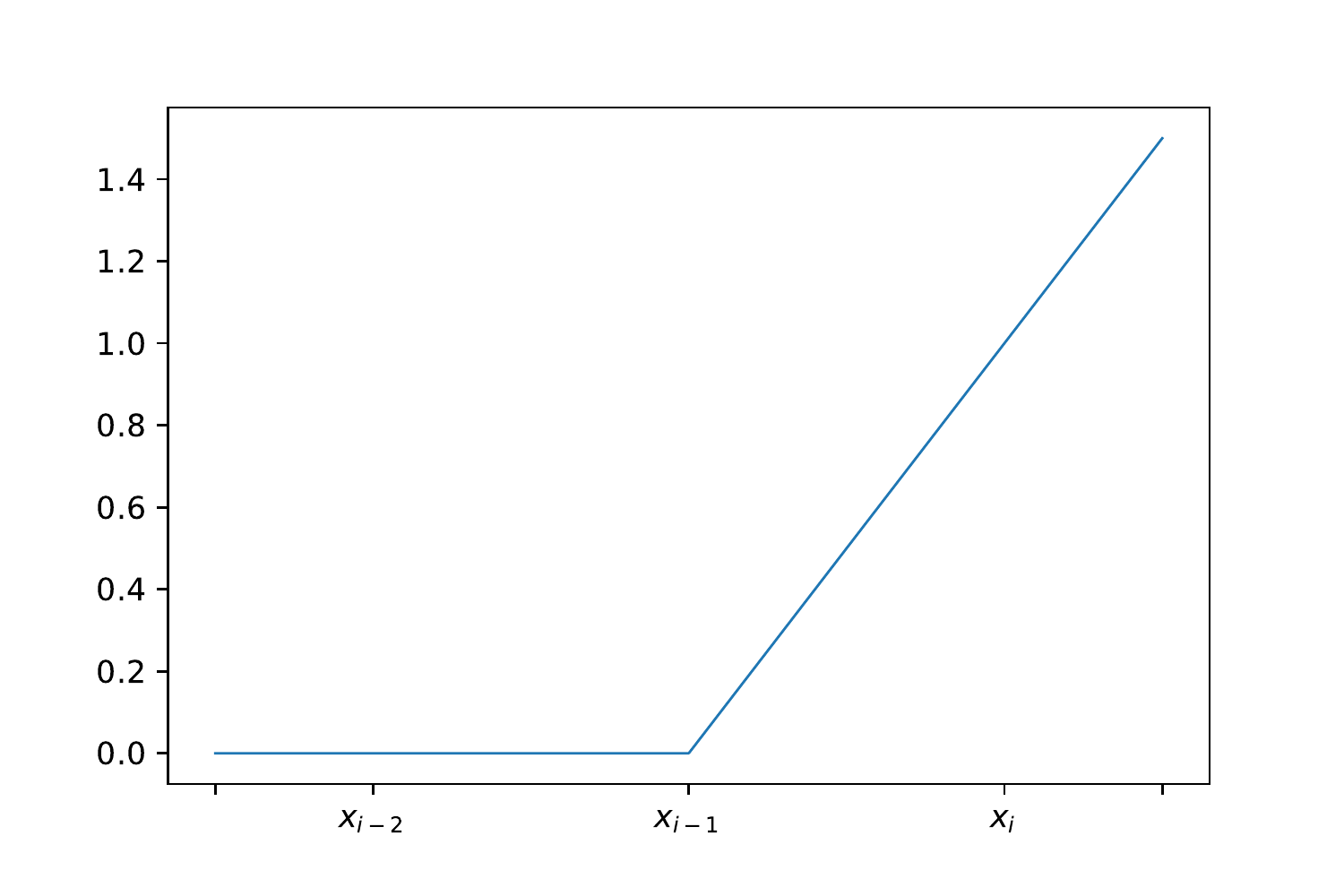} 
		\includegraphics[scale=0.45]{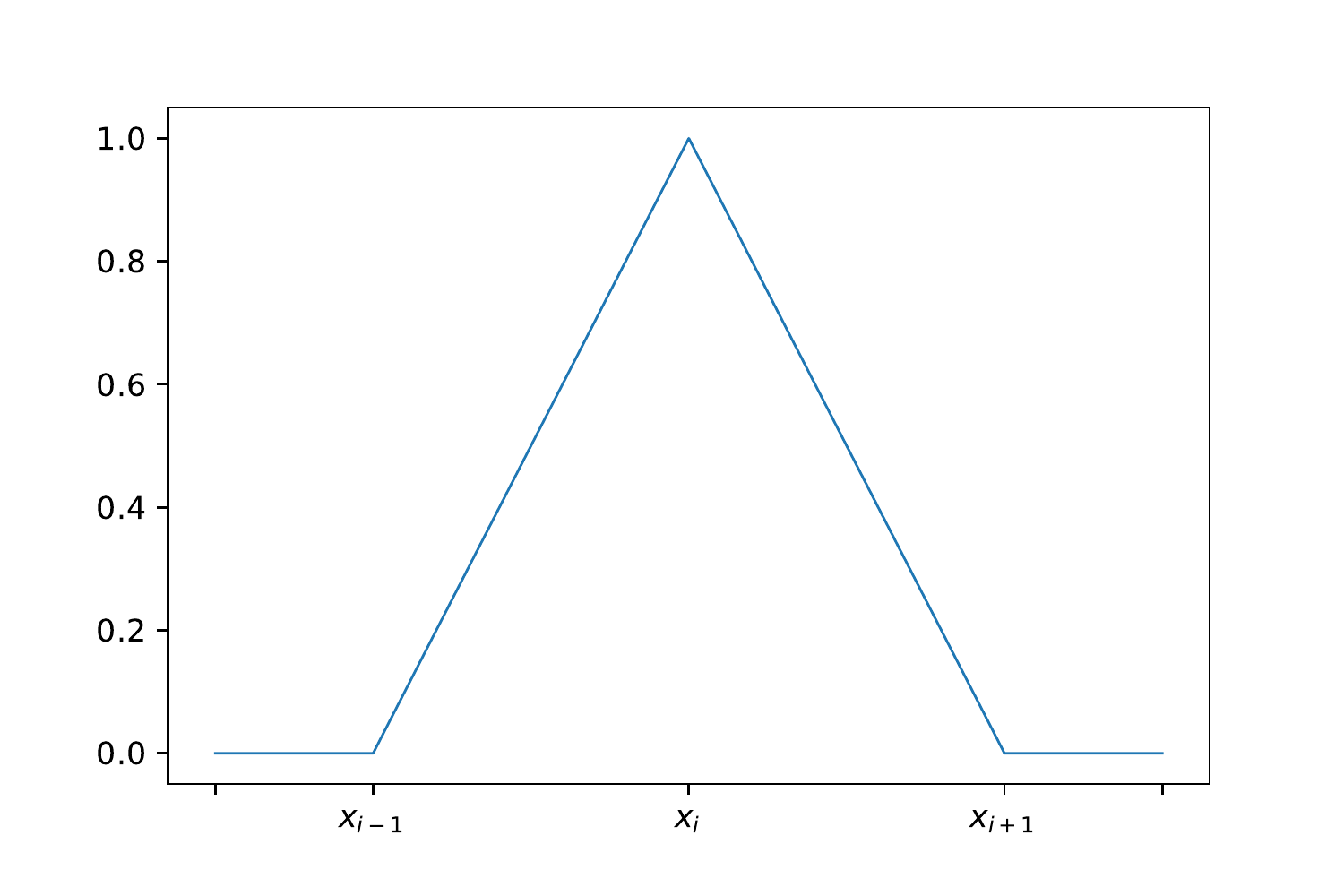}
		\caption{ReLU basis function $\psi_i(x)$ and Hat basis function $\varphi_i(x)$}
		\label{twobases}
	\end{figure}
Note that the Hat basis is the standard basis typically used in the theory of finite element methods, while the ReLU basis is based upon the common rectified linear activation function used in deep learning \cite{nair2010rectified}.
Let $\Psi (x)= \left( \psi_1(x), \psi_2(x), \cdots, \psi_{n}(x)\right)^T$ and $\Phi(x) = \left( \varphi_1(x), \varphi_2(x), \cdots, \varphi_n(x) \right) ^T$.
It is easy to verify that 
\begin{equation}\label{CBasis}
 	\Phi = C\Psi,
\end{equation}
and the change of basis matrix which converts between these bases is given by
\begin{equation}\label{matrixC}
	C = \begin{pmatrix}
		1 & -2 & 1   &    && \\
		&  1 & -2 & 1   && \\
		&&\ddots&\ddots&\ddots& \\
		&&    & 1 & -2  & 1& \\ 
		&&    &    & 1 &-2 \\
		&&    &    &  &1
	\end{pmatrix}.
\end{equation}
and
\begin{equation}
C^{-1} = 
\begin{pmatrix}
1 &2&3&4&\cdots&n \\
0&1&2&3&\cdots&n-1 \\
0&0&1&2&\cdots&n-2 \\
0&0&0&1&\cdots&n-3\\ 
\vdots&\vdots&\vdots&\vdots&\ddots&\vdots \\
0&0&0&0&\cdots&1
\end{pmatrix}.
\end{equation}

\begin{itemize}
\item When the basis in linear finite element space is chosen as Hat basis, which motivates the activation function is chosen as $\sigma(x)={\rm Hat}(x)$, we denote the mass matrix $M_\sigma=M_\Phi$. 
Then it is easy to see that 
\begin{equation}\label{MphiM}
M_\Phi = \frac{h}{6}M,
\end{equation}
 where
\begin{equation}\label{Mphi}
 M=\begin{pmatrix}
4 & 1 &  &  &     \\
1&  4 & 1 &  &   \\
&\ddots&\ddots&\ddots \\
&    & 1 & 4  & 1 \\ 
&    &    & 1 &2 \\
\end{pmatrix}
\in \mathbb{R}^{n\times n}.
\end{equation} 
\item When the basis in linear finite element space is chosen as ReLU basis, which corresponds to the activation function is chosen as $\sigma(x)=\mbox{ReLU}(x)=\max\{0,x\}$, we denote the mass matrix $M_\sigma=M_\Psi$.
\item Due to \eqref{CBasis}, the relationship between $M_\Phi$ and $M_\Psi$ is given by 
\begin{equation}\label{MM}
M_{\Psi}=C^{-1}M_\Phi C^{-T}.
\end{equation}
\end{itemize}

\subsection{Spectral analysis of the matrices $M_\Phi$ and $M_{\Psi}$}\label{spectral}
First we have estimates of the eigenvalues of matrix $M$:
\begin{theorem}\label{MPhieigen}
All the eigenvalues of $M$ are in $[\frac{1}{6},1]$. 
\end{theorem}
\begin{proof}
Let $\lambda_{k,M}$ be an eigenvalue of $M$, by the Gershgorin circle theorem, we know that 
$$
\lambda_{k,M}\in \{|\lambda-\frac{2}{3}|\le \frac{1}{6}\}\cup\{|\lambda-\frac{2}{3}|\le \frac{1}{3}\}\cup \{|\lambda-\frac{1}{3}|\le \frac{1}{6}\}
$$
namely 
$\lambda_{k,M}\in[\frac{1}{6}, 1]$ for $k=1,2,\cdots, n$. 
\end{proof}

From Theorem \ref{MPhieigen}, the estimate of the eigenvalues of $M_\Phi$ can be obtained immediately by \eqref{MphiM} and hence Theorem \ref{hat-matrix-theorem} is obtained. 

Next, we estimate the eigenvalues of $M_{\Psi}$. We introduce the following matrix $A$ which is related to $C$. 
\begin{equation}\label{matrixA}
 A=
\begin{pmatrix}
	2 & -1 &    &    & \\
	-1&  2 & -1 &    & \\
	  &\ddots&\ddots&\ddots& \\
	  &    & -1 & 2  & -1 \\ 
	  &    &    & -1 &1
\end{pmatrix}
\in \mathbb{R}^{n\times n}.
\end{equation}
The eigenvalues and corresponding eigenvectors of the matrix $A$, using the result in  \cite{yueh2005eigenvalues} can be obtained:
\begin{lemma} \cite{yueh2005eigenvalues}\label{eigenAphi}
The eigenvalues $\lambda_{k,A}, 1\le k\le n$ and corresponding eigenvectors $\xi_{A}^k=(\xi_{A,j}^k)_{j=1}^n, 1\le k\le n $ of $A$ are 
\begin{equation}
	\lambda_{k,A} = 4\cos^2\frac{k\pi}{2n+1},\quad \xi_{A,j}^k =- \sin\left((n+\frac12-k)t_j\pi\right)\quad 
\end{equation}
with $\quad t_j=\frac{2j}{2n+1}, 1\le j\le n$. 
\end{lemma}
\begin{lemma}\label{CandA2}
Note that $C$ is defined by \eqref{matrixC} and $A$ is defined by \eqref{matrixA}, we have
\begin{equation}\label{CandA}
	CC^T= A^2 + B,
\end{equation}
where
\begin{equation}
\begin{aligned}
	B &=
	\begin{pmatrix}
		1&0&\cdots&0&0&0 \\
		0&0&\cdots&0&0 &0\\
		\vdots & \vdots  & \ddots& \vdots& \vdots& \vdots\\
		0&0&\cdots&0&0 &0\\
		0&0&\cdots&0&-1&1 \\
		0&0&\cdots&0&1&-1
	\end{pmatrix}
	\\
	&=a_0a_0^T-a_1a_1^T
	\in \mathbb {R}^{n\times n}.
	\end{aligned}
\end{equation}
and $a_0=\begin{pmatrix}1\\
0\\
 \vdots\\
  0\\
  0\\
  0
  \end{pmatrix}
  \in \mathbb{R}^n$,
 $a_1=\begin{pmatrix}0\\
 0\\
 \vdots\\ 0\\-1\\1  \end{pmatrix} \in\mathbb{R}^n$.
\end{lemma}

\begin{proof}
By direct computation, we find that there is a relationship between $A$ and $C$ as follows
\begin{equation}
	A = -C
	\begin{pmatrix}
		0&1 \\
		I_{n-1}&0
	\end{pmatrix}
	+B_1
\end{equation}
where $I_{n-1}\in \mathbb{R}^{(n-1)\times (n-1)}$ is the identity matrix and
\begin{equation*}
	B_1 = 
	\begin{pmatrix}
		0&0&\cdots&0&0&1 \\
		0&0&\cdots&0&0&0 \\
		\vdots & \vdots & \ddots  & \vdots& \vdots& \vdots\\
		0&0&\cdots&0&0&0 \\
		0&0&\cdots&0&0&-1 \\
		0&0&\cdots&0&0&1
	\end{pmatrix}
	=a_3a_4^T
	\in \mathbb {R}^{n\times n},
\end{equation*}
where $a_3= \begin{pmatrix}1\\0\\ \vdots\\ 0\\ -1\\ 1\end{pmatrix}$ and $a_4=\begin{pmatrix}0\\ 0\\ \vdots \\
0\\ 0\\1\end{pmatrix}$. 

Then 
$$
C=(-A+B_1)
\begin{pmatrix}
		0&1 \\
		I_{n-1}&0
	\end{pmatrix}^{-1}
$$ 
	and 
$$C^T=\begin{pmatrix}
		0&1 \\
		I_{n-1}&0
	\end{pmatrix}^{-T}(-A+B_1^T).
$$
Noting that  $\begin{pmatrix}
		0&1 \\
		I_{n-1}&0
	\end{pmatrix}^{-1}=\begin{pmatrix}
		0&1 \\
		I_{n-1}&0
	\end{pmatrix}^{T}$, by direct computation of $CC^T$, \eqref{CandA} is desired. 

\end{proof}

\begin{lemma}\cite{fulton2000eigenvalues}\label{Weyl}
Let $R\in R^{n\times n}, N\in R^{n\times n}$ and $L\in R^{n\times n}$ are symmetric matrices satisfying 
\begin{equation}
R=N+L
\end{equation}
and $\nu_1\ge \nu_2\ge \cdots\ge \nu_n $ are the eigenvalues of $R$, $\gamma_1\ge \gamma_2\ge \cdots\ge \gamma_n$ are the eigenvalues of $N$ and $\beta_1\ge \beta_2\cdots\ge \beta_n$ are the eigenvalues of $L$, then we have 
\begin{equation}\label{Wyle0}
\gamma_j+\beta_k\le \nu_i\le \gamma_p+\beta_q
\end{equation}
where $j+k-n\ge i\ge p+q-1$. 
\end{lemma}

In order to estimate the eigenvalues of $M_\Psi$, we need to estimate the eigenvalues of $CC^T$: 
\begin{lemma}\label{eiganCC}
Let $\nu_1\ge \nu_2\ge \cdots\ge \nu_n $ be the eigenvalues of $CC^T$ and $\lambda^2_{j,A}, j=1,\cdots, n$ be the eigenvalues of $A^2$, then we have 
\begin{equation}\label{estimate}
\begin{aligned}
&\lambda^2_{2,A}\le \nu_1\le \lambda^2_{1,A}+1;\\
&\lambda^2_{i+1,A}\le \nu_i\le \lambda^2_{i-1,A}, \quad 2\le i\le n-1;\\
&0< \nu_n\le \lambda^2_{n-1, A}.
\end{aligned}
\end{equation}
In addition, we have 
\begin{equation}\label{estimate1}
 \nu_1\le 16, 
 \end{equation}
 and 
 \begin{equation}\label{estimaten}
 \frac{4}{n^2(n+1)^2}\le\nu_n.
 \end{equation}
 \end{lemma}
 \begin{proof}
%
%

By Weyl's inequality for pertubation matrix, namely Lemma \ref{Weyl}, we have 
\begin{equation}\label{Wyle}
\lambda^2_{j,A}+\beta_k\le \nu_i\le \lambda^2_{p,A}+\beta_q
\end{equation}
where $j+k-n\ge i\ge p+q-1$ and $\lambda^2_{j,A}, j=1,\cdots, n$ are the eigenvalues of $A^2$. 
Further noting that $CC^T$ is positive definite and by direct computing we can obtain $\beta_1=1,  \beta_2=\cdots= \beta_{n-1}=0, \beta_n=-2$  and hence 
\eqref{estimate} is proved. Next we only need to prove \eqref{estimate1} and \eqref{estimaten}. 
Since $CC^T$ is symmetric positive definite, we have 
\begin{align*}
\nu_1&=\lambda_{\max}(CC^T)=\rho(CC^T)\le \|CC^T\|_{\infty}\\
&\le \|C\|_{\infty}\|C^T\|_{\infty}=4\times 4=16,
\end{align*}
where $\rho(CC^T)$ is the spectral radius of $CC^T$. 

Since 
\begin{align*}
\lambda_{\max}(C^{-T}C^{-1})&=\rho(C^{-T}C^{-1})\le \|C^{-T}C^{-1}\|_{\infty}\\
&\le \|C^{-T}\|_{\infty}\|C^{-1}\|_{\infty}\\
&=\frac{n(n+1)}{2}\frac{n(n+1)}{2}=\frac{n^2(n+1)^2}{4},
\end{align*}
where $\rho(C^{-T}C^{-1})$ is the spectral radius of $C^{-T}C^{-1}$,
then
\begin{align*}
\nu_n&=\lambda_{\min}(CC^T)=\frac{1}{\lambda_{\max}(C^{-T}C^{-1})}\ge \frac{4}{n^2(n+1)^2}.
\end{align*}
\end{proof}

\begin{lemma}\label{minmax}
(Courant-Fisher min-max theorem) For any given matrix $E \in \mathbb{R}^{n \times n}, E=E^{T}$, suppose $\lambda_{1} \ge \lambda_{2} \ge \cdots \ge \lambda_{n}$ are the eigenvalues of  $E$, then
$$
\lambda_{n+1-k}=\min _{\{S \mid \operatorname{dim} S=k\}} \max _{x \in S} \frac{(E x, x)}{(x, x)}.
$$
\end{lemma}
Now we give the eigenvalues of $M_{\Psi}$ as follows:
\begin{theorem}\label{theigenMr}
The eigenvalues of $M_{\Psi}$ satisfy 
\begin{equation}\label{eigenMr}
\lambda_{k,M_{\Psi}}=m_k h\nu^{-1}_{n+1-k}, \quad k=1,2,\cdots, n
\end{equation}
where $\frac{1}{6}\le m_{k}\le 1$ is a constant, $\nu_{n+1-k}$, $k=1,2, \cdots, n$ are the eigenvalues of $C C^{T}$.
\end{theorem}
\begin{proof} 
Noting \eqref{MM}, namely
\begin{equation}
	M_{\Psi}=C^{-1}M_\Phi C^{-T}.
\end{equation}

For any given $S$ with dim$S=k$, we have 
\begin{equation*}
\begin{aligned}
\max_{x\in S} \frac{(M_{\Psi}x, x)}{(x,x)}&=\max_{x\in S} \frac{(C^{-1}M_\Phi C^{-T}x, x)}{(x,x)}\\
&=\max_{x\in S} \frac{(M_\Phi C^{-T}x, C^{-T}x)}{(x,x)}
\end{aligned}
\end{equation*}
For the above given $S$, let $x_s$ satisfy 
\begin{equation*}
\begin{aligned}
\max_{x\in S} \frac{(M_\Phi C^{-T}x, C^{-T}x)}{(x, x)}=\frac{(M_\Phi C^{-T}x_s, C^{-T}x_s)}{(x_s, x_s)},
\end{aligned}
\end{equation*}
then we have 
\begin{equation*}
\begin{aligned}
\max_{x\in S} \frac{(M_{\Psi}x, x)}{(x,x)}&=\frac{(M_\Phi C^{-T}x_s, C^{-T}x_s)}{(x_s,x_s)}\\
&\le \lambda_{\max}(M_\Phi)\frac{(C^{-T}x_s, C^{-T}x_s)}{(x_s, x_s)}\\
&\le  \lambda_{\max}(M_\Phi)\max_{x\in S}\frac{(C^{-T}x, C^{-T}x)}{(x, x)}\\
&=\lambda_{\max}(M_\Phi)\max_{x\in S}\frac{(C^{-1}C^{-T}x,x)}{(x, x)}.
\end{aligned}
\end{equation*}
Now let $S_k$ satisfy 
$$
\min_{\{S|{\rm dim} S=k\}}\max_{x\in S}\frac{(C^{-1}C^{-T}x,x)}{(x, x)}=\max_{x\in S_k}\frac{(C^{-1}C^{-T}x,x)}{(x, x)},
$$
then we have 
\begin{equation*}
\begin{aligned}
&\min_{\{S|{\rm dim} S=k\}}\max_{x\in S}\frac{(M_{\Psi}x,x)}{(x, x)}\\
&\le \max_{x\in S_k} \frac{(M_{\Psi}x, x)}{(x,x)}\\
&\le \lambda_{\max}(M_\Phi)\max_{x\in S_k}\frac{(C^{-1}C^{-T}x,x)}{(x, x)}\\
&=\lambda_{\max}(M_\Phi) \min_{\{S|{\rm dim} S=k\}}\max_{x\in S}\frac{(C^{-1}C^{-T}x,x)}{(x, x)}.
\end{aligned}
\end{equation*}
By Lemma \ref{minmax} and noting that $CC^T=C(C^TC)C^{-1}$, we have 
\begin{equation}\label{upper}
\begin{aligned}
\lambda_{n+1-k,r}\le& \lambda_{\max}(M_\Phi) \lambda_{n+1-k}(C^{-1}C^{-T})\\
=&\lambda_{\max}(M_\Phi) \frac{1}{\lambda_{k}(C^{T}C)}\\
=&\lambda_{\max}(M_\Phi) \frac{1}{\lambda_{k}(CC^T)}
\end{aligned}
\end{equation}
Similarly, we have 
\begin{equation}\label{lower}
\lambda_{\min}(M_\Phi) \frac{1}{\lambda_{k}(CC^T)}\le \lambda_{n+1-k,r}
\end{equation}
Combining \eqref{upper} and \eqref{lower}, we have 
\begin{equation*}
\lambda_{\min}(M_\Phi) \frac{1}{\lambda_{k}(CC^T)}\le \lambda_{n+1-k,r}\le \lambda_{\max}(M_\Phi) \frac{1}{\lambda_{k}(CC^T)}.
\end{equation*}
Noting that the eigenvalues of  $M_{\Phi}$ is in $[\frac{h}{6},h]$, then there exists a constant $m_{n+1-k}\in [\frac 16, 1]$ such that 
\begin{equation}
 \lambda_{n+1-k,r}=m_{n+1-k} h \frac{1}{\lambda_{k}(CC^T)}.
\end{equation}
Finally, by Lemma \ref{Weyl}, we have
\begin{equation}
 \lambda_{k,M_{\Psi}}=m_k h \frac{1}{\lambda_{n+1-k}(CC^T)}=m_k h\nu^{-1}_{n+1-k}.
\end{equation}
\end{proof}

%
%
%

%

\subsection{Proof  of Theorem \ref{relu-matrix-theorem}}\label{Proofrelu-matrix-theorem}
\begin{proof}

From the equation \eqref{eigenMr}  shown in Theorem \ref{theigenMr}, Theorem \ref{relu-matrix-theorem} can be obtained easily. 
\end{proof}


\subsection{Eigenvectors of $M_\Psi$}\label{eigenMPsi}
\begin{theorem}\label{xi1xn}
Let $\lambda_{1,\Psi} >\lambda_{2,\Psi} >\cdots >\lambda_{n,\Psi}$ and $\xi_{\Psi}^1, \xi_{\Psi}^2, \cdots, \xi_{\Psi}^n$ be the eigenvalues  and  corresponding eigenvectors of $M_{\Psi}$.  
Extend $v^k_h(x)= \xi_{\Psi}^k\cdot \Psi(x)=C^{-T} \xi_{\Psi}^k\cdot \Phi(x)\in V_n$ by zero to $[x_{-1}, x_n]$, then $v^k_h(x)$ satisfies 
\begin{align}
 \frac{\int_{0}^1|\mathcal{D}_h v^k_h(x)|^2dx} {\int_{0}^1| v^k_h(x)|^2dx} =m^{-1}_kn^{4}\nu_{n+1-k},
\end{align}
where $m_k\in [1/6,1]$ is a constant independent of $h$ and $\mathcal{D}_h$ is the second order finite difference operator.  In particular, 
\begin{enumerate}
\item for $k=1$, we have
\begin{align}
c_0'\ge \frac{\int_{0}^1|\mathcal{D}_h v^1_h(x)|^2dx} {\int_{0}^1| v^1_h(x)|^2dx} \ge c_0,
\end{align}
where $c_0$ and $c_0'$ are constants independent of $h$. 
\item  for $k=n$, we have
\begin{align}
c_1'n^{4} \ge \frac{\int_{0}^1|\mathcal{D}_h v^n_h(x)|^2dx} {\int_{0}^1| v^n_h(x)|^2dx} \ge c_1n^{4},
\end{align}
where $c_1$ and $c_1'$ are constants independent of $h$. 
\end{enumerate}
\end{theorem}

\begin{proof}

Noting that $\xi_{\Psi}^k$ and $\lambda_{k,M_{\Psi}}$ satisfy 
$$M_{\Psi} \xi_{\Psi}^k= \lambda_{k,M_{\Psi}} \xi_{\Psi}^k,
$$ 
and  
$$
M_{\Psi}=C^{-1} M_\Phi C^{-T},
$$
we have
\begin{equation}\label{def:eigen}
\begin{aligned}
	C^{-1}M_{\Phi}C^{-T}\xi_{\Psi}^k &= \lambda_{k,M_{\Psi}} \xi_{\Psi}^k, \\
	C^{-T}C^{-1}M_{\Phi}C^{-T}\xi_{\Psi}^k &= \lambda_{k,M_{\Psi}}C^{-T}\xi_{\Psi}^k. 
\end{aligned}
\end{equation}
Define $v= C^{-T}\xi_{\Psi}^k$, then we have
\begin{align}
	(CC^T)^{-1} M_{\Phi} v &= \lambda_{k,M_{\Psi}} v,\\
	CC^T v &= \lambda_{k,M_{\Psi}}^{-1} M_{\Phi} v.
\end{align}
Denoting
$v=(v_1,v_2,\cdots,v_n)^T$, we have 
\begin{align}
C^Tv=\begin{pmatrix}
v_1\\
 -2v_1+v_2\\
  v_1-2v_2+v_3\\
   \vdots\\
    v_{n-2}-2v_{n-1}+v_n
    \end{pmatrix}
\end{align}
and 
\begin{align*}
&(CC^Tv, v)=(C^Tv, C^Tv)\\
=&v_1^2+(-2v_1+v_2)^2+ (v_1-2v_2+v_3)^2\\
&+ \cdots+ (v_{n-2}-2v_{n-1}+v_n)^2\\
=&\lambda_{k,M_{\Psi}}^{-1}(M_\Phi v,v)=\lambda_{k,M_{\Psi}}^{-1}(v^k_h(x),v^k_h(x)).
\end{align*}
Noting that $v^k_h(x_0)=0$ and the zero extension of $v^k_h(x)$, we have the corresponding extension for the vector $v$ by $v_0=0$ and $v_{-1}=0$. 

Therefore, noting that $\mathcal{D}_h$ is the second order finite difference operator, we have
\begin{align*}
&(CC^Tv, v)\\
=&(v_{-1}-2v_0+v_1)^2+(v_0-2v_1+v_2)^2\\
&+ (v_1-2v_2+v_3)^2+ \cdots+ (v_{n-2}-2v_{n-1}+v_n)^2\\
=&h^4\sum_{i=1}^{n-1} |\mathcal{D}_hv^k_h(x_i)|^2
=h^3\int_0^1 |\mathcal{D}_hv^k_h(x)|^2 dx.
\end{align*}
 Hence, by \eqref{eigenMr}, we have
\begin{align*}
 h^3\int_0^1|\mathcal{D}_h v_h(x)|^2dx&=(CC^Tv, v)=\lambda_{k,M_{\Psi}}^{-1}(v^k_h(x),v^k_h(x))\\
 &=m^{-1}_kh^{-1}\nu_{n+1-k}\int_0^1| v^k_h(x)|^2dx
\end{align*}
namely
\begin{align}
 \frac{\int_0^1|\mathcal{D}_h v^k_h(x)|^2dx} {\int_0^1| v^k_h(x)|^2dx}=m^{-1}_kn^{4}\nu_{n+1-k},
\end{align}
where we used $h=\frac{1}{n}$. 
\end{proof}
From Theorem \ref{xi1xn}, Theorem \ref{relu-eigenfunction-thm} can be obtained. 

For example, when $n=128$, we can see that $v_h^1(x)=\xi_{\Psi}^1\cdot \Psi(x)$ and $v_h^2(x)=\xi_{\Psi}^2\cdot \Psi(x)$ are very smooth functions. And $v_h^{127}(x)=\xi_r^{127}\cdot \Psi(x)$ and $v_h^{128}(x)=\xi_r^{128}\cdot \Psi(x)$ are highly oscillatory functions, see Figure \ref{relu-eigenfunction-figure}.

\end{document}